%% file: main.tex
\documentclass[preprint,12pt,authoryear,english]{elsarticle}
\usepackage[T1]{fontenc}
\usepackage[utf8]{inputenc}
\usepackage{array,booktabs}
\usepackage{float}
\usepackage{color}
\usepackage{graphicx,tabularx}
\usepackage{subfigure,caption,subcaption}
\usepackage{babel}
\usepackage{pifont}
\usepackage{amsmath}
\usepackage{amsfonts}
\usepackage{amssymb}
\usepackage{amsthm}
\usepackage{rotating}
\usepackage{proof}
\usepackage{adjustbox}
\usepackage{multirow}
\usepackage{url}
\usepackage{algorithm}     
\usepackage{algpseudocode}  
\usepackage{caption}

\newtheorem{definition}{Definition}
\newtheorem{theorem}{Theorem}

\newtheorem{lemma}[theorem]{Lemma}
\newtheorem{corollary}[theorem]{Corollary}

\newcommand{\HIDE}[1]{}

\makeatletter
\def\ps@pprintTitle{%
  \let\@oddhead\@empty
  \let\@evenhead\@empty
  \let\@oddfoot\@empty
  \let\@evenfoot\@oddfoot
}

\makeatother

\begin{document}
\begin{frontmatter}
\title{Convergence issues in Relational Concept Analysis based on  AOC-posets}

\author[lbl1]{Xavier Dolques}
\author[lbl1]{Agn\`es Braud}
\author[lbl2]{Alain Gutierrez}
\author[lbl2]{Marianne Huchard\corref{cor1}}
\ead{marianne.huchard@lirmm.fr}
\author[lbl1]{Florence Le Ber}

\cortext[cor1]{Corresponding author}

\address[lbl1]{Univ. de Strasbourg, ENGEES, CNRS,
ICube UMR 7357, F-67000 Strasbourg, France}

\address[lbl2]{LIRMM, Univ. Montpellier, CNRS,
Montpellier, France}


\begin{abstract}
Formal Concept Analysis (FCA) is an approach for conceptual classification building and rule discovery from a binary table describing a set of objects by a set of attributes. Extensions have been proposed to deal with non-binary and more complex data, such as Relational Concept Analysis (RCA) for multi-relational data. RCA aims to highlight groups of objects characterized by their relationships with other groups of objects.
The richer and more complex nature of the underlying data allows RCA to produce richer results than FCA, at the expense of higher computational and interpretive complexity.
The most commonly used conceptual classification structure in FCA is the concept lattice. However, in many  applications, concept-lattice substructures—such as AOC-posets—are preferred over the full lattice, either to mitigate combinatorial blow-up or to focus on the most informative parts of the structure.
 Indeed, in an AOC-poset, only concepts introducing an object or an attribute are represented, which makes AOC-posets smaller and easier to compute and use than concept lattices. 
Although RCA was originally defined on concept lattices, it can also be instantiated on AOC-posets. 
RCA is iterative and its convergence is guaranteed in the lattice-based setting, but this guarantee is lost when using AOC-posets.
In this paper, we investigate this loss of convergence in detail. 
We show why convergence is no longer guaranteed in the general case, identify conditions under which it can still be ensured, and discuss how a dataset can be transformed to recover convergence. 
We also propose a convergent variant of the process, which preserves the AOC-poset structure:
relational attributes, once created, are never removed, which guarantees convergence at the price of attributes that may refer to concepts absent from the final structures. 
\end{abstract}

\begin{keyword}
Formal Concept Analysis \sep
Concept lattice \sep 
Relational Concept Analysis \sep AOC-posets \sep convergence
\end{keyword}

\end{frontmatter}
\input{tex__introduction}
\input{tex__backgroundrca}
\input{tex__rca_gsh}

\input{tex__convergence}

\input{tex__approaches_convergence}

\input{tex__relatedwork}
\input{tex__conclusion}

\section*{Declaration of generative AI and AI-assisted technologies in the writing process}

During the preparation of this work, the authors used Claude (Anthropic) in order to improve the language and readability of the manuscript, to obtain feedback on the mathematical definitions, proofs, and examples, and to assist with the drafting of LaTeX code. After using this tool, the authors reviewed and edited the content as needed and take full responsibility for the content of the published article.

\section*{Acknowledgement}
This work was supported by the French National Research Agency  Grant ANR-21-CE23-0023 (SmartFCA).

\bibliographystyle{elsarticle-harv}
\bibliography{bib-propre}

\input{tex__appendix}
\end{document}

%% file: tex__introduction.tex
\section{Introduction}

Formal Concept Analysis (FCA) \citep{Gant99a} provides a natural framework when the goal of data analysis is not primarily prediction, but the discovery and organization of interpretable and explainable knowledge patterns, in line with a white-box perspective on knowledge representation \citep{10.1007/s10618-024-01041-y}.
Based on lattice theory, FCA aims to discover conceptual structures from sets of objects and their attributes. It identifies formal concepts, each defined by a set of objects and the attributes they share, and organizes them into a concept lattice that makes abstraction, specialization, and dependency relations explicit. 
This contrasts with many mainstream machine-learning approaches, including supervised learning, clustering, dimensionality-reduction techniques, and neural models, which often focus on predictive accuracy, numerical similarity, or compact numerical representations \citep{hastie2009elements,LeCun2015}.
While explainable AI often aims at providing post-hoc explanations for predictive black-box models~\citep{10.1145/3236009}, FCA offers an intrinsically interpretable symbolic representation
and synthetic representation of the data. 
It supports knowledge discovery by producing explicit conceptual structures, abstract relations, and implications that can be inspected, interpreted, and discussed by domain experts.
FCA has been successfully used for analyzing tabular data, either binary or multivalued \citep{DBLP:journals/eswa/PoelmansIKD13,DBLP:journals/eswa/PoelmansKID13}. Various extensions have later been proposed to process more complex data, such as multidimensional data \citep{DBLP:journals/order/Voutsadakis02}, fuzzy data \citep{DBLP:conf/cla/BelohlavekV05}, and relational data \citep{rouane2013,ferre2020graph,DBLP:journals/dam/KottersE20}.

Relational Concept Analysis (RCA) \citep{rouane2013} is one of the extensions to handle relational data.
 In this framework, objects are described by attributes, and by relationships between them. RCA is based on the iterative use of FCA, computing several concept lattices, which are connected by links that abstract the relations between objects. 
 The result complements graph-pattern-based analyses \citep{ferre2020graph,DBLP:journals/dam/KottersE20} by highlighting how objects are classified separately within each category according to these relations.
Another distinctive feature of RCA is its use of several scaling operators borrowed from Description Logics \citep{baader03} to build links between concepts. This contrasts with approaches restricted to the existential operator, such as \cite{ferre2020graph,DBLP:journals/dam/KottersE20}, and enables the extraction of richer patterns.
Nevertheless,  since RCA groups  objects using relationships to  objects at any distance, it often comes with a combinatorial explosion, and patterns of interest are difficult to extract from the huge set of built concepts.  
Note that this is true for other FCA-based methods for relational data.
Various strategies can be used in RCA to cope with this complexity, including separating the initial formal object sets into smaller ones after a first analysis, introducing queries \citep{azmehCla11}, using frequency thresholds or specific measures to select the concepts of interest \citep{Stumme2002,buzmakov2014concept,DBLP:journals/isci/KuznetsovM18}, or using an alternative conceptual structure—namely, an AOC-poset—as a substitute for the concept lattices \citep{DBLP:conf/cla/DolquesBH13}. 
AOC-posets are sub-orders of lattices restricted to concepts introducing  objects or attributes, called respectively object-concepts (OC) and attribute-concepts (AC). 
They are smaller and computationally more efficient than concept lattices while preserving the original information \citep{Godi93a}. 

Using AOC-posets for RCA was motivated by the analysis of particular relational datasets, e.g. in  environmental data \citep{ijgis2016,braud2022}, where we showed that this approach provided a reasonable number of concepts and relevant rules.
A second advantage is that the object- and attribute-concepts contained in AOC-posets may be the only outputs needed in some applications. This is the case, for instance, in software engineering tasks such as class model refactoring \citep{DBLP:conf/cla/MirallesMHNDD15}. In this setting, the goal is to derive more general classes from an initial class model in which generalization relations are incomplete or missing. Attribute-concepts can guide the construction of these generalized classes, while object-concepts help preserve the classes of the initial model.
From these experiments, one may suggest that RCA-AOC has strong potential, both for focusing on the most relevant patterns in the context of a given application and for limiting the amount of extracted knowledge.
Nevertheless, in order to generalize 
the use of  an AOC-poset-based RCA process (RCA-AOC)  to other applications that rely on different—and potentially more complex—data schemas, it is important to verify whether RCA’s original properties are preserved. Unfortunately, we found that using AOC-posets within RCA can introduce divergence issues. A critical issue then arises: understanding when and why divergences occur, and devising ways either to eliminate them or to reliably exploit the results in their presence. 

In this paper, we illustrate the potential risk of process divergence through three examples involving the two most widely used scaling operators, namely existential scaling and strict universal scaling. One of these examples is concrete and is grounded in a UML class-model refactoring task. Furthermore, we identify properties that prevent divergence in most cases.
We also propose a variant of the process that guarantees termination, builds genuine AOC-posets, and yields 
concepts required in certain applications, such as the UML class-model refactoring task presented in this paper.

The rest of the paper is organized as follows. 
Section \ref{sec:rca_background} introduces the basic definitions of RCA.
Section \ref{sec:rca_gsh} details RCA-AOC, the RCA process based on AOC-posets.
Section \ref{sec:convergence} presents three examples of process divergence, including one drawn from a real-world software engineering application. 
In Sect.~\ref{sec_ensuring}, we discuss different approaches to ensuring convergence in the AOC-poset-based RCA process, including conditions on the data and on the process itself, as well as a convergent variant of the process. 
Section \ref{sec:relatedwork} presents the state of the art about AOC-posets and RCA, including theoretical developments, variants, and applications.
Finally, Sect.~\ref{sec:conclusion} concludes the paper and outlines directions for future work.

%% file: tex__backgroundrca.tex
\section{Introduction to Relational Concept Analysis}
\label{sec:rca_background}

In this section, we introduce the RCA  process and show the graphical form of its results. Definitions are illustrated with an example taken from the \textsc{CNRS Miti’80
2021 Paradise} project \citep{fokou-elhaff2024}.

\subsection{FCA basics}

RCA is a relational variant of Formal Concept Analysis (FCA)~\citep{Gant99a}. 
FCA takes as input a dataset, called a formal context, composed of objects described by attributes. Table \ref{table:plants} shows a formal context, denoted ${\mathcal K}_\textit{Plants}$, which describes plants used in ancient remedies and their useful parts or characteristics.
For example, leaves and roots of \textit{Archangelica officinalis} ($angelica$), an herbaceous plant, can be used in remedies. The aim of FCA is to extract an ordered set of concepts from the input dataset.

\input{Remedes__plants}

\begin{definition}[Formal Context] 
A formal context is a triple ${\mathcal K}=(G,  M,  I)$ ~where $G$~and ${M}$~are  finite sets
{of} objects and attributes, respectively, and ${I}$~is the incidence relation,
i.e., ${I}~\subseteq~{ G} \times {M}$. 
\end{definition}
\begin{definition}[Derivation operators]
Let $\mathcal{K} = (G, M, I)$ be a formal context. The two
\textit{derivation operators}, both denoted by $(\cdot)'$, are defined,
for $X \subseteq G$ and $Y \subseteq M$, by:
\[
X' = \{ m \in M \mid \forall g \in X,\ (g,m) \in I \}
\qquad
Y' = \{ g \in G \mid \forall m \in Y,\ (g,m) \in I \}
\]
$X'$ is the set of attributes shared by all objects of $X$, and $Y'$
the set of objects owning all attributes of $Y$. Composing the two
operators yields the closure operators $(\cdot)''$ on $G$ and on $M$.
For a single object $o \in G$, we write $\{o\}'$ and $\{o\}''$.
\end{definition}
\begin{definition}[Formal Concept]
A formal concept of $\mathcal{K} = (G, M, I)$ is a pair
$C = (X, Y)$ with $X \subseteq G$ and $Y \subseteq M$, such that
$X' = Y$ and $Y' = X$.
$X = Extent(C)$ is the \textit{extent} of the concept, i.e., the set of
objects falling under the concept, and $Y = Intent(C)$ is its
\textit{intent}, i.e., the set of attributes shared by these objects.
\end{definition}

\begin{definition}[Concept Specialization Order and Lattice]
Let ${\mathcal C}_{\mathcal K}$ be the set of all concepts built from  formal context $\mathcal K$. Let 
 $C_1=(X_1, Y_1)$ and $C_2=(X_2, Y_2)$ be two elements of ${\mathcal C}_{\mathcal K}$.
The concept specialization order $\leq_{s}$  is defined by $C_1 \leq_{s}  C_2$ if and only if $X_1 \subseteq X_2$ (and equivalently  $Y_2 \subseteq Y_1$).
$C_1$ is called a subconcept of $C_2$, while $C_2$ is called a superconcept of $C_1$.
The concept set ${\mathcal C}_{\mathcal K}$ provided with the  specialization order,  (${\mathcal C}_{\mathcal K}$, $\leq_{s}$),  has a lattice structure, and is called the concept lattice associated with ${\mathcal K}$.
\end{definition}
The concept lattice built on ${\mathcal K}_\textit{Plants}$ is shown in Fig. \ref{fig:latticeplants}.
To reduce redundant information in the presentation of concept lattices, figures only show  introduced objects (resp. introduced attributes) in concepts. 
E.g. \texttt{C\_plants\_8} (denoted \texttt{cp\_8} for short) introduces attribute $flowers$ and object $camomile$\footnote{\texttt{C\_plants\_8} is the introducer concept of attribute $flowers$ and object $camomile$.}.
But the complete concept definition is 
$Intent(\mathtt{cp\_8})$ = $\{herbaceous$,  $flowers\}$ with $herbaceous$ top-down inherited, and 
$Extent($ $\mathtt{cp\_8})$ = $\{camomile$,  $borage$, $garlic\}$ with $borage$ and $garlic$ bottom-up inherited.
The presentation exploits the fact that when an attribute belongs to the intent of a concept $C$, it is (top-down) inherited by $C$ subconcepts. Thus it is sufficient to show an attribute in the highest concept having it in its intent. This highest concept is the $introducing$ concept of the attribute. The same simplification can be symmetrically applied to the objects, which are bottom-up inherited.

\begin{figure}[htb]
\centering
\includegraphics[width=0.45\linewidth]{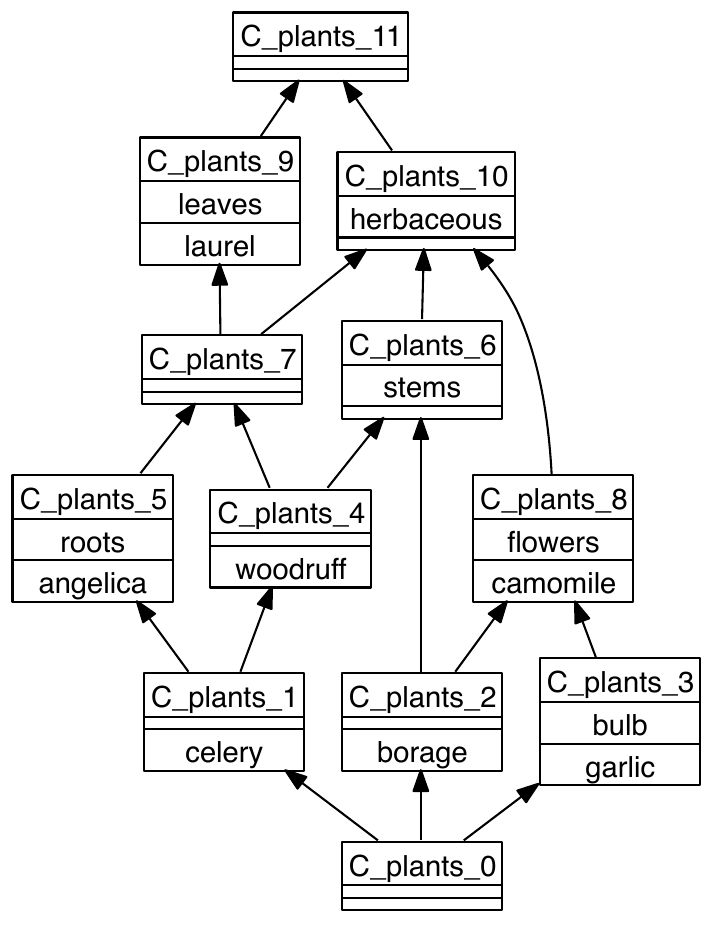}
\caption{Concept Lattice  on ${\mathcal K}_\textit{Plants}$. Each concept is represented in 3 parts: its identifier (top), its intent (middle), and its extent (bottom); only introduced attributes and objects are shown in intents and extents respectively. In the text, \texttt{C\_plants\_i} is shortened to  \texttt{cp\_i}.}
\label{fig:latticeplants} 
\end{figure}

\subsection{RCA process}

Relational Concept Analysis (RCA) aims at extending FCA to take into account datasets where objects of several categories are described by attributes and by relations to objects \citep{HuchardHRV07,rouane2013}. 
The main idea of RCA is as follows: when objects of a category are connected via a relation
to objects of another category (or the same category), concepts formed on top of objects of the latter category can be used to form concepts on top of objects of the former category.  The whole dataset is called a Relational Context Family. 

\begin{definition}[Relational Context Family (RCF)]
\label{def:rel-context}
A Relational Context Family (denoted RCF) is a $({\mathbf K},{\mathbf R})$ pair where:
\begin{itemize}
\item ${\mathbf K} = \{{\cal K}_i\}_{i=1,\ldots,n}$ is a set of contexts, ${\cal K}_i=(G_i,M_i,I_i)$. 
\item ${\mathbf R}= \{r_{j}\}_{j=1,\ldots,m}$ is a set of relations,  
$r_j \subseteq G_{s_{r_j}} \times G_{t_{r_j}}$ 
where 
$s_{r_j}, t_{r_j} \in \{1,\ldots,n\}$, and $G_{s_{r_j}}$ and $G_{t_{r_j}}$ are resp. the source (domain) and target (range) of $r_j$.
\end{itemize}
\end{definition}

In the following, we assume the object sets ($G_i$, $i=1,...,n $) to be pairwise disjoint. 
The relationship between objects of a category and concepts formed on top of objects of another  category is
 implemented  thanks to \textit{scaling operators}, which are inspired by operators of Description Logics. This results in the creation of special attributes called \textit{relational attributes}, of the form $qr(C)$, where $q$ is a scaling operator, $r$ a relation between two object sets $G_{s_{r}}$ and $G_{t_{r}}$,  and $C$ a concept from the lattice built on the context ${\cal K}_{t_{r}} = (G_{t_{r}}, M_{t_{r}}, I_{t_{r}})$. 
The most used {scaling operators} are:
\begin{itemize} 
\item The \textit{existential} scaling operator ($\exists$):  an object $o \in  G_{s_{r}}$ is in relation by $\exists r$ with a concept $C$ from  ${\cal K}_{t_{r}}$ if $r(o)$ has a non-empty intersection with $Extent(C)$, where $r(o)$ denotes the image set of $o$ by $r$, i.e. $o$ is connected via $r$ to at least one element from $Extent(C)$. {A relational attribute based on this scaling operator is denoted as $\exists r(C)$}.
\item The \textit{universal} scaling operator ($\forall$):  an object $o \in  G_{s_{r}}$ is in relation by $\forall r$ with a concept $C$ from  ${\cal K}_{t_{r}}$ if  $r(o)$ is  included in the extent of $C$, i.e. $o$ is connected via $r$ only to elements from $Extent(C)$. {A relational attribute based on this scaling operator is denoted as $\forall r(C)$}. 
\item The \textit{strict universal} scaling operator ($\exists\forall$):  an object $o \in  G_{s_{r}}$ is in relation by $\exists\forall r$ with a concept $C$ from  ${\cal K}_{t_{r}}$ if  $r(o)$ is \textit{non-empty} and included in the extent of $C$, i.e. $o$ is connected via $r$ only to elements from $Extent(C)$ (and \textit{at least one}). {A relational attribute based on this scaling operator is denoted as $\exists\forall r(C)$}.
\end{itemize}

When processing an RCF $({\mathbf K},{\mathbf R})$ with RCA, different scaling operators can be applied to the relations of ${\mathbf R}$. We introduce then the function $\rho:{\mathbf R} \rightarrow \{\exists, \forall, \exists\forall\}$   that associates a scaling operator to a relation from ${\mathbf R}$. 

We here  explain these principles relying on an example taken from a database on ancient Arabic remedies \citep{fokou-elhaff2024}. The relational context family, called \texttt{RP} in the following, is composed of a Plants context, ${\mathcal K}_\textit{Plants}$ (see Table~\ref{table:plants}), a Remedies context, denoted ${\mathcal K}_\textit{Remedies}$ (Table at the left-hand side of Fig.~\ref{fig:remedies}), and a \textit{contains} relation, denoted  ${r}_{contains}$  (Table \ref{table:contains}). Remedies are described by their ingredients and their application form;  two forms are here considered: pills and potions.
Let us for example consider the concept lattice of Fig.~\ref{fig:latticeplants} and the \texttt{cp\_8} concept which groups plants whose flowers are used in remedies, namely \textit{camomile}, \textit{borage} and \textit{garlic}. Now, according to the existential scaling operator, \texttt{remedy1}, \texttt{remedy2} and \texttt{remedy3} can be linked to \texttt{cp\_8} because they contain at least one plant with useful flowers,
i.e. at least one plant in $Extent($\texttt{cp\_8}$)$  (Table~\ref{table:contains}).  
To represent this property, a new attribute is added to objects \texttt{remedies}. This is done for each concept from  ${\mathcal K}_\textit{Plants}$ leading to an extension of the original context, as defined below.
\begin{figure}[htb]
\begin{minipage}{0.5\linewidth}
\centering
\input{Remedes__remedes}
\end{minipage}
\begin{minipage}{0.45\linewidth}
\centering
\includegraphics[width=0.32\linewidth]{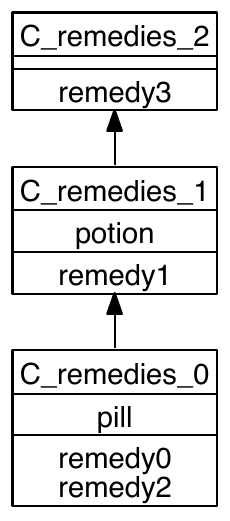}
\end{minipage}
\caption{Formal Context ${\mathcal K}_\textit{Remedies}$ (left) and concept lattice (right) for remedies. In the text, \texttt{C\_remedies\_i} is shortened to  \texttt{cr\_i}.}
\label{fig:remedies} 
\end{figure}
\input{Remedes__contains}

\begin{definition}[Existential scaling ($\exists$)]
\label{def:exist-scaling}
Let ${\cal K}=(G,M,I)$ be a context in a relational context family, and $r$ a relation, where $G= G_{s_r}$ is the source of $r$, and $G_{t_r}$, the target of $r$, is the object set of a formal context ${\cal K}_{t_r}=(G_{t_r},M_{t_r},I_{t_r})$.
Let also  ${\cal C}_{t_r}$ be the set of concepts of the lattice built on ${\cal K}_{t_r}$. 
The application of $\mathbb{S}_{\exists}$, the  existential scaling operator,  to context ${\cal K}$, denoted by
$ \mathbb{S}_{\exists}$(${\cal K},r,{\cal C}_{t_r}$), 
gives the extension ${\cal K}^+=(G^+,M^+,I^{+})$ of ${\cal K}$, with:
\begin{itemize}
\item 
$G^{+}=G$
\item 
$M^{+}=\{ \exists r(C) ~|~ C \in {\cal C}_{t_r}\}$.
\item 
$I^{+}=\{ (o, \exists r(C)) ~|~ o \in G, C \in {\cal C}_{t_r}, r(o) \cap Extent(C) \neq \emptyset\}$
\end{itemize}
\end{definition}

Table \ref{table:existsplantsRCA} shows $\mathbb{S}_{\exists }(K_\textit{Remedies},r_{contains},{\cal C}_\textit{Plants})$ where
${\cal C}_\textit{Plants}$ is the set of concepts of the lattice of Fig.~\ref{fig:latticeplants}.
In this case, $I^{+}$ contains e.g.
$(\texttt{remedy1},$ $\exists~{r_{contains}}(\texttt{cp\_8}))$,
$(\texttt{remedy2},$  $\exists~{r_{contains}}(\texttt{cp\_8}))$,
and
$(\texttt{remedy3},$  $\exists~{r_{contains}}$ $(\texttt{cp\_8}))$, 
because \texttt{remedy1} contains $camomile$,  \texttt{remedy2} and \texttt{remedy3} contain $garlic$, both plants in \texttt{cp\_8} extent.
 This extension is added to the original context and then used to update the Remedies lattice. The new concept lattice is shown in Fig. \ref{fig:latticeremedes2}.

\input{Remedes__remedes_existentiel}

\begin{definition}[Strict universal scaling ($\exists\forall$)]
\label{def:univ-scaling}
Let ${\cal K}=(G,M,I)$ be a context in a relational context family, and $r$ a relation, 
where $G= G_{s_r}$ is the source of $r$, and $G_{t_r}$, the target of $r$, is the object set of a formal context ${\cal K}_{t_r}=(G_{t_r},M_{t_r},I_{t_r})$.
Let also  ${\cal C}_{t_r}$ be the set of concepts of the lattice
built on ${\cal K}_{t_r}$. 
The application  of $ \mathbb{S}_{\exists\forall}$, the strict universal scaling operator, to context ${\cal K}$, denoted by
 $\mathbb{S}_{\exists\forall}$(${\cal K},r,{\cal C}_{t_r}$), gives the extension 
${\cal K}^+=(G^+,M^+,I^{+})$, with:
\begin{itemize}
\item 
$G^{+}=G$
\item 
$M^{+}=\{ \exists\forall r(C) ~|~ C \in {\cal C}_{t_r}\}$ 
\item 
$I^{+}=\{ (o, \exists\forall r(C)) ~|~ o \in G, C \in {\cal C}_{t_r}, r(o) \neq \emptyset \,\mathit{and}\, r(o) \subseteq Extent(C)\}$
\end{itemize}
\end{definition}

\begin{definition}[Relational extension of a context]
\label{def:rel-extent}
Let ${\cal K} = (G,M,I)$ be a context in a relational context family $({\mathbf K},{\mathbf R})$, and $R_G$ the subset of relations $r \in \mathbf{R}$ such that $G_{s_{r}}=G$.
The relational extension of ${\cal K}$, denoted  $\mathbb{E}_{\rho,{\mathbf C}}({\cal K})$, is the result of the apposition (denoted by the symbol `|') of ${\cal K}$ with all extensions $\mathbb{S}_{\rho(r)}$(${\cal K},r,{\cal C}_{t_r}$) for $r \in R_G$, and $\mathbf{C} =$$ \bigcup_{r \in R_G}{\mathcal{C}_{t_{r}}}$.
$$ \mathbb{E}_{\rho,{\mathbf C}}({\cal K}) =  {\cal K}~ |~ 
\mathbb{S}_{\rho(r_{1})}( {\cal K},r_{1},{\cal C}_{t_{r_1}})~ 
|~ \ldots |~ \mathbb{S}_{\rho(r_{k})}( {\cal K},r_{k},{\cal C}_{t_{r_k}})$$
\end{definition}

By extension, $E^*_{\rho}(\mathbf{K})$ denotes the relational extension
of $\mathbf{K}$, which is composed of all the relational extensions of
all $\mathcal{K}_i$ in $\mathbf{K}$:
$$E^*_{\rho}(\mathbf{K}) = \{E_{\rho,{\mathbf C_1}}(\mathcal{K}_1), \ldots,
E_{\rho,{\mathbf C_n}}(\mathcal{K}_n)\}$$
where, for each $i$:
$${\mathbf C_i} =
\bigcup_{r \in R_{G_i}} \mathcal{C}_{t_r}$$

A relational extension of  ${\mathbf K} = \{{\mathcal K}_\textit{Plants}, {\mathcal K}_\textit{Remedies}\}$  is composed of Table \ref{table:plants} (no outgoing relation), and the left table of Fig.~\ref{fig:remedies} apposed to Table \ref{table:existsplantsRCA}.  The resulting context is ${\mathbf K}^{1} = \mathbb{E}^*_{\rho}({\mathbf K}) = \{{\mathcal K}_\textit{Plants}, \mathbb{E}_{\rho,{\mathbf C}}({\cal K}_\textit{Remedies})\}$. Figure~\ref{fig:latticeremedes2} shows the concept lattice built from the extended context  $\mathbb{E}_{\rho,{\mathbf C}}({\cal K}_\textit{Remedies}) = {\mathcal K}_\textit{Remedies}  | \mathbb{S}_{ \exists}(K_\textit{Remedies}, r_\mathit{contains},{\cal C}_\textit{Plants})$.

\begin{figure}[htb]
\centering
\includegraphics[width=0.6\linewidth]{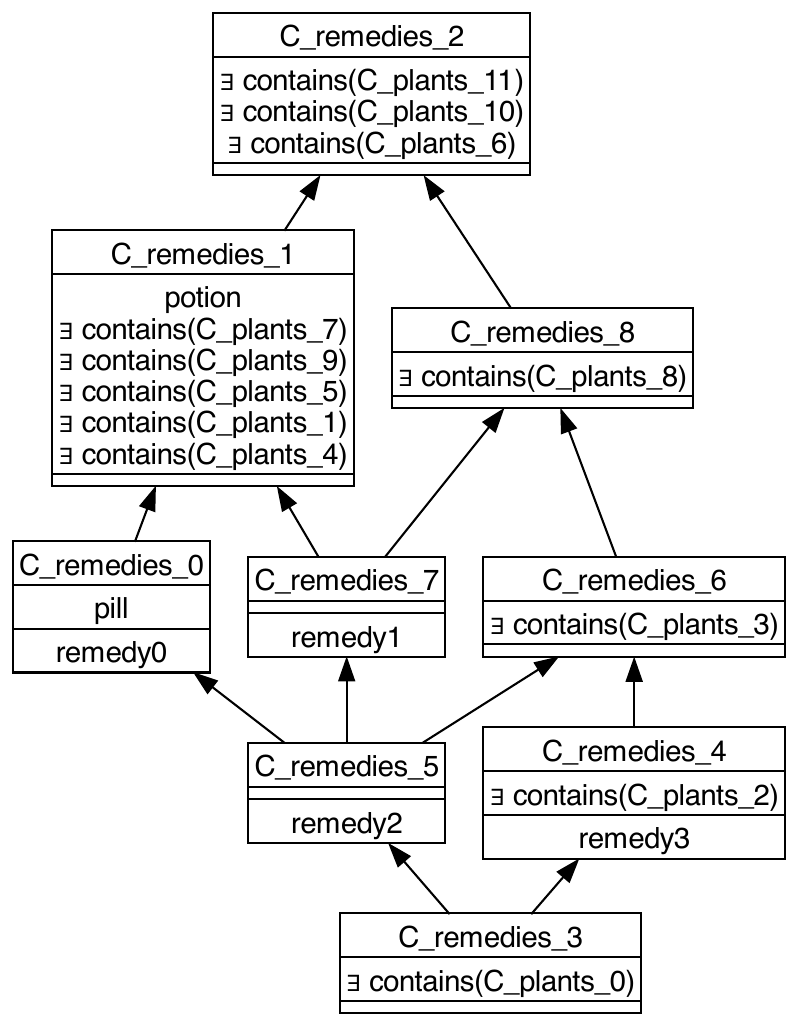}
\caption{Concept Lattice  on $\mathbb{S}_{\exists }({\mathcal K}_\textit{Remedies},r_{contains},{\cal C}_\textit{Plants})$ apposed to ${\mathcal K}_\textit{Remedies}$.}
\label{fig:latticeremedes2} 
\end{figure}

A whole construction process consists in building a finite sequence of contexts and concept lattices associated with $({\mathbf K},{\mathbf R})$ and $\rho$.   The first set of contexts (step $0$) is ${\mathbf K}^{0} ={\mathbf K}$.
The contexts of step $p$ are extended with relational attributes built on the concepts of ${\mathbf C}^{p-1}$, i.e. the set of concepts of all lattices built at the previous step. These extended contexts are used to build new lattices providing a new set of concepts ${\mathbf C}^{p}$ that will be used in the next step.
The process stops when a fixpoint is obtained, i.e. lattices associated with the same formal context at two successive steps have the same sets of concept extents. The result is a family of interrelated lattices, called a \textit{Concept Lattice Family} (CLF).

\subsection{RCA output as graphs}

RCA is usually described as a process, starting from an RCF and producing a family of interrelated concept lattices \citep{rouane2013}. Nevertheless, several works have highlighted the links between these outputs and graphs, in various domains, e.g. 
\cite{dolques2010learning}  produced
model transformation patterns in the software domain. 
\cite{nica-kbs2020,nica-dam2020} proposed a general method, called RCA-Seq, to extract a hierarchy of directed acyclic graphs 
from the lattice family, starting from one chosen lattice. In the resulting  hierarchy,  each graph represents a navigation path through the lattice family. Redundant links have been removed.  In the same way,  \cite{ferre2018hierarchies} introduced a generic representation of RCA outputs as a hierarchy of concept graphs where each concept of the lattice family belongs to one concept graph and each concept graph exhibits the relationships between several concepts.

Following this idea, a recent work has shown that RCA outputs can be represented as a set of relational patterns with no loss of information \citep{fokou-ijar2025}. These relational patterns are directly comparable to graph patterns produced by the Graph-FCA approach \citep{ferre2020graph} when inverse relations are included in the relational context family.
The set of relational patterns is derived from a dependency graph $G=(V,E)$, where $V$ is the set of all concepts described with unary attributes, $E$ is the set of dependencies (relational attributes and subsumption) between concepts. This graph is built on top of the concept lattice family, by extracting maximal subsets of related concepts and  removing redundant links.
This approach cannot be directly applied on RCF with one-way relations. But the approach proposed by \cite{nica-dam2020}  for temporal data can be adapted for any data model of a relational context family.

\begin{figure}[htb]
    \centering
    \includegraphics[width=\textwidth]{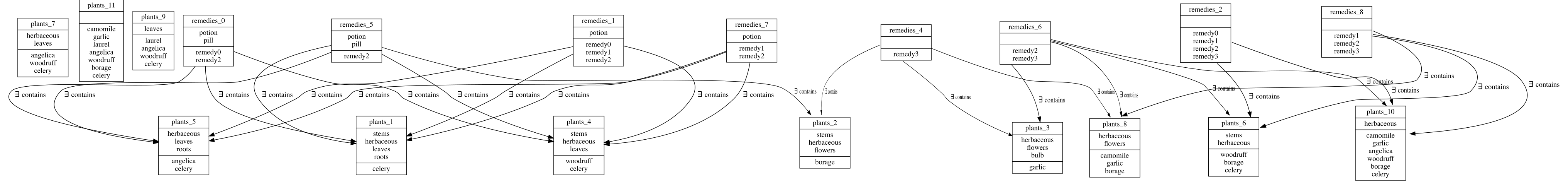}
    \includegraphics[width=\textwidth]{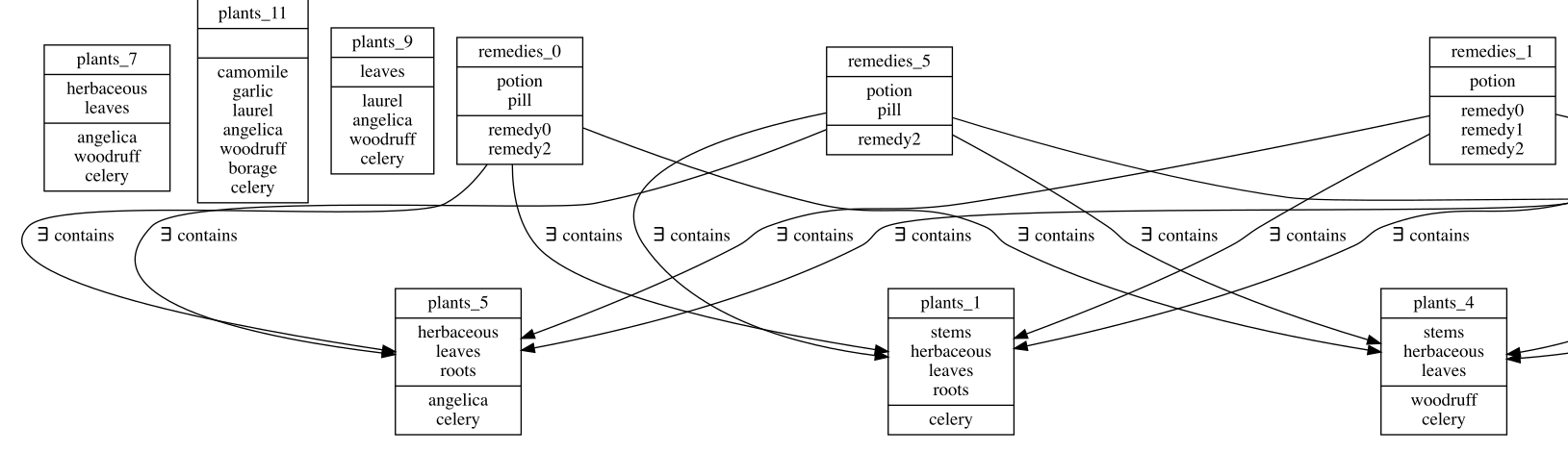}
    \caption{General overview of the patterns extracted from the concept lattice family built on \texttt{RP} with the existential quantifier and a focus on \texttt{cr\_0} and linked  nodes.}
    \label{fig:graphremedesplant}
\end{figure}

As formalized by \cite{nica-dam2020} and \cite{fokou-ijar2025}, the process of pattern extraction requires to remove some redundant links. 
For example, remedy concept \texttt{cr\_0} (see Fig. \ref{fig:latticeremedes2}) has eight relational attributes  pointing at \texttt{cp\_1}, \texttt{cp\_4},  \texttt{cp\_5}, \texttt{cp\_6}, \texttt{cp\_7}, \texttt{cp\_9}, \texttt{cp\_10} and \texttt{cp\_11}. 
Plant concept \texttt{cp\_11} (top concept, see Fig. \ref{fig:latticeplants}) brings no information  as well as \texttt{cp\_7}. 
Plant concepts \texttt{cp\_6}, \texttt{cp\_9} and  \texttt{cp\_10} introduce  attributes that are inherited by lower concepts; besides, \texttt{cp\_9}  introduces an object that is not linked to remedies of  \texttt{cr\_0} (\texttt{remedy0} and \texttt{remedy2}). Finally, in the graph built on this lattice family (see Fig. \ref{fig:graphremedesplant}), only concepts \texttt{cp\_1} (denoted \texttt{plants\_1}), \texttt{cp\_4}  (\texttt{plants\_4}), and \texttt{cp\_5} (\texttt{plants\_5})  are  represented as linked to \texttt{cr\_0} (\texttt{remedies\_0}).

%% file: Remedes__plants.tex
\begin{table}
\caption{Formal Context of plants and their used parts (${\mathcal K}_{Plants}$)}
\label{table:plants}
\begin{center}
\footnotesize
\begin{tabular}{|l| *{8}{m{1.4cm}|}}
\hline
& \hbox{herbaceous} & bulb & leaves & roots &  stems & flowers  \\
\hline
\textit{camomile} & \hspace{0.8cm}$\times$ & & & & & \hspace{0.8cm}$\times$ \\
\hline
\textit{garlic} & \hspace{0.8cm}$\times$ & \hspace{0.8cm}$\times$ & & & &\hspace{0.8cm}$\times$\\
\hline
\textit{laurel} & & & \hspace{0.8cm}$\times$ & & &\\
\hline
\textit{angelica} &\hspace{0.8cm}$\times$	&&\hspace{0.8cm}$\times$&\hspace{0.8cm}$\times$&&\\
\hline
\textit{woodruff} &\hspace{0.8cm}$\times$		&&\hspace{0.8cm}$\times$	&	&\hspace{0.8cm}$\times$	&\\
\hline
\textit{borage} &\hspace{0.8cm}$\times$	&	& 	&	&\hspace{0.8cm}$\times$&\hspace{0.8cm}$\times$\\
\hline
\textit{celery} &\hspace{0.8cm}$\times$& 	&\hspace{0.8cm}$\times$	&\hspace{0.8cm}$\times$	&\hspace{0.8cm}$\times$	&  \\
\hline
\end{tabular}
\end{center} 
\end{table}

%% file: Remedes__remedes.tex
\centering
\small
\begin{tabular}{|l|c|c|}
\hline
& \textbf{pill}& \textbf{potion} \\
\hline
\textbf{remedy0}& $\times$  &  $\times$ \\
\hline
\textbf{remedy1}&  & $\times$  \\
\hline
\textbf{remedy2} & $\times$  &  $\times$\\
\hline
\textbf{remedy3}&   & \\
\hline
\end{tabular}

%% file: Remedes__contains.tex
\begin{table}[htb]
\caption{Relation  $\mathit{r}_{contains}$ }
\label{table:contains}
\centering 
{\scriptsize
\begin{tabular}{|l|c|c|c|c|c|c|c|}
\hline
$r_\mathit{contains}$ & \textit{\textbf{camomile}}&  \textit{\textbf{garlic}}&  \textit{\textbf{laurel}}& \textit{\textbf{angelica}}& \textit{\textbf{woodruff}} 
&  \textit{\textbf{borage}} &  \textit{\textbf{celery}}\\
\hline
\textbf{remedy0}&	&	&	&$\times$ & $\times$	& & $\times$	\\
\hline
\textbf{remedy1}&$\times$	& &	 & $\times$	 & $\times$	 & 	&$\times$	 \\
\hline
\textbf{remedy2} &	&$\times$	&	&$\times$	&$\times$	& &	$\times$	\\
\hline
\textbf{remedy3}	&$\times$	&$\times$	&	&	&	&	$\times$	&		 \\
\hline
\end{tabular}}
\end{table}

%% file: Remedes__remedes_existentiel.tex
\begin{table}
\caption{Existential scaling of the formal context for Remedies based on the relation $r_{contains}$ (shortened to $r$ in the column headers),
$\mathbb{S}_{\exists}(\mathcal{K}_\textit{Remedies}, \mathit{r_{contains}}, \mathcal{C}_{\textit{Plants}})$. 
}
\label{table:existsplantsRCA}
\centering 
\setlength{\tabcolsep}{2pt}
\scriptsize{
\begin{tabular}{|l|*{12}{c|}}
\hline
\texttt{remedies}  & \begin{sideways}\textbf{$\exists$ r(cp\_11)}\end{sideways} & \begin{sideways}\textbf{$\exists$ r(cp\_10)}\end{sideways} & \begin{sideways}\textbf{$\exists$ r(cp\_3)}\end{sideways} & \begin{sideways}\textbf{$\exists$ r(cp\_0)}\end{sideways} & \begin{sideways}\textbf{$\exists$ r(cp\_7)}\end{sideways} & \begin{sideways}\textbf{$\exists$ r(cp\_9)}\end{sideways} & 
\begin{sideways}\textbf{$\exists$ r(cp\_5)} \end{sideways}& \begin{sideways}\textbf{$\exists$ r(cp\_1)}\end{sideways} & \begin{sideways}\textbf{$\exists$ r(cp\_4)}\end{sideways} & \begin{sideways}\textbf{$\exists$ r(cp\_6)}\end{sideways} & \begin{sideways}\textbf{$\exists$ r(cp\_2)}\end{sideways} & \begin{sideways}\textbf{$\exists$ r(cp\_8)}\end{sideways}\\
\hline
\textbf{remedy0}& $\times$ & $\times$  &  & & $\times$ & $\times$ & $\times$ & $\times$ & $\times$ & $\times$  &  & \\
\hline
\textbf{remedy1} &  $\times$ & $\times$  &  & & $\times$ & $\times$ & $\times$ & $\times$ & $\times$ & $\times$  & & $\times$ \\
\hline
\textbf{remedy2}&  $\times$ & $\times$ & $\times$  & & $\times$ & $\times$ & $\times$ & $\times$ & $\times$ & $\times$  & & $\times$ \\
\hline
\textbf{remedy3} &  $\times$ & $\times$ & $\times$  &  &  &  &  &  & & $\times$ & $\times$ & $\times$ \\
\hline
\end{tabular}
}
\end{table}

%% file: tex__rca_gsh.tex
\section{Relational Concept Analysis based on AOC-posets}
\label{sec:rca_gsh}

A variant of RCA based on AOC-posets --called RCA-AOC-- was introduced by \cite{DBLP:conf/cla/DolquesBH13}. We here  explain its principles relying on the same relational context family.
The AOC-poset associated with a concept lattice is its suborder restricted to concepts that introduce at least one object or one attribute. 
The AOC-posets built on the Plants-Remedies relational context family are shown in Fig.~\ref{fig:GSHremedesplantes}.
Three concepts have been removed from the Plant concept lattice of Fig.~\ref{fig:latticeplants}: \texttt{cp\_11} (top), \texttt{cp\_0} (bottom), and \texttt{cp\_7} (internal), leading to the Plant AOC-poset on the right-hand side of Fig. \ref{fig:GSHremedesplantes} (note that the concepts have been renumbered between Fig. \ref{fig:latticeplants} and Fig. \ref{fig:GSHremedesplantes}).

\begin{figure}[htb]
\begin{minipage}{0.35\linewidth}
\centering
\includegraphics[width=0.40\linewidth]{Remedes__remedestep0.pdf}
\end{minipage}
\begin{minipage}{0.65\linewidth}
\centering
\includegraphics[width=0.90\linewidth]{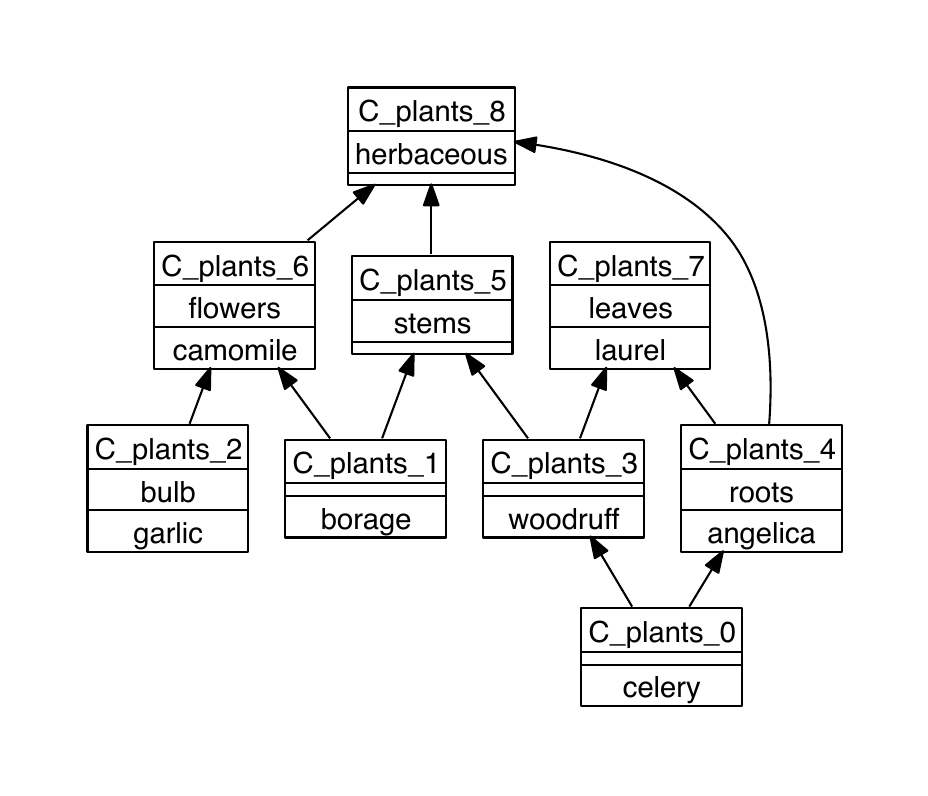}
\end{minipage}
\caption{AOC-posets of Remedies (left) and Plants (right).}
\label{fig:GSHremedesplantes} 
\end{figure}

Now, since we rely on AOC-posets, the existential scaling operator is defined slightly differently from classical RCA \citep{rouane2013}. While the scaling operation on RCA relies on the fact that all the concepts of the concept lattice have been created (see Definitions \ref{def:exist-scaling} and \ref{def:univ-scaling}), we acknowledge here that only a subset of the concept lattice, such as an AOC-poset, can be considered to build relational attributes. Then the previous definitions of scaling operators are unchanged except that ${\cal C}_{t_r}$ is any set of concepts, not necessarily the entire set of concepts of the concept lattice.

Table \ref{table:existsplants} shows $\mathbb{S}_{\exists }(\mathcal{K}_{\mathit{Remedies}},r_{\mathit{contains}},{\cal C}_{\mathit{Plants}})$ where
${\cal C}_{Plants}$ is the set of concepts of the AOC-poset of the right-hand side of Fig.~\ref{fig:GSHremedesplantes}.
$I^{+}$ contains $(\mathtt{remedy3},$ $\exists~{\mathit{contains}}(\mathtt{cp\_1}))$, $(\mathtt{remedy3},$ $\exists~{\mathit{contains}}(\mathtt{cp\_2}))$, $(\mathtt{remedy3},$ $\exists~{\mathit{contains}}(\mathtt{cp\_5}))$, $(\mathtt{remedy3},$ $\exists~{\mathit{contains}}(\mathtt{cp\_6}))$ and $(\mathtt{remedy3},\exists~{\mathit{contains}}(\mathtt{cp\_8}))$ because ${contains}$ $(\mathtt{remedy3})=\{camomile, garlic, borage\}$ and each of the pointed concepts contains at least one of those plants in its extent.

\input{Remedes__AOCstep2-0}

In the case of the strict universal scaling operator, 
$I^{+}=\{ (o, \exists\forall r(C)) | o \in G, C \in {\cal C}_{t_r}, r(o) \subseteq Extent(C) ~and~ r(o) \neq \emptyset\}$.
Table \ref{table:forallplants} shows 
 $\mathbb{S}_{\exists\forall }(\mathcal{K}_{Remedies},$ $\mathit{contains},$ ${\cal C}_{\mathit{Plants}})$.
In this case, $I^{+}$ contains e.g.   $(\mathtt{remedy3},$ $\exists\forall~{\mathit{contains}}(\mathtt{cp\_6}))$ because all plants found in \texttt{remedy3} belong to the extent of this concept. But $I^{+}$  does not contain $(\mathtt{remedy3},$ $\exists\forall~{\mathit{contains}}(\mathtt{cp\_2}))$ because \texttt{remedy3} contains plants (e.g. \textit{camomile}) that are not in the extent of this concept (reduced to \textit{garlic}).

\input{Remedes__AOCuniversel}

\begin{figure}[htb]
\centering
\includegraphics[width=0.6\linewidth]{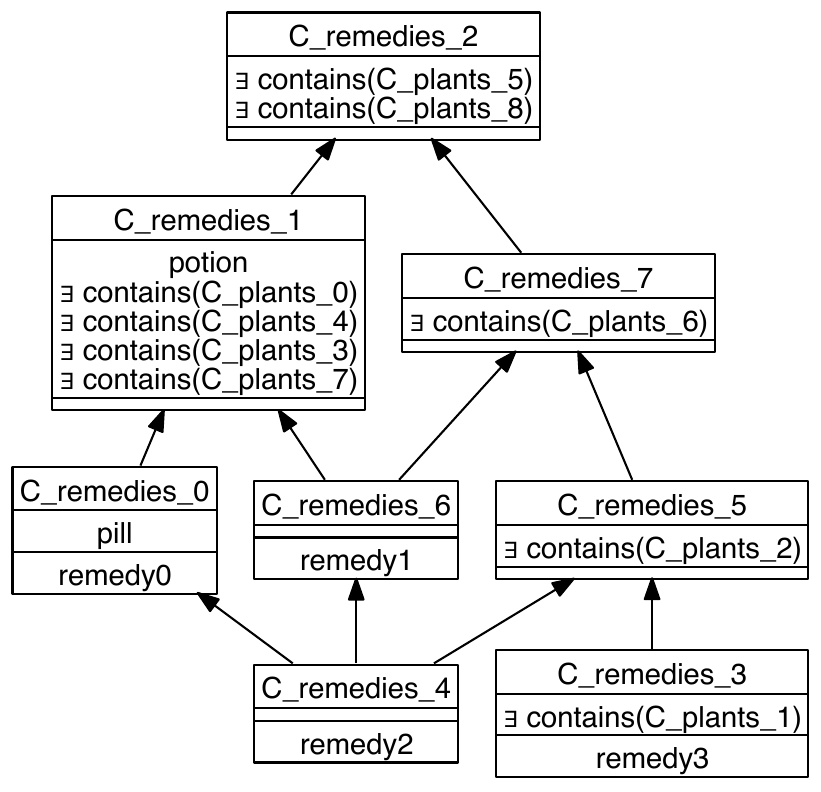}
\caption{AOC-poset for the $K_{Remedies}$ formal context extended with the existential scaling on ${\cal C}_{Plants}$.}
\label{fig:step0-Remedies}
\end{figure}

As for RCA, the whole construction process consists in building a (possibly infinite) sequence of contexts and AOC-posets associated with $({\mathbf K},{\mathbf R})$ and $\rho$.   
The difference here with classical RCA is that step $p+1$ relies on the original context ${\mathbf K}$ and the concept set of AOC-posets ${\mathbf C}^p$ to build ${\mathbf K}^{p+1}$, rather than on ${\mathbf  K}^{p}$ and the set of concepts from lattices of step $p$ (this is discussed in Section \ref{sec_rcaacoconv}).

In the current simple example, the following steps do not produce any new concept (a fixpoint is reached). More complex datasets may include cycles between objects. This can lead to convergence problems in the RCA-AOC process, as discussed in the following section.

\paragraph*{RCA-AOC output as graphs}
Obviously, patterns built from AOC-posets are included in the corresponding patterns built from RCA lattices. Furthermore, since building patterns requires removing redundancies, i.e. concepts that bring no information, building patterns based on AOC-poset is more straightforward than based on RCA lattices and should lead to the same result.
For example, Remedy AO-concept\footnote{An AO-concept is a concept introducing at least one attribute or object, i.e. it belongs to the AOC-poset.} \texttt{cr\_0}  in Fig. \ref{fig:step0-Remedies} corresponds to  concept \texttt{cr\_0}  in Fig. \ref{fig:latticeremedes2}. It has only six relational attributes, pointing at  Plant AO-concepts \texttt{cp\_0}, \texttt{cp\_3}, \texttt{cp\_4}, \texttt{cp\_5}, \texttt{cp\_7}, and \texttt{cp\_8}. Plant AO-concepts \texttt{cp\_5} and \texttt{cp\_8} introduce only attributes that are inherited by lower AO-concepts, \texttt{cp\_7} introduces an object that is  not linked to remedies of AO-concept \texttt{cr\_0}. In the final pattern only  \texttt{cp\_0} (\texttt{celery}), \texttt{cp\_3} (\texttt{woodruff}, \texttt{celery}) and \texttt{cp\_4} (\texttt{angelica}, \texttt{celery}) are linked to Remedy AO-concept \texttt{cr\_0} as for Remedy concept \texttt{cr\_0} in Fig. \ref{fig:graphremedesplant}.

%% file: Remedes__AOCstep2-0.tex
\begin{table}
\caption{Existential Scaling of Formal Context of Remedies with AOC-poset of Plants $\mathbb{S}_{\exists}(K_{Remedies}, r_\mathit{contains},{\cal C}_{Plants})$. $\exists r(\texttt{cp\_i})$ abbreviates $\exists {r_\mathit{contains}}(\texttt{cp\_i})$.}
\label{table:existsplants}
\centering \scriptsize
\begin{tabular}{|l|*{9}{c|}}

\hline
\texttt{remedies} & \begin{sideways}\textbf{$\exists$ r(cp\_2)}\end{sideways} & \begin{sideways}\textbf{$\exists$ r(cp\_0)}\end{sideways} & \begin{sideways}\textbf{$\exists$ r(cp\_4)}\end{sideways} & \begin{sideways}\textbf{$\exists$ r(cp\_3)}\end{sideways} & \begin{sideways}\textbf{$\exists$ r(cp\_5)}\end{sideways} & \begin{sideways}\textbf{$\exists$ r(cp\_7)}\end{sideways} & \begin{sideways}\textbf{$\exists$ r(cp\_8)}\end{sideways} & \begin{sideways}\textbf{$\exists$ r(cp\_1)}\end{sideways} & \begin{sideways}\textbf{$\exists$ r(cp\_6)}\end{sideways}\\
\hline
\textbf{remedy0}&  & $\times$ & $\times$ & $\times$ & $\times$ & $\times$ & $\times$  &  & \\
\hline
\textbf{remedy1} &  & $\times$ & $\times$ & $\times$ & $\times$ & $\times$ & $\times$  & & $\times$ \\
\hline
\textbf{remedy2}&  $\times$ & $\times$ & $\times$ & $\times$ & $\times$ & $\times$ & $\times$  & & $\times$ \\
\hline
\textbf{remedy3} &   $\times$  &  &  & & $\times$  & & $\times$ & $\times$ & $\times$ \\
\hline
\end{tabular}
\end{table}

%% file: Remedes__AOCuniversel.tex
\begin{table}[htb]
    \centering
   \caption{Strict universal scaling of formal context of Remedies with AOC-poset of Plants $\mathbb{S}_{\exists\forall}(K_{Remedies}, r_\mathit{contains},{\cal C}_{Plants})$. $\exists\forall r(\texttt{cp\_i})$ abbreviates $\exists\forall {r_\mathit{contains}}(\texttt{cp\_i})$.}
\label{table:forallplants}
\scriptsize

\begin{tabular}{|l|*{9}{c|}}

\hline
\texttt{remedies} &  \begin{sideways}\textbf{$\exists\forall$ r(cp\_2)}\end{sideways} & \begin{sideways}\textbf{$\exists\forall$ r(cp\_0)}\end{sideways} & \begin{sideways}\textbf{$\exists\forall$ r(cp\_4)}\end{sideways} & \begin{sideways}\textbf{$\exists\forall$ r(cp\_3)}\end{sideways} & \begin{sideways}\textbf{$\exists\forall$ r(cp\_5)}\end{sideways} & \begin{sideways}\textbf{$\exists\forall$ r(cp\_7)}\end{sideways} & \begin{sideways}\textbf{$\exists\forall$ r(cp\_8)}\end{sideways} & \begin{sideways}\textbf{$\exists\forall$ r(cp\_1)}\end{sideways} & \begin{sideways}\textbf{$\exists\forall$ r(cp\_6)}\end{sideways}\\
\hline
\textbf{remedy0}&   &  &  &  & & $\times$ & $\times$  &  & \\
\hline
\textbf{remedy1} &  &  &  &  &  & & $\times$  &  & \\
\hline
\textbf{remedy2}&   &  &  &  &  & & $\times$  &  & \\
\hline
\textbf{remedy3} &   &  &  &  &  & & $\times$  & & $\times$ \\
\hline
\end{tabular}
\end{table}

%% file: tex__convergence.tex
\section{Convergence issues in RCA-AOC}
\label{sec:convergence}

The convergence of an iterative process is the property of reaching a stable solution after a finite number of steps.
This notion is fundamental, as it enables the design of practical stopping criteria that can be effectively verified.
The stop condition for RCA and RCA-AOC is the equivalence of concept-posets (concept lattice or AOC-poset) associated with the same initial context between two successive iterations. 
Two concept-posets are equivalent if the sets of their extents are equal, as the equivalence of two concepts 
is the equality of their extents.

 Convergence is ensured with the RCA specification from \cite{rouane2013} where the set of concepts used at each step for building relational attributes is the set of concepts of the whole lattice. Besides, it makes it possible to interpret the last built lattices using only the last step lattices, while forgetting the lattices of the previous steps. 
By interpretation, we mean that every concept occurring in the expression of a relational attribute at the last step is present in one of the lattices obtained at this step. As a consequence, we can interpret these attributes recursively, as done by \cite{DBLP:conf/concepts/GutierrezHMZ25}, or extract a graph representation, as in \citep{nica-kbs2020,nica-dam2020}, without encountering any concept that is not defined at this step.
Unfortunately, RCA-AOC may not converge in the general case. 
In this section, we illustrate this divergence through several examples based on the two scaling operators  $\exists$ and $\exists\forall$. The existential scaling operator is the most commonly used in practice, and the strict universal scaling operator is a natural alternative for discovering more restricted concepts.
The first example (Sect. \ref{sec_uml}) is a concrete situation that may arise when applying RCA-AOC in the software engineering domain, more specifically when using it as a tool to support conceptual model refactoring. The next two examples (Sect. \ref{sec:counterexamples}) are small formal cases that highlight typical divergence patterns for $\exists$ and $\exists\forall$ operators respectively.

\input{tex__umlBank}

\subsection{Diverging examples}\label{sec:counterexamples}

In this section, we present two small formal examples {highlighting the mechanism of divergence for two quantifiers}.

\subsubsection{Existential scaling}

This example reduces the UML divergence example 
to its essence. 
The RCF  is presented in Table \ref{tab:rcf-exists}, and the first steps of the process are illustrated by Table \ref{tab:gsh-exists} with all the created AOC-posets.

\begin{table}[htb!]
\centering
\caption{RCF for the counterexample illustrating divergence using the existential scaling operator ($\exists$) on each relation. \label{tab:rcf-exists}}
\footnotesize
\begin{tabular}{|*{2}{@{\hspace{0.07cm}}c@{\hspace{0.07cm}}|}}
\hline
Formal Contexts&Relations\\
\hline
\rule[-0.7cm]{0pt}{0cm}
\begin{tabular}{|*{1}{@{\hspace{0.07cm}}c@{\hspace{0.07cm}}|}}
\hline
\textbf{$K_1$}\\
\hline
o1\\

o2\\
\hline
\end{tabular}
\begin{tabular}{|*{3}{@{\hspace{0.07cm}}c@{\hspace{0.07cm}}|}}
\hline
\textbf{$K_2$}&a1&a2\\
\hline
o3&$\times$&\\

o4&&$\times$\\
\hline
\end{tabular}
\begin{tabular}{|*{3}{@{\hspace{0.07cm}}c@{\hspace{0.07cm}}|}}
\hline
\textbf{$K_3$}&a3&a4\\
\hline
o5&$\times$&\\

o6&&$\times$\\
\hline
\end{tabular}
\begin{tabular}{|*{3}{@{\hspace{0.07cm}}c@{\hspace{0.07cm}}|}}
\hline
\textbf{$K_4$}&a5&a6\\
\hline
o7&$\times$&\\

o8&&$\times$\\
\hline
\end{tabular}
&
\begin{tabular}{|*{3}{@{\hspace{0.07cm}}c@{\hspace{0.07cm}}|}}
\hline
$r_1$ &o3&o4\\
\hline
o1     &$\times$&   \\
o2     & &$\times$  \\
\hline
\end{tabular}
\begin{tabular}{|*{3}{@{\hspace{0.07cm}}c@{\hspace{0.07cm}}|}}
\hline
$r_2$ &o1&o2\\
\hline
o7      &$\times$&   \\
o8      & &$\times$  \\
\hline
\end{tabular}
\begin{tabular}{|*{3}{@{\hspace{0.07cm}}c@{\hspace{0.07cm}}|}}
\hline
$r_3$ &o7&o8\\
\hline
o5      &$\times$&  \\
o6      & &$\times$   \\
\hline
\end{tabular}
\begin{tabular}{|*{3}{@{\hspace{0.07cm}}c@{\hspace{0.07cm}}|}}
\hline
$r_4$ &o5&o6\\
\hline
o7      &$\times$&  \\
o8      & &$\times$  \\
\hline
\end{tabular}
\rule{0pt}{0.8cm}
\\
\hline
\end{tabular}
\end{table}

\begin{table}[ht]
\centering
\caption{AOC-posets obtained from applying RCA-AOC on the RCF from Table \ref{tab:rcf-exists} using the existential scaling operator ($\exists$) on each relation. 
For each step, AOC-posets come, from left to right, from $K_1$, $K_2$, $K_3$ and $K_4$.\label{tab:gsh-exists}}

\begin{tabular}{|c|cccc|}
\hline
\begin{sideways}step 0\end{sideways}&
\includegraphics[scale=0.35]{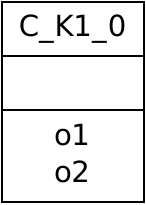}&
\includegraphics[scale=0.35]{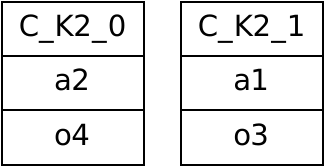}&
\includegraphics[scale=0.35]{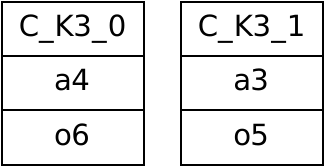}&
\includegraphics[scale=0.35]{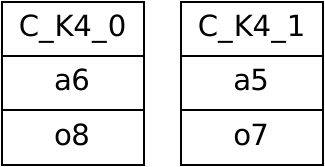}
\\ \hline
\begin{sideways}step 1\end{sideways}&

\includegraphics[scale=0.35]{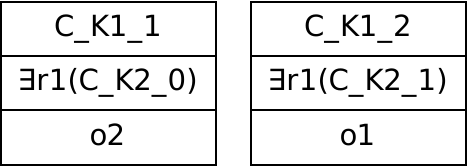}&
\includegraphics[scale=0.35]{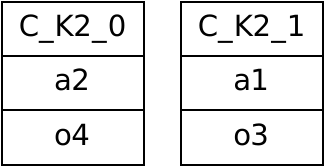}&
\includegraphics[scale=0.35]{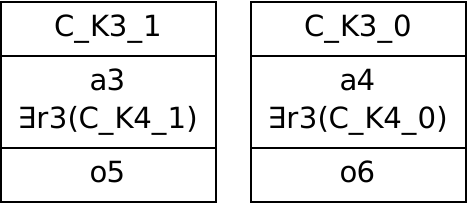}&
\includegraphics[scale=0.35]{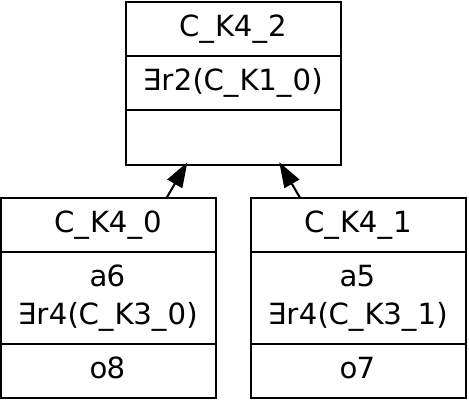}
\\ \hline
\begin{sideways}step 2\end{sideways}&
\includegraphics[scale=0.35]{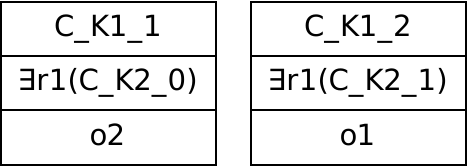}&
\includegraphics[scale=0.35]{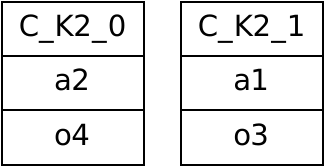}&
\includegraphics[scale=0.35]{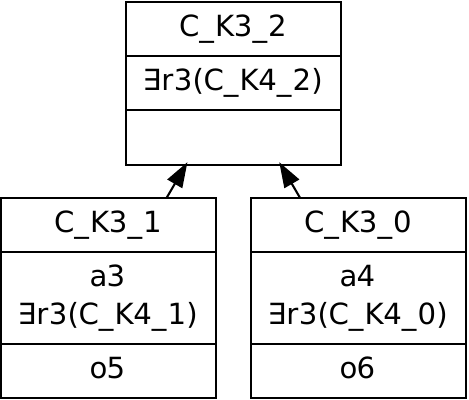}&
\includegraphics[scale=0.35]{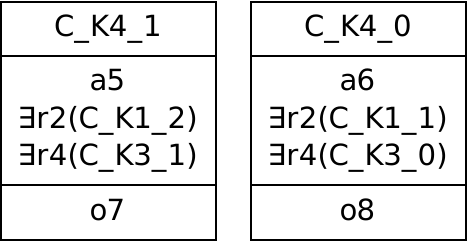}
\\ \hline
\begin{sideways}step 3\end{sideways}&
\includegraphics[scale=0.35]{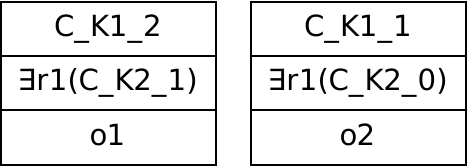}&
\includegraphics[scale=0.35]{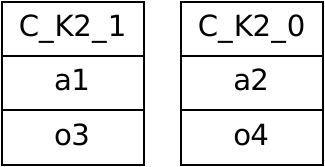}&
\includegraphics[scale=0.35]{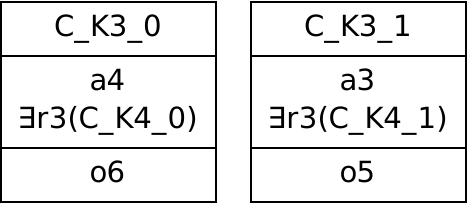}&
\includegraphics[scale=0.35]{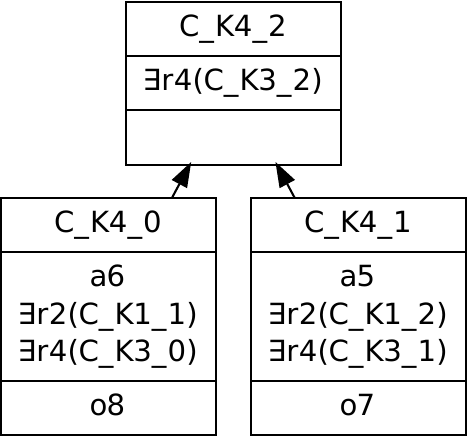}
\\ \hline
\end{tabular}
\end{table}

Let us look at the AOC-posets built during the process and shown in Table~\ref{tab:gsh-exists}.
To help the reader, a concept name contains the context from which the concept is built (e.g. C\_K1\_0 is the first concept built from the $K_1$ context) and the same name is reused for the same concept in the following steps.
In step 0, seven concepts are created, one for context $K_1$, and two in the other contexts.
In step 1, C\_K4\_2 is created depending on C\_K1\_0 from step 0 and contains no object in its simplified extent. So as C\_K1\_0 is removed in step 1, it leads to the removal of C\_K4\_2 in step 2 with no creation of new concepts in this AOC-poset.  However C\_K3\_2 has been created in step 2 because of the presence of C\_K4\_2 in step 1, and has an empty simplified extent. In step 3, C\_K3\_2 is thus removed, and no new concept is created in this AOC-poset. 
{In the same step,  C\_K4\_2 reappears because of C\_K3\_2: we have an inter-dependency between two concepts in a way that they will appear alternately. Step 3 is equivalent to step 1, step 4 will be equivalent to step 2, step 5 to step 3, and so on. This leads to an infinite loop.}

\subsubsection{Strict Universal scaling}

 Table \ref{tab:rcf-forall} presents an RCF which, if we use the strict universal operator for every relation, makes the RCA-AOC process loop infinitely. The first steps are detailed in Table \ref{tab:gsh-forall} by showing the AOC-posets obtained. The divergence of the process appears at step 4 as we obtain AOC-posets equivalent to those of step 0, which means that step 5 will be equivalent to step 1, etc.

\begin{table}[htb!]
\centering
\caption{RCF for the counterexample illustrating divergence using the strict universal scaling operator ($\exists\forall$) on each relation. \label{tab:rcf-forall}}
\footnotesize
\begin{tabular}{|*{2}{@{\hspace{0.1cm}}c@{\hspace{0.1cm}}|}}
\hline
Formal Contexts&Relations\\
\hline
\rule[-0.65cm]{0pt}{0cm}
\begin{tabular}{|*{2}{@{\hspace{0.1cm}}c@{\hspace{0.1cm}}|}}
\hline
\textbf{$K_1$}&a1\\
\hline
o1&\\

o2&\\
\hline
\end{tabular}
\begin{tabular}{|*{2}{@{\hspace{0.1cm}}c@{\hspace{0.1cm}}|}}
\hline
\textbf{$K_2$}&a2\\
\hline
o3&\\

o4&\\
\hline
\end{tabular}
\begin{tabular}{|*{2}{@{\hspace{0.1cm}}c@{\hspace{0.1cm}}|}}
\hline
\textbf{$K_3$}&a3\\
\hline
o5&\\

o6&\\
\hline
\end{tabular}
&
\begin{tabular}{|*{3}{@{\hspace{0.1cm}}c@{\hspace{0.1cm}}|}}
\hline
$R_1$ &o5&o6\\
\hline
o1      &$\times$&   \\
o2      & &$\times$  \\
\hline
\end{tabular}
\begin{tabular}{|*{3}{@{\hspace{0.1cm}}c@{\hspace{0.1cm}}|}}
\hline
$R_2$ &o5&o6\\
\hline
o3      &$\times$&   \\
o4      & &$\times$  \\
\hline
\end{tabular}
\begin{tabular}{|*{3}{@{\hspace{0.1cm}}c@{\hspace{0.1cm}}|}}
\hline
$R_3$ &o1&o2\\
\hline
o5      &$\times$&$\times$  \\
o6      & &   \\
\hline
\end{tabular}
\begin{tabular}{|*{3}{@{\hspace{0.1cm}}c@{\hspace{0.1cm}}|}}
\hline
$R_4$ &o3&o4\\
\hline
o5     & &  \\
o6      &$\times$&$\times$  \\
\hline
\end{tabular}
\rule{0pt}{0.8cm}
\\

\hline
\end{tabular}
\end{table}

\begin{table}[htbp]
\centering
\caption{AOC-posets obtained from applying RCA-AOC on the relational context family from  Table \ref{tab:rcf-forall} using the strict universal scaling operator ($\exists\forall$) on each relation. For each step, AOC-posets come, from left to right, from $K_1$, $K_2$ and $K_3$. \label{tab:gsh-forall}}

\begin{tabular}{|c|ccc|}
\hline
\begin{sideways}step 0\end{sideways}&
\includegraphics[scale=0.35]{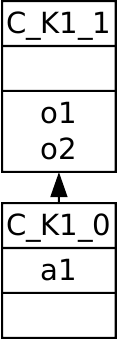}&
\includegraphics[scale=0.35]{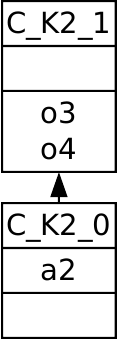}&
\includegraphics[scale=0.35]{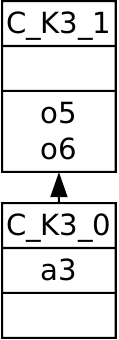}
\\ \hline
\begin{sideways}step 1 \end{sideways}&
\includegraphics[scale=0.35]{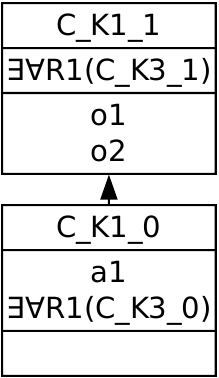}&
\includegraphics[scale=0.35]{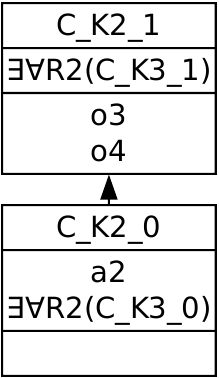}&
\includegraphics[scale=0.35]{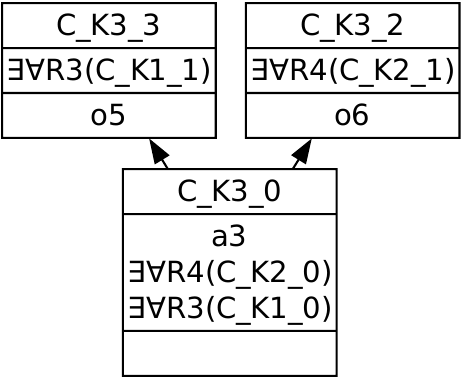}
\\ \hline
\begin{sideways}step 2\end{sideways}&
\includegraphics[scale=0.35]{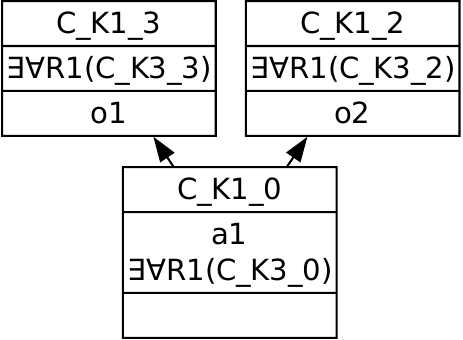}&
\includegraphics[scale=0.35]{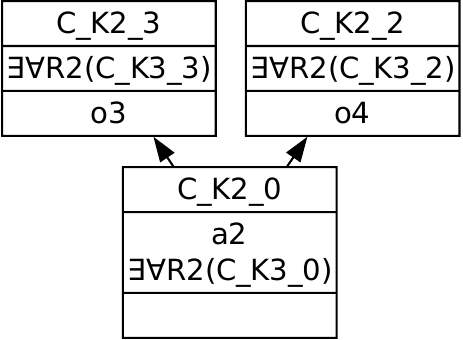}&
\includegraphics[scale=0.35]{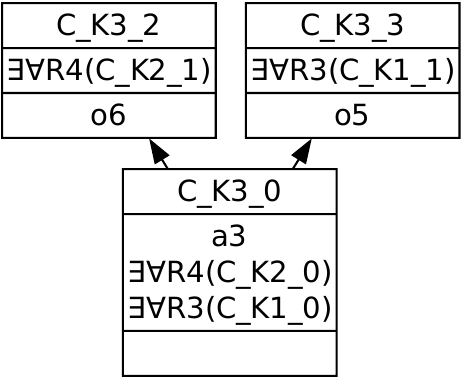}
\\ \hline
\begin{sideways}step 3\end{sideways}&
\includegraphics[scale=0.35]{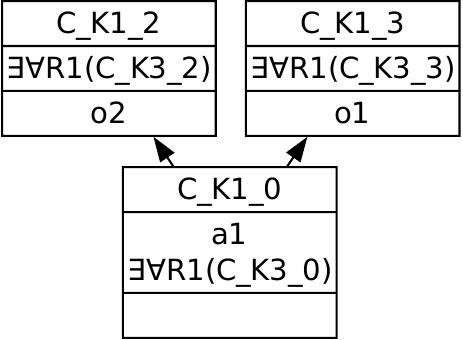}&
\includegraphics[scale=0.35]{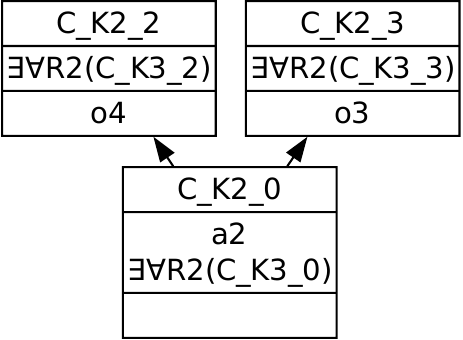}&
\includegraphics[scale=0.35]{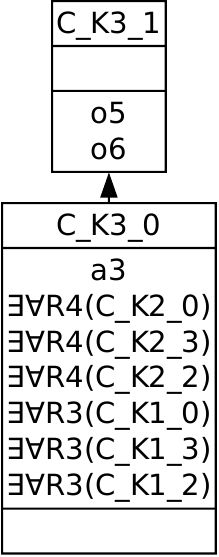}
\\ \hline
\begin{sideways}step 4\end{sideways}&
\includegraphics[scale=0.35]{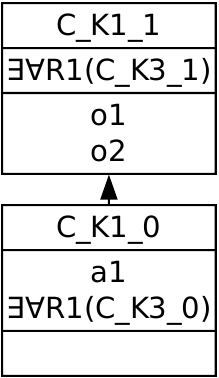}&
\includegraphics[scale=0.35]{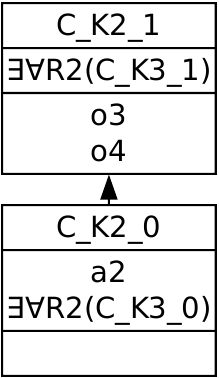}&
\includegraphics[scale=0.35]{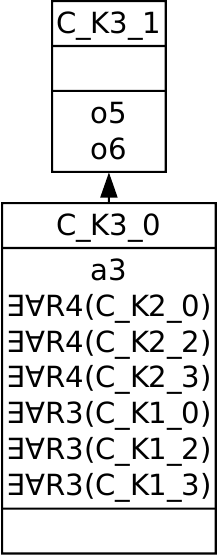}\\
\hline
\end{tabular}
\end{table}

Let us look at the AOC-posets built during the process and shown in Table~\ref{tab:gsh-forall}. Concepts are named as previously.
The existence of C\_K3\_2 and C\_K3\_3 (which appear at step 1) depends on the existence of C\_K2\_1 and C\_K1\_1. If C\_K2\_1 and C\_K1\_1 are not present at the previous step then C\_K3\_2 and C\_K3\_3 cannot be created as we see at steps 3 and 4. Instead their extents (composed respectively of o6 and o5) are contained in the extent of C\_K3\_1, that cannot appear simultaneously with C\_K3\_2 and C\_K3\_3 in an AOC-poset.

On the other hand, C\_K1\_2, C\_K1\_3, C\_K2\_2 and C\_K2\_3 (steps 2 and 3) depend on the existence of C\_K3\_2 and C\_K3\_3 at the previous step. When C\_K1\_2, C\_K1\_3, C\_K2\_2 and C\_K2\_3 appear, C\_K1\_1 and C\_K2\_1 cannot appear at the same step.
Those inter-dependencies lead to the divergence of the process. 

These two counterexamples show that convergence in RCA-AOC is not
always ensured. We present in the following some conditions to ensure it.

%% file: tex__umlBank.tex
\subsection{A concrete divergent example}
\label{sec_uml}

Divergence may arise in various concrete situations. In this section, we consider an application of FCA and RCA to software engineering. This application, developed over the years \citep{Godi93a,DBLP:conf/iccs/DaoHHRV04,HuchardHRV07,DBLP:conf/cla/MirallesMHNDD15,DBLP:conf/concepts/GuenouneGHLMMZ25}, aims to improve the abstraction level of object-oriented code or conceptual models (e.g. database schemas, ER diagrams, or UML class models) by adding new elements (classes, attributes, operations, etc.) that factor out common information and represent more general notions. The example we focus on is taken from the UML world. It is deliberately simple and only aims to illustrate concretely the possibility of divergence in this process.

A UML class model is primarily composed of classes that represent notions (concepts) relevant to the software being developed. For example, in software dedicated to the banking domain, one might introduce a class \texttt{BankAccount}. Additional information is attached to these classes, such as properties (also called attributes) and operations (also called methods). For instance, the class \texttt{BankAccount} can be enriched with a property \texttt{accountNum\-ber} and an operation \texttt{withdraw}. These elements are themselves described by further information, such as the type of a property (e.g. \texttt{String}) or the parameters of an operation (e.g. \texttt{amount} of type \texttt{Real}).

The UML metamodel, as defined by the OMG \citep{OMG-UML}, specifies all the entities that may appear in a UML model, and in particular in a class model.
We here rely on an excerpt of this metamodel, sketched in Fig. \ref{fig_bank_mm_uml1}.
A class owns operations and properties. 
Classes and properties are described by their name. In addition, properties have a type.  An operation owns parameters. Each parameter has a name, a direction, where \texttt{ParameterDirectionKind} is \texttt{in, inout, out, or return}, and a type. 
{This excerpt deliberately omits some meta-classes and meta-attributes of the UML metamodel — notably the name meta-attribute of Operation — so that the two operations are indistinguishable at first and are initially grouped into a single concept.}

\begin{minipage}[htb]{0.45\textwidth}
  \centering
  \includegraphics[width=\textwidth]{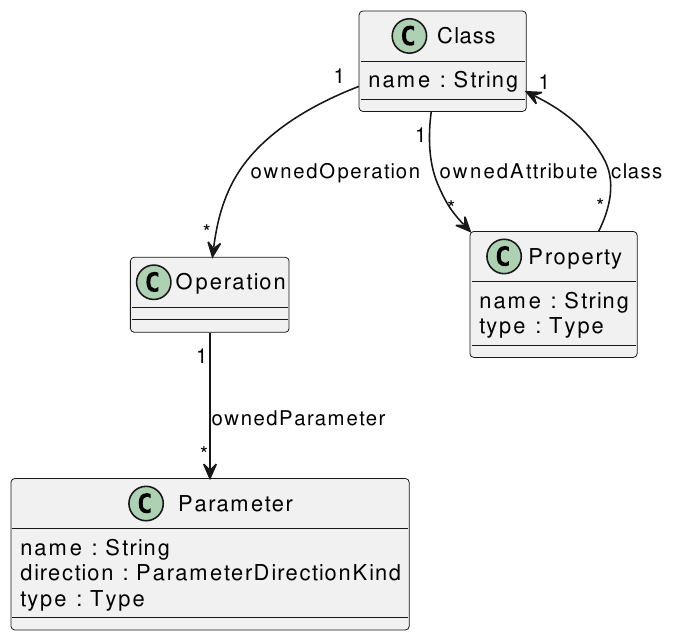}
  \captionof{figure}{Excerpt of the OMG UML metamodel. Arrows represent UML associations. Since we are describing a metamodel, the various elements (classes, associations, etc.) are referred to as meta-elements.}
  \label{fig_bank_mm_uml1}
\end{minipage}
\hfill
\begin{minipage}[htb]{0.45\textwidth}
  \centering
   \includegraphics[width=\textwidth]{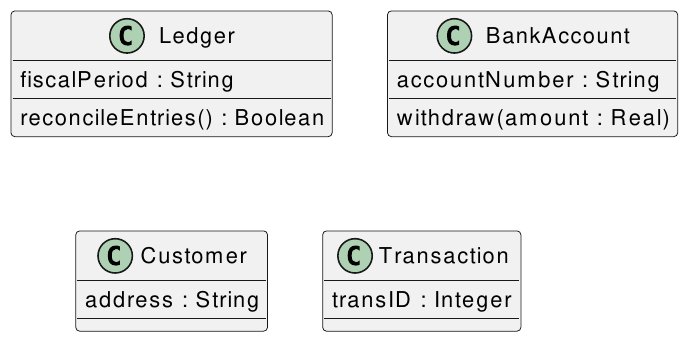}
   \vspace{0.2cm}
    \hrule
    \vspace{0.5cm}
   \includegraphics[width=\textwidth]{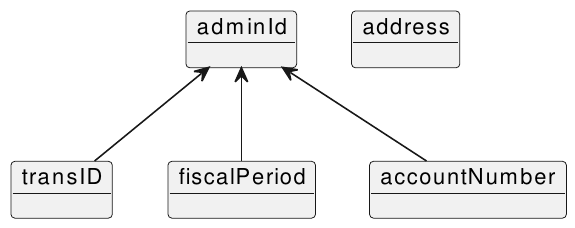}
  \captionof{figure}{Top: A simple UML model in the Bank domain. Bottom: a lexical resource. Arrows represent a hyponymy relation, pointing from the hyponym to the hypernym.
  }
  \label{fig_bank_m_uml2}
\end{minipage}
\vspace{.5cm}

An example of a succinct UML class model in the banking domain is shown at the top of Fig.~\ref{fig_bank_m_uml2}.
This class diagram (Bank model) comprises four classes, \texttt{Ledger}, \texttt{BankAccount}, \texttt{Customer} and \texttt{Transaction}.
\texttt{BankAccount} has been described previously. 
\texttt{Ledger} owns the property \texttt{fiscalPeriod} and the operation \texttt{reconcileEntries} which returns a Boolean. \texttt{Customer} is simply 
described by the property \texttt{address}, while \texttt{Transaction} has the property \texttt{transID}. 

To analyze the UML banking class model, for example for refactoring purposes, we need a description language for UML models, that is a UML metamodel. Based on this metamodel, we can encode the  banking class model as a Relational Context Family.
There are various ways of proceeding; a simple one consists in associating: (1) a formal context to each meta-class of the UML metamodel (i.e. \texttt{Class}, \texttt{Operation}, \texttt{Parameter}, \texttt{Property}); (2) a relational context to each meta-association (i.e. \texttt{ow\-ned\-Operation}, \texttt{class}, \texttt{ownedAttribute}, \texttt{ow\-ned\-Parameter}).
This is reflected respectively in the structure of Tables \ref{fc_for_bank1} and \ref{rc_for_bank2}.

\input{tables__rcftBank}

The formal and relational contexts are then populated with the elements of the UML model, here the simple Bank model.
For this step, the UML class model is viewed as an instantiation of the metamodel of Fig. \ref{fig_bank_mm_uml1} (the UML instance diagram is then used).
This instantiation is presented in Fig. \ref{fig_bank_instantiation}. 
It indicates for example that \texttt{Ledger} is a \texttt{Class}, \texttt{fiscalPeriod} is a \texttt{Property},  \texttt{withdraw} is an \texttt{Operation} which has \texttt{amount} as its input  \texttt{Parameter}. 
It is common to also consider additional information, e.g. information taken from a lexical resource that can be added as attributes of the contexts. For example, the bottom of Fig.~\ref{fig_bank_m_uml2} shows a possible lexical resource on attribute names. Arrows represent a hyponymy relation.

\begin{minipage}[t]{\textwidth}
  \centering
  \includegraphics[width=0.5\linewidth]{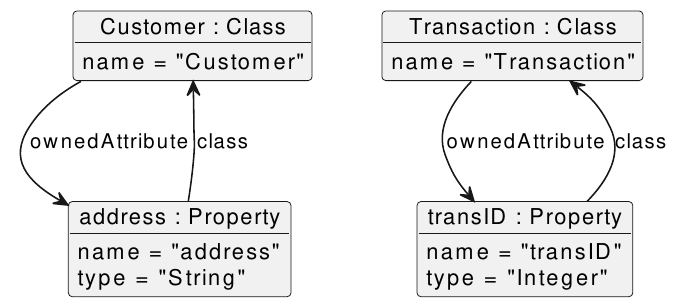}
    \includegraphics[width=\linewidth]{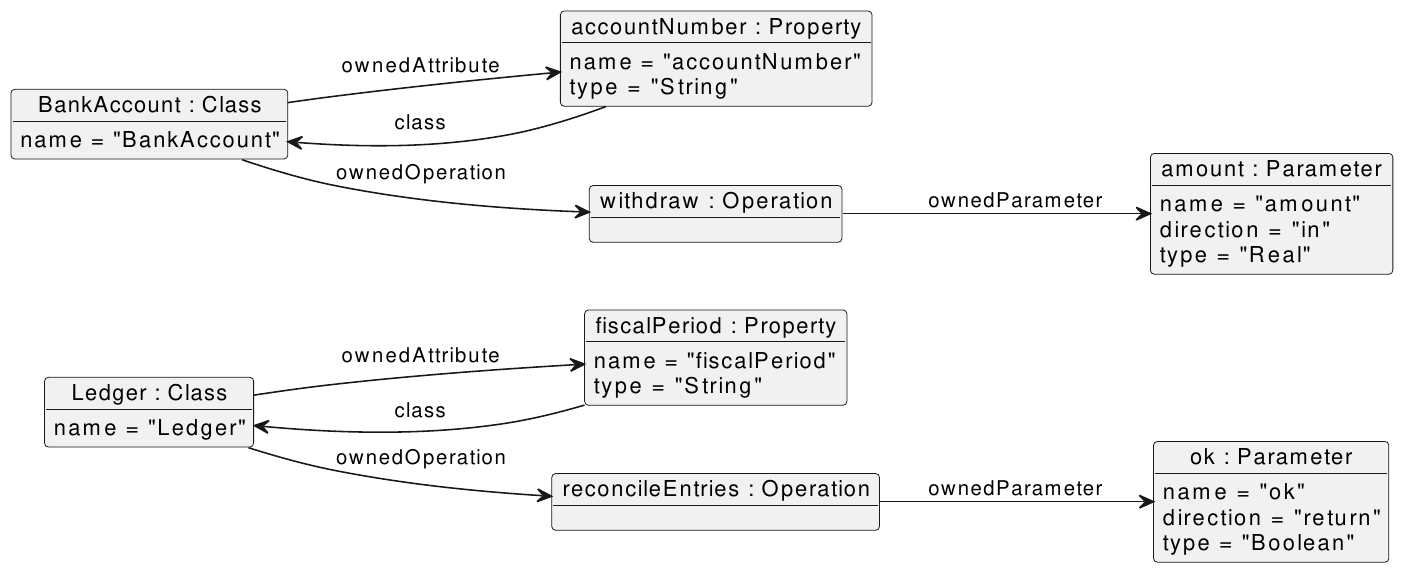}
  \captionof{figure}{The UML model of Fig. \ref{fig_bank_m_uml2} presented as an instantiation of the UML meta\-model of Fig. \ref{fig_bank_mm_uml1}.}
  \label{fig_bank_instantiation}
\end{minipage}

The RCA-AOC process can now be started. In this application, the $\exists$ operator is used, since the refactoring principle consists in adding a new superclass to classes that share at least one element (attribute or operation) appearing in the extent of a concept.

\begin{figure}[htb]
  \centering
    \includegraphics[width=0.8\linewidth]{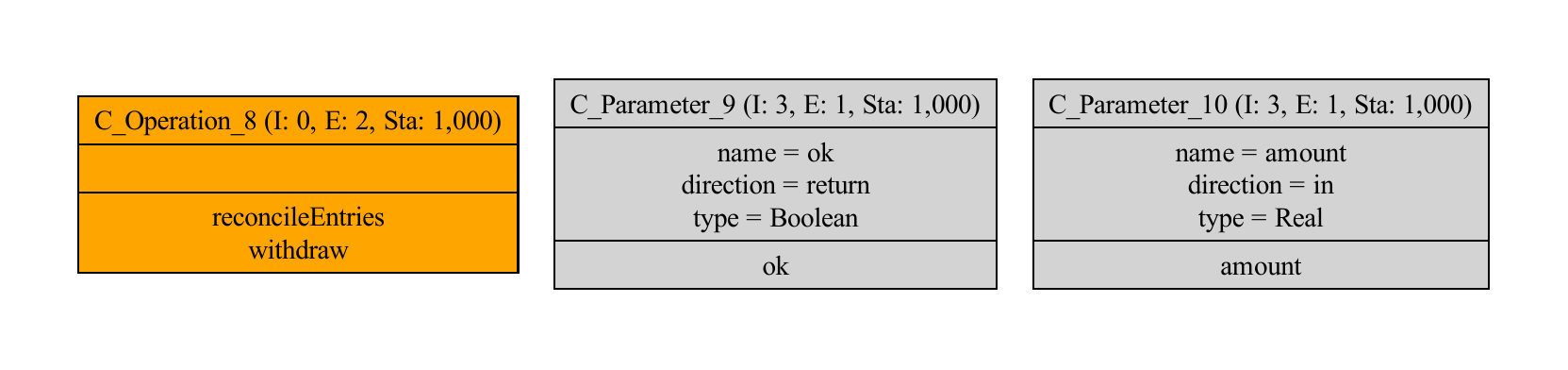}
     \includegraphics[width=\linewidth]{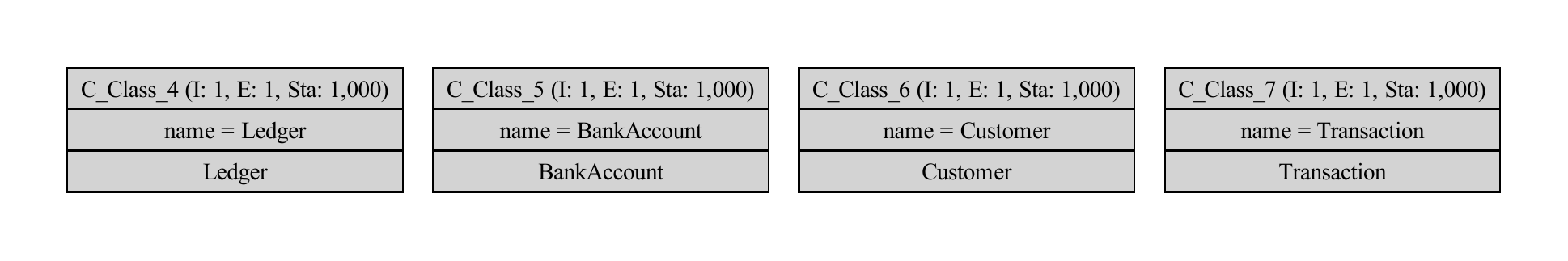}
     \includegraphics[width=\linewidth]{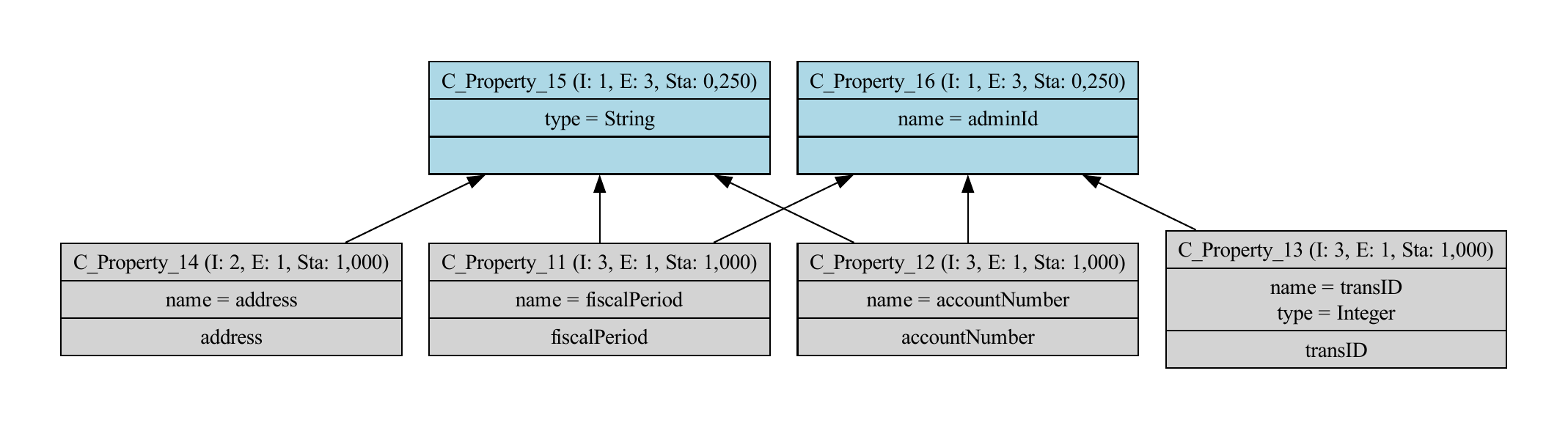}
    \caption{Bank example - AOC-posets at step 0. In the paper, \texttt{C\_Operation\_i} may be shortened to \texttt{co\_i}, \texttt{C\_Parameter\_i} may be shortened to  \texttt{cpa\_i}, \texttt{C\_Class\_i} may be shortened to  \texttt{cc\_i}, and \texttt{C\_Property\_i} may be shortened to \texttt{cpr\_i}.
    For each concept, I, E, and Sta denote the intent cardinality, extent cardinality, and stability metric in decimal comma notation, respectively.}
  \label{fig_bank_step0}
\end{figure}

At step 0 of the RCA-AOC process (Fig. \ref{fig_bank_step0}), both operations are grouped into the operation concept \texttt{co\_8}, as they have no attributes and are indistinguishable at first. Parameters are separated into distinct concepts \texttt{cpa\_9} and  \texttt{cpa\_10}.
Similarly classes are distributed in concepts \texttt{cc\_4} to \texttt{cc\_7}.
The property concept \texttt{cpr\_15} groups properties whose type is \texttt{String}.
The property concept \texttt{cpr\_16} groups properties whose name is a hyponym of \texttt{adminId}.

\begin{figure}[h!]
  \centering
      \includegraphics[width=\linewidth]{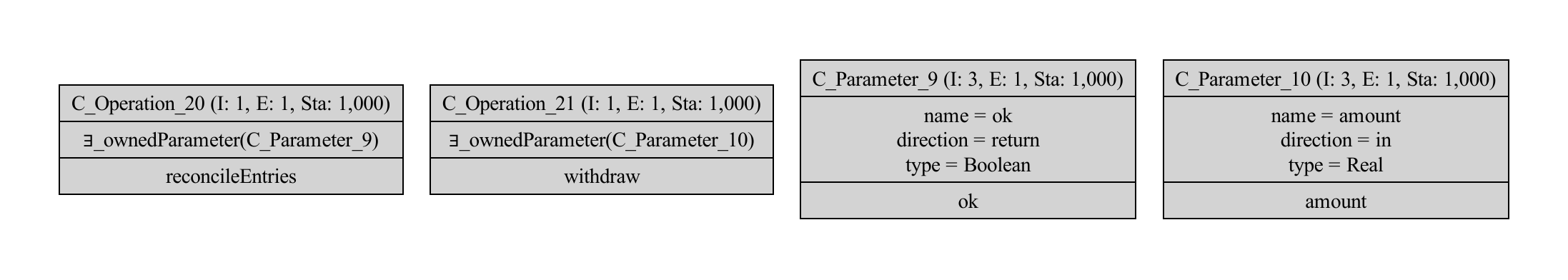}
    \includegraphics[width=0.8\linewidth]{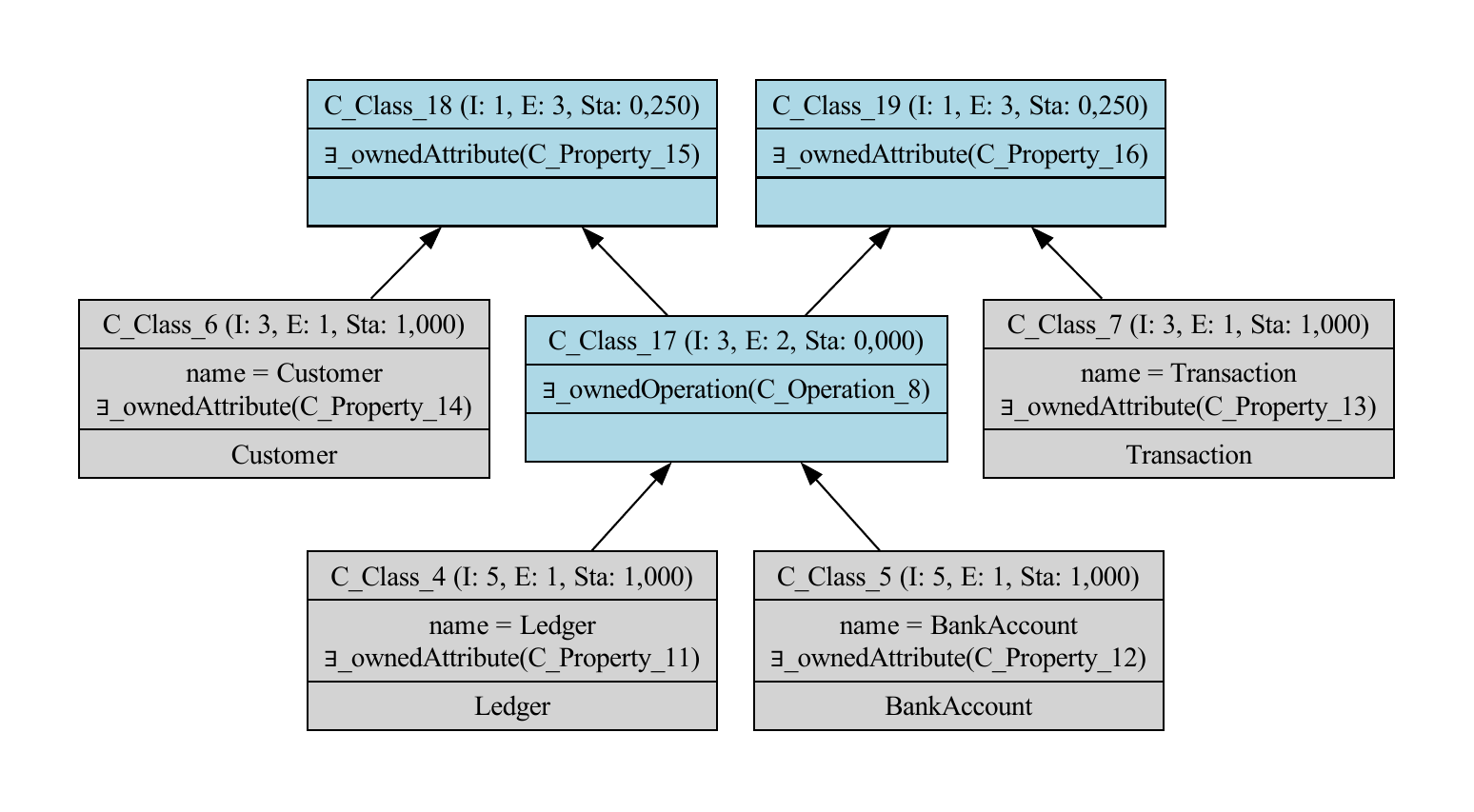}
      \includegraphics[width=\linewidth]{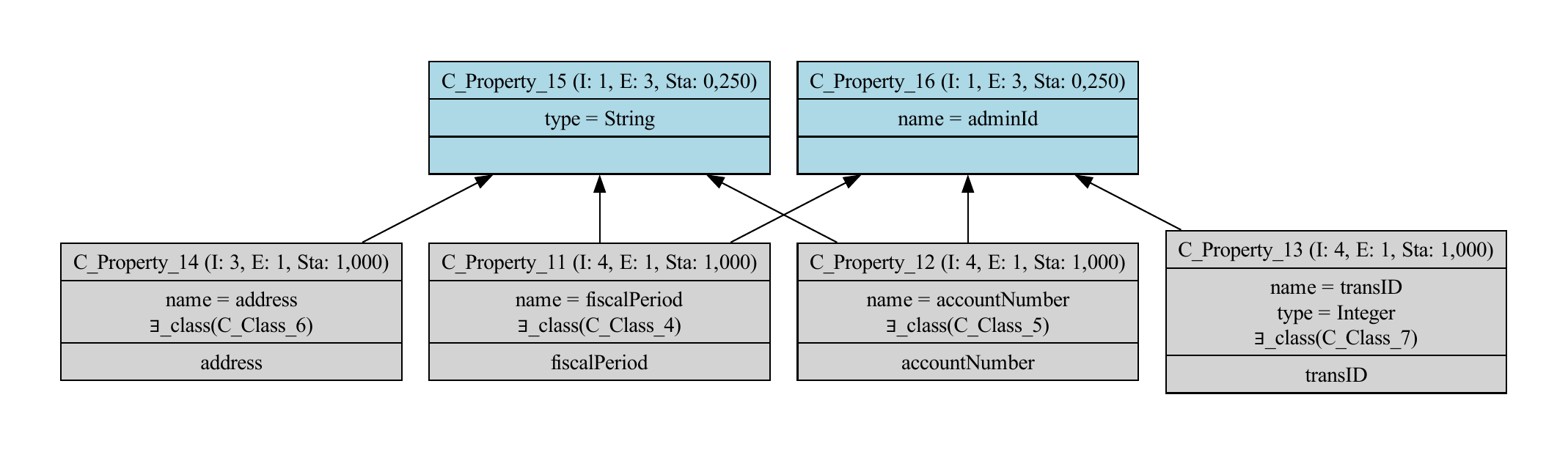}  
    \caption{Bank example - AOC-posets at step 1.}
  \label{fig_bank_step1}
\end{figure}

At step 1 (Fig. \ref{fig_bank_step1}), several  relational attributes are introduced. 
Firstly \texttt{$\exists$\_ownedAttribute(C\_Property\_15)}, aka \texttt{$\exists$\_ownedAttribute(type=String)} if we replace the concept by its intent, is shared by \texttt{Customer}, \texttt{Ledger} and \texttt{BankAccount} and factored out in concept \texttt{cc\_18}.
Secondly, \texttt{$\exists$\_ownedAttribute} \texttt{(cpr\_16)}, aka \texttt{$\exists$\_ownedAttribute(name=adminId)} if we replace the concept by its intent, is shared by \texttt{Transaction}, \texttt{Ledger} and \texttt{BankAccount}  and factored out in concept \texttt{cc\_19}.
Lastly, \texttt{$\exists$\_ownedOperation(co\_8)}  
is shared by \texttt{Ledger} and \texttt{BankAccount} and factored out in concept \texttt{cc\_17}.
Let us note that this last relational attribute
refers to a concept \texttt{co\_8}  absent at this step~1 {by construction of the relational attributes that are based on concepts of the previous step, i.e. step 0}.
Each bottom property concept (\texttt{cpr\_11} to \texttt{cpr\_14}) is completed with a relational attribute pointing to the class that owns the property, e.g. 
\texttt{$\exists$\_class(\texttt{cc\_6})} is added as \texttt{address} is owned by \texttt{Customer}, belonging to  \texttt{cc\_6}  extent.

\begin{figure}[h!]
  \centering
    \includegraphics[width=\linewidth]{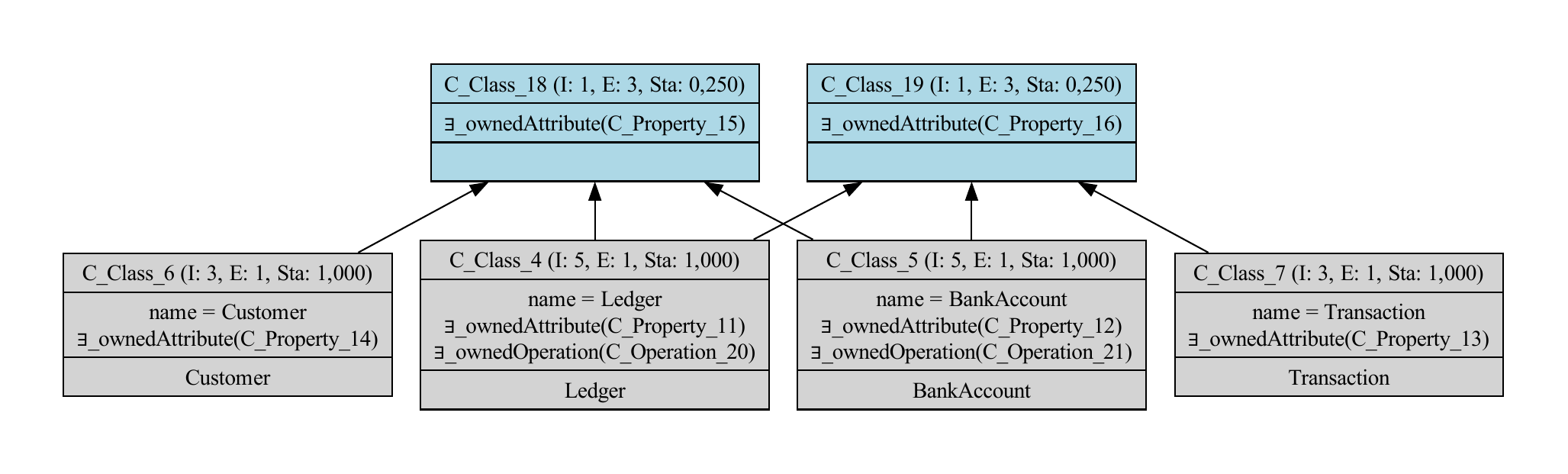}
      \includegraphics[width=0.8\linewidth]{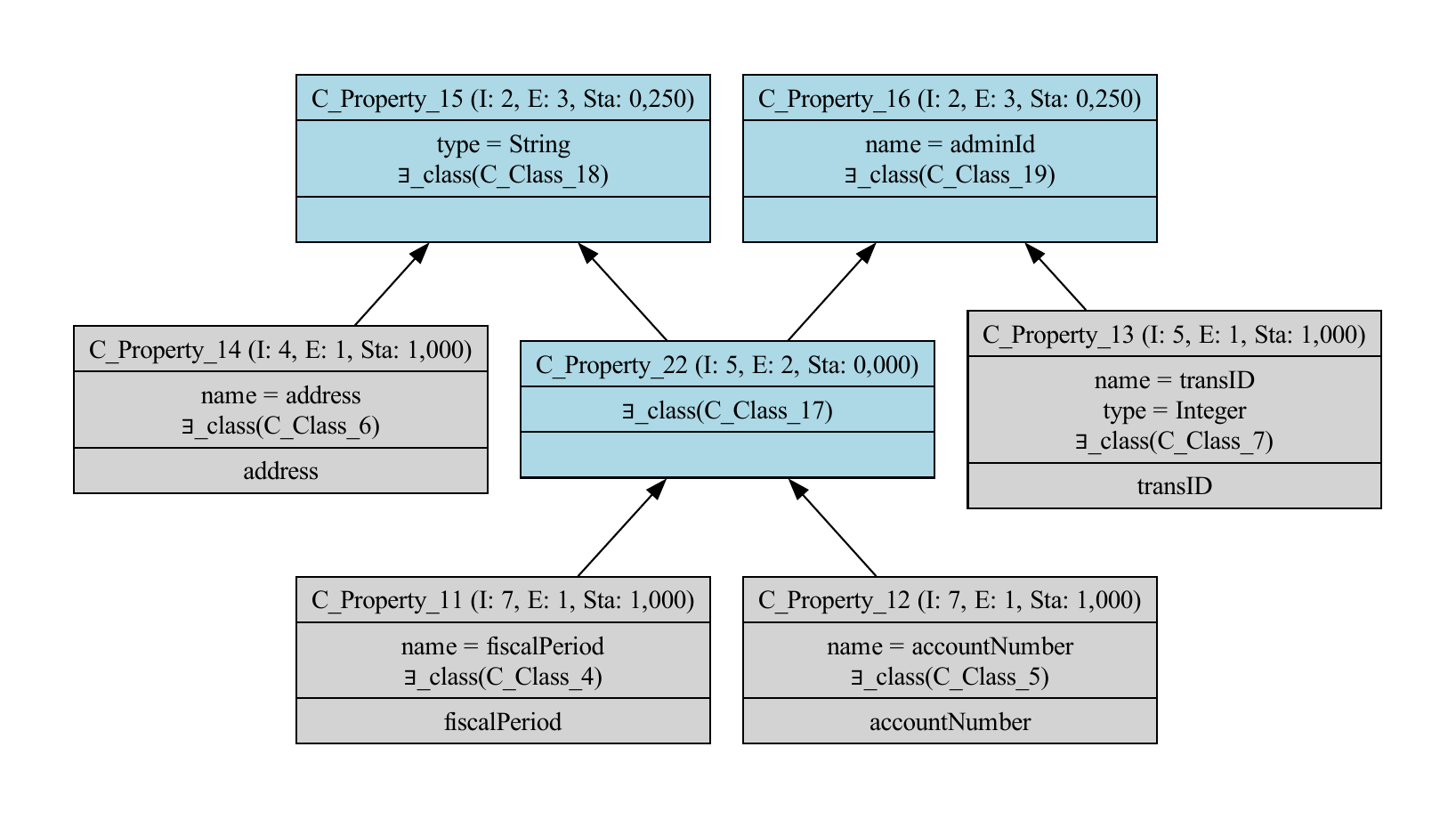}  
    \caption{Bank example - Class and Property AOC-posets at step 2 (Operation and Parameter AOC-posets are unchanged).}
  \label{fig_bank_step2}
\end{figure}

At step 2 (Fig. \ref{fig_bank_step2}), \texttt{cc\_17}, which existed at step 1, no longer exists, as \texttt{co\_8}, a concept that existed at step 0, no longer existed at step 1.
But as \texttt{cc\_17} was present at step 1, it is used to build \texttt{cpr\_22}, to factor out  \texttt{$\exists$\_class(cc\_17)} which is shared by \texttt{fiscalPeriod} and \texttt{accountNumber}.

\begin{figure}[htb]
  \centering
    \includegraphics[width=0.8\linewidth]{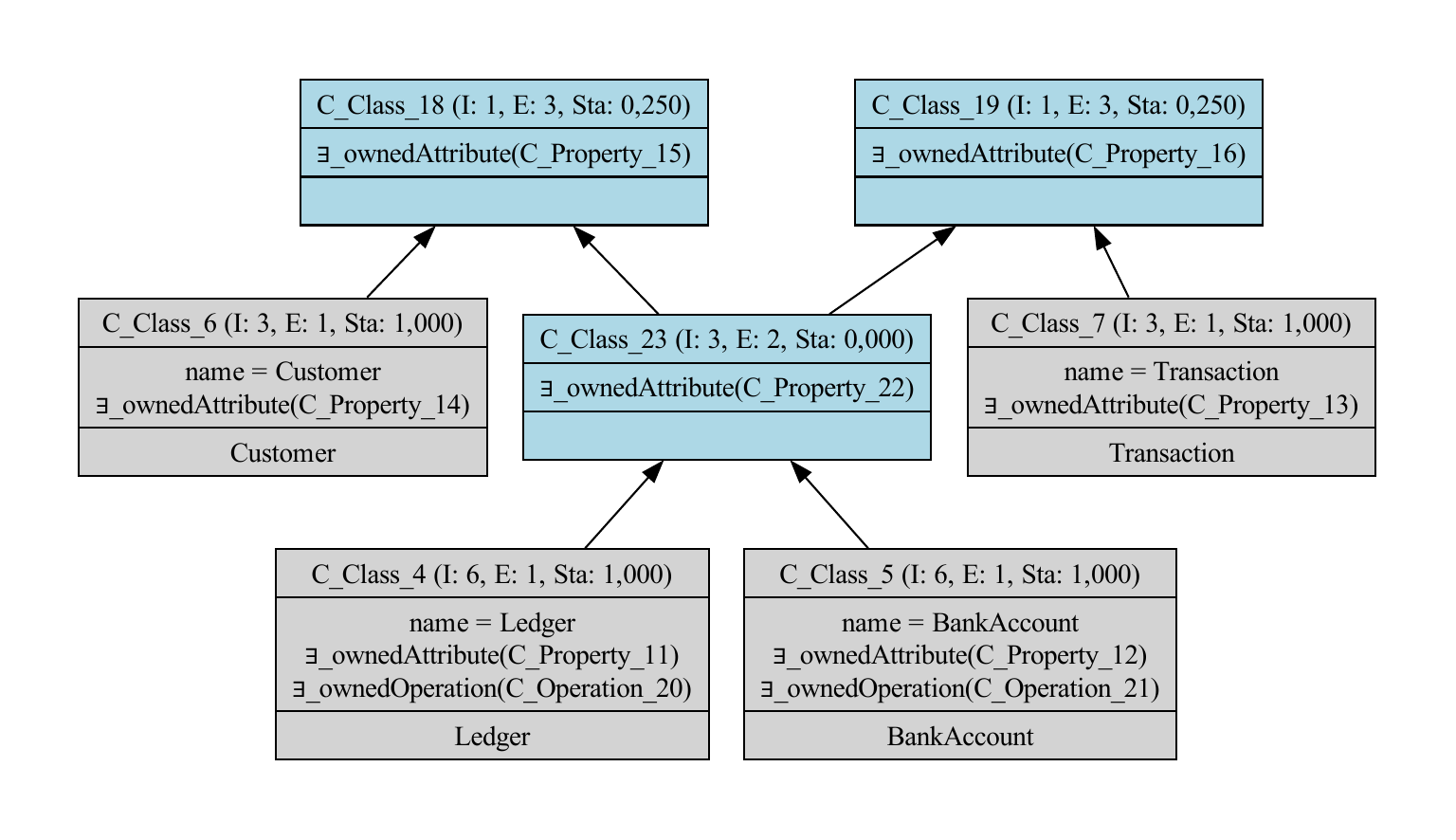}
      \includegraphics[width=\linewidth]{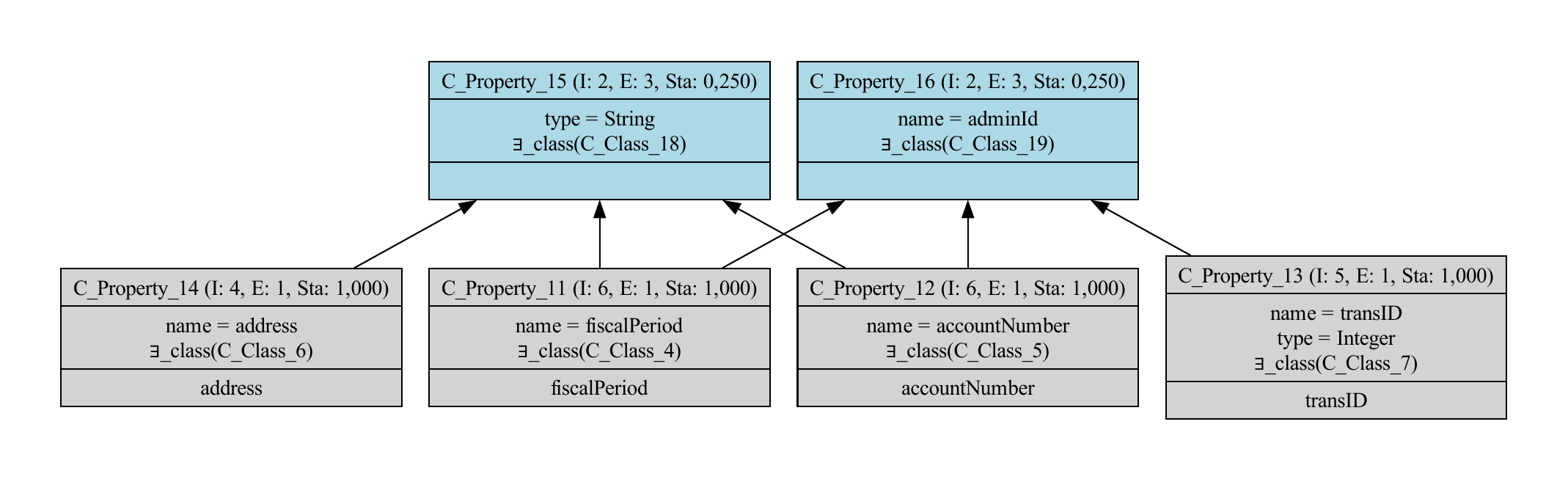}  
    \caption{Bank example - Class and Property AOC-posets at step 3 (Operation and Parameter AOC-posets are unchanged).}
  \label{fig_bank_step3}
\end{figure}

At step 3 (Fig. \ref{fig_bank_step3}), \texttt{cpr\_22} of step 2 is used to build  \texttt{cc\_23}, to factor out  \texttt{$\exists$\_ownedAttribute(cpr\_22)}, which is shared by \texttt{Ledger} and \texttt{BankAccount}.
But  \texttt{cpr\_22} disappears at this step, since concept \texttt{cc\_17} does not exist at step 2.

Class and Property AOC-posets from step 4 are equivalent to those from step 2 because they refer to concepts of step 3 and 1 respectively, that are equivalent (up to a renaming of the concepts). From step 1 to 2, the same concepts have disappeared as from step 3 to 4. These configurations will therefore alternate indefinitely.

The AOC-posets of step 1 and 2 can be used to suggest an evolution of the banking model of Fig. \ref{fig_bank_m_uml2}.
\texttt{cpr\_22} (step 2) suggests introducing a new UML property \texttt{adminId:String} that can be specialized by UML properties \texttt{fiscalPeriod} and \texttt{accountNumber}, using UML constraint syntax \texttt{\{redefines adminId:String\}}. 
\texttt{cc\_17} (step 1) suggests a UML class that a software engineer may call \texttt{FinancialStructure}, on the basis of the factored out property \texttt{adminId:String} and the names of the subclasses \texttt{Ledger} and \texttt{BankAccount}. 
\texttt{cc\_19} (step 1) suggests a UML class \texttt{AdminAsset} to factor out \texttt{adminId}.

The refactored UML class model is shown in Fig. \ref{fig_bankRefactored}.
The names of the new superclasses have to be found by the software engineer, possibly assisted by a lexical resource or a generative AI agent  \citep{DBLP:conf/concepts/GuenouneGHLMMZ25}.

\begin{figure}[htb]
  \centering
     \includegraphics[width=\linewidth]{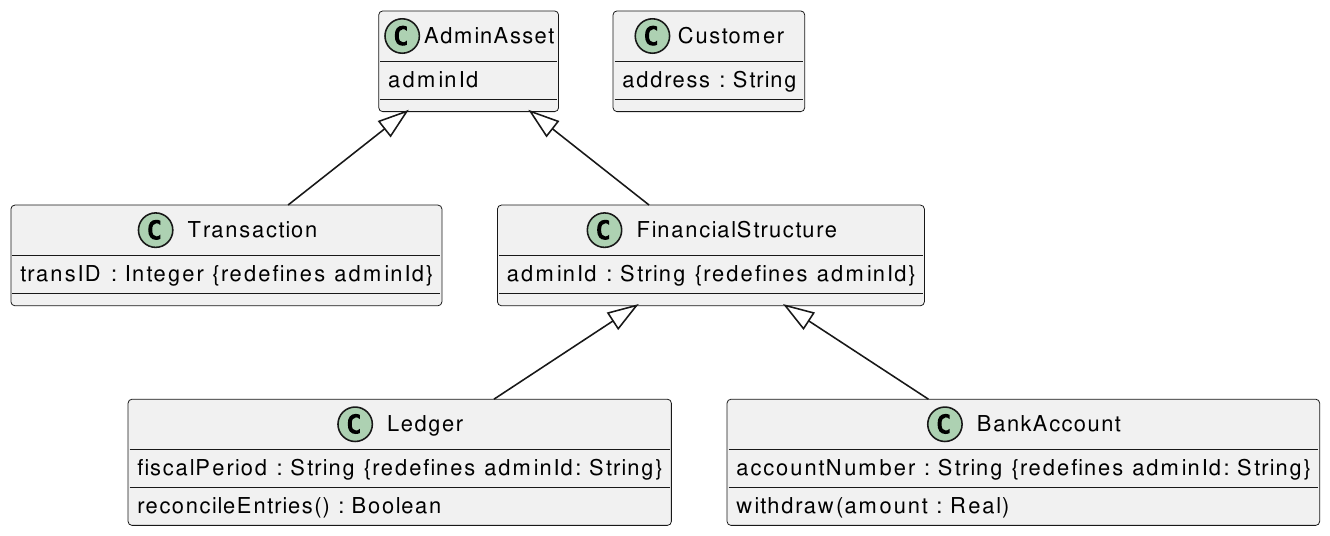}
    \caption{Refactoring suggestions for the Bank model. The two new superclasses factor out the notion of owning an administrative identification (in \texttt{AdminAsset}), and the fact that this identification is a String (in \texttt{FinancialStructure}), and group coherent classes. 
    }
  \label{fig_bankRefactored}
\end{figure}

%% file: tables__rcftBank.tex
\renewcommand{\arraystretch}{1.12}
\begin{table}[htb]
  \centering
  \caption{Formal Contexts of the Bank UML fragment}
  \label{fc_for_bank1}
  \scriptsize
  \begin{minipage}[t]{\textwidth}
    \raggedright

    
\textit{Class}\\[0.35em]
  \begin{tabular}{|l|*{4}{c}|}
  \hline
  
    & name{=} 
    & name{=} 
    & name{=} 
    & name{=}  \\

    & Ledger
    & BankAccount
    & Customer
    & Transaction \\
    \hline
    Ledger      & $\times$ &   &   &   \\

    BankAccount &   & $\times$ &   &   \\

    Customer    &   &   & $\times$ &   \\

    Transaction &   &   &   & $\times$ \\
\hline
  \end{tabular}

    \vspace{0.9em}

    \textit{Operation}\\[0.35em]
      \begin{tabular}{|l|}
        \hline
\\
        \hline
        reconcileEntries \\
        withdraw \\
        \hline
      \end{tabular}
    
  \end{minipage}
  
    \begin{minipage}[t]{\textwidth}
    \raggedright
        \vspace{0.9em}
    
    \textit{Parameter}\\[0.35em]
      \begin{tabular}{|l|*{6}{c}|}
        \hline

        & name{=} & direction{=}& type{=}
        & name{=} & direction{=} & type{=} \\
        &  ok & return &  Boolean
        &  amount &  in &  Real \\
        \hline
        ok      & $\times$& $\times$& $\times$&  &  &  \\
        amount  &  &  &  & $\times$& $\times$& $\times$\\
        \hline
      \end{tabular}

        \vspace{0.9em}
\textit{Property}\\[0.35em]
\begin{tabular}{|l|*{7}{@{\hspace{1.5mm}}c@{\hspace{1.5mm}}}|}
    \hline
    &
    name{=} &
    type{=} &
    name{=} &
    type{=} &
    name{=} &
    name{=} &
    name{=} \\
    &
    transID &
    Integer &
    address &
    String &
    fiscalPeriod &
    accountNumber &
    adminId \\
    \hline
    transID       & $\times$& $\times$&   &   &   &   &  $\times$\\
    address       &   &   & $\times$& $\times$&   &   &   \\
    fiscalPeriod  &   &   &   & $\times$& $\times$&   &  $\times$\\
    accountNumber &   &   &   & $\times$&   & $\times$&  $\times$\\
    \hline
  \end{tabular}

    \end{minipage}
  \end{table}

  \begin{table}[htb]
  \centering
  \caption{Relational Contexts of the Bank UML fragment. For the sake of simplicity, \texttt{ownedAttribute} is the only relational context for which we consider an inverse (\texttt{class}) in this example.}
  \label{rc_for_bank2}
  \scriptsize

  \begin{minipage}[t]{0.49\textwidth}
    \raggedright
    \textit{ownedOperation}\\[0.35em]
    
      \begin{tabular}{|l|*{2}{c}|}
        \hline
        & reconcileEntries & withdraw \\
        \hline
        Ledger      & $\times$&   \\
        BankAccount &   & $\times$\\
        Customer    &   &   \\
        Transaction &   &   \\
        \hline
      \end{tabular}
    
 \end{minipage}
 \begin{minipage}[t]{0.49\textwidth}
    \raggedright
    \textit{ownedParameter}\\[0.35em]
    
      \begin{tabular}{|l|*{2}{c}|}
        \hline
        & ok & amount \\
        \hline
        reconcileEntries & $\times$&   \\
        withdraw         &   & $\times$\\
        \hline
      \end{tabular}
    
\end{minipage}
  
  \vspace{0.9em}
  
    \begin{minipage}[t]{\textwidth}
    \raggedright
    \textit{ownedAttribute}\\[0.35em]
   
      \begin{tabular}{|l|*{4}{c}|}
        \hline
        & transID & address & fiscalPeriod & accountNumber \\
        \hline
        Ledger      &   &   & $\times$&   \\
        BankAccount &   &   &   & $\times$\\
        Customer    &   & $\times$&   &   \\
        Transaction & $\times$&   &   &   \\
        \hline
      \end{tabular}

    \vspace{0.9em}

    \textit{class}\\[0.35em]
    
      \begin{tabular}{|l|*{4}{c}|}
        \hline
        & Ledger & BankAccount & Customer & Transaction \\
        \hline
        transID       &   &   &   & $\times$\\
        address       &   &   & $\times$&   \\
        fiscalPeriod  & $\times$&   &   &   \\
        accountNumber &   & $\times$&   &   \\
        \hline
      \end{tabular}

  \end{minipage}

\end{table}

%% file: tex__approaches_convergence.tex
\section{Approaches to Ensuring RCA-AOC Convergence}
\label{sec_ensuring}

The lack of convergence guarantee is a threat that could discourage users from adopting RCA-AOC. 
Indeed, the occurrence of a divergence case would make the result far more difficult to interpret. To resolve this threat,  and in order to make our method robust we need to find ways to guarantee convergence with as little impact as possible on data.
In this section, we discuss different approaches to ensuring convergence, either by requiring appropriate constraints on the dataset or by introducing a convergent variant of the RCA-AOC process.

\subsection{Preliminaries}
\label{sec_preliminaries}

We introduce here some useful definitions. 
We first define basic notions regarding contexts and concept-posets. We then define the notion of identity of relational attributes, and the convergence for RCA-AOC, that relies on the existence of a fixpoint for each concept-poset of the relational family. Finally we define the notion of dependency graph for an RCF and introduce the notion of  identified objects.

\begin{definition}[Concept and Concept-poset equivalences]
\label{def_equiv}
Let $C_1$ and $C_2$ be two concepts from two different concept-posets $\mathcal{A}_1$ and $\mathcal{A}_2$. $C_1$ and $C_2$ are equivalent, denoted $C_1 \sim C_2$, iff $\mathit{Extent}(C_1)=\mathit{Extent}(C_2)$.
$\mathcal{A}_1$ and $\mathcal{A}_2$ are equivalent iff $Ext_{\mathcal{A}_1}=Ext_{\mathcal{A}_2}$ where $Ext_{\mathcal{A}_1} = \bigcup_{C \in \mathcal{A}_1} \{ \mathit{Extent(C)}\}$ and $Ext_{\mathcal{A}_2} = \bigcup_{C \in \mathcal{A}_2} \{ \mathit{Extent(C)}\}$. We denote it $\mathcal{A}_1 \equiv \mathcal{A}_2$.
\end{definition}

Example: concept-posets from steps 3 and 4 in Table \ref{tab:gsh-forall} are such that  $Ext_{\mathcal{A}_1^3} = \{\{o_1\},\{o_2\},\emptyset\}$, $Ext_{\mathcal{A}_2^3} = \{\{o_3\}, \{o_4\}, \emptyset\}$, $Ext_{\mathcal{A}_3^3} = \{\{o_5, o_6\}, \emptyset\}$ while $Ext_{\mathcal{A}_1^4} = \{\{o_1,o_2\},\emptyset\}$, $Ext_{\mathcal{A}_2^4} = \{\{o_3, o_4\}, \emptyset\}$, and  $Ext_{\mathcal{A}_3^4}$  $= Ext_{\mathcal{A}_3^3}$. 

\begin{definition}[Identity of relational attributes]
\label{def_attridentity}
The identity of a relational attribute $\rho(r)\,r(C)$ is defined as the triple $(\rho(r), r, \mathit{Extent}(C))$. Then, two relational attributes built, possibly at different steps, with the same quantifier and the same relation, and on equivalent concepts, are the same attribute. 
\end{definition}

The incidence of a relational attribute is determined at its creation and is never recomputed afterwards; it only depends on this identifying triple.
E.g. Table \ref{tab:gsh-forall}, the relational attribute $\exists\forall$R2(C\_K3\_1) is the same in C\_K2\_1 at step 1 and step 4. Note that, in this example, equivalent concepts occurring at different steps have the same identifier.

We  now introduce two definitions for context growing, the first one based on context inclusion, the second one based on the inclusion of sets of extents in the associated posets. 
\begin{definition}[Context inclusion and monotonic growth]
Let $\mathcal{K} = (G, M, I)$ and $\mathcal{K}' = (G, M', I')$ be two
contexts. $\mathcal{K} \subseteq \mathcal{K}'$ iff $M \subseteq M'$ and
$I = I' \cap (G \times M)$. Furthermore, let  $(\mathbf{K}, \mathbf{R})$ be an RCF, and  ${\cal K}_i$ a context of $\mathbf{K}$. Let $\mathcal{K}_i^0, \mathcal{K}_i^1, ... \mathcal{K}_i^n, ...$ be the succession of extended formal contexts of $\mathcal{K}_i$ in the RCA-AOC process. The context $\mathcal{K}_i$ is growing monotonically between steps $j$ and $k$ iff $\forall l \in [j..k[$
$\mathcal{K}_i^l \subseteq \mathcal{K}_i^{l+1}$. The context is said to grow monotonically from step $j$ when $k = \infty$. 
\end{definition}

\begin{definition}[Context monotonic poset-growth]
 Let  $(\mathbf{K}, \mathbf{R})$ be an RCF, and  ${\cal K}_i$ a context of $\mathbf{K}$. Let $\mathcal{K}_i^0, \mathcal{K}_i^1, ... \mathcal{K}_i^n, ...$ be the succession of extended formal contexts of $\mathcal{K}_i$ in the RCA-AOC process, and $\mathcal{A}_i^0, \mathcal{A}_i^1, ... \mathcal{A}_i^n, ...$ the corresponding concept-posets. The context $\mathcal{K}_i$ is poset-growing monotonically between steps $j$ and $k$ iff $\forall l \in [j..k[$, $Ext_{\mathcal{A}_i^l}  \subseteq Ext_{\mathcal{A}_i^{l+1}}$ (also denoted $\mathcal{A}_i^l \subseteq \mathcal{A}_i^{l+1}$). The context is said to poset-grow monotonically from step $j$ when $k = \infty$. 
\end{definition}

In the RCA-AOC process, each concept-poset corresponds to a {possibly extended} object-attribute context.  A {concept-poset fixpoint} refers to the final concept-poset generated from this context if it exists. {The verification of the fixpoint relies on the equivalence of successive concept-posets.}

\begin{definition}[Concept-poset fixpoint]\label{def_fixpoint}
 Let  $(\mathbf{K}, \mathbf{R})$ be an RCF, and  ${\cal K}_i$ a context of $\mathbf{K}$. Let $\mathcal{K}_i^0, \mathcal{K}_i^1, ... \mathcal{K}_i^n, ...$ be the succession of extended formal contexts of $\mathcal{K}_i$ in the RCA-AOC process, and $\mathcal{A}_i^0, \mathcal{A}_i^1, ... \mathcal{A}_i^n, ...$ the corresponding concept-posets. If there exists $p$  such that for all $q \geq p$ $\mathcal{A}_i^q \equiv \mathcal{A}_i^{q+1}$, then the context ${\cal K}_i$ admits a concept-poset fixpoint, $\mathcal{A}_i^p$, associated to the extended context $\mathcal{K}_i^p$ of step $p$.
\end{definition}

\begin{definition}[Convergence of RCA-AOC]
The convergence of an application of RCA-AOC on a relational context family $(\mathbf{K}, \mathbf{R})$ is determined by the existence of a fixpoint in the successively generated concept-poset families, i.e. by the existence of a concept-poset fixpoint for each original context ${\cal K}_i\in \mathbf{K}$.
\end{definition}

If all concept-posets reach a fixpoint, then the process can be stopped in practice at step $n$ such that the last concept-poset reaches its own  fixpoint, i.e. 
$\forall i, \mathcal{A}_i^{n} \equiv \mathcal{A}_i^{n-1}$.

\begin{definition}[Dependency graph]
The dependency graph of a relational context family $RCF = (\mathbf{K},
\mathbf{R})$ is defined as the graph $\mathbf{G} = (\mathbf{K},
\mathbf{E})$, with
$\mathbf{E} = \{(\mathcal{K}_{s_r}, \mathcal{K}_{t_r}) \mid r \in
\mathbf{R}\}$.
\end{definition}

The dependency graph shows which contexts have to be considered for the analysis of a given object-object context. 
Figure \ref{fig:ex_gd} shows the dependency graph for the RCF of Table \ref{tab:rcf-forall}.

\begin{figure}
    \centering
    \includegraphics[scale=0.5]{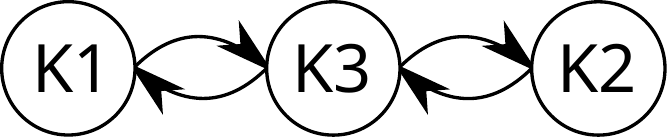}
    \caption{Dependency graph for RCF in Table \ref{tab:rcf-forall}.} 
    \label{fig:ex_gd}
\end{figure}

\begin{definition}[Identified object]
Let ${\cal K} = (G,M,I)$ be an initial object-attribute context from an RCF. An object  $o \in G$ such that  $\{o\}''$=$\{o\}$ is called an identified object. 
\end{definition}

\begin{definition}[Context of identified objects]
A context of identified objects is an object-attribute context $\mathcal{K_{ \mathit{id} }}=(G,M,I)$ where $\forall o \in G, \{o\}''=\{o\}$. As a consequence, concepts of $\mathcal{A_{\mathit{id}}}$ with a non-empty simplified extent have an  extent of size 1.
\end{definition}

Based on these definitions,  
we first examine different configurations of the dependency graph and convergence conditions (Sect. \ref{sec_convconditions}). Then,  we propose a process that ensures convergence (Sect. \ref{sec_convprocess}). Finally, we introduce a variant of RCA-AOC (Sect. \ref{sec_rcaacoconv}).

\subsection{Convergence conditions}
\label{sec_convconditions}

We examine here the conditions of convergence of RCA-AOC, based on the structure of the dependency graph and  some properties of the contexts.
A first lemma states that if a context grows monotonically, then this context reaches a fixpoint.

\begin{lemma} \label{thm:growingcontext}
Let $(\mathbf{K}, \mathbf{R})$ be an RCF, and $\mathcal{K}_i$ a context
of $\mathbf{K}$. Let $\mathcal{K}_i^0, \mathcal{K}_i^1, \ldots,
\mathcal{K}_i^n,$ $\ldots$ be the succession of extended formal contexts
of $\mathcal{K}_i$ in the RCA-AOC process. If there exists $q \in
\mathbb{N}$ from which $\mathcal{K}_i$ grows monotonically, then it admits a concept-poset
fixpoint.
\end{lemma}

This follows directly from the fact that the 
maximal size of the extended context is bounded, and from a certain step it will therefore stop changing. Then the associated  concept-poset will also be stable.

The following  theorem states that, if all the successors in $\mathbf{G}$ of an object-attribute context have a concept-poset fixpoint, then this context has a concept-poset fixpoint too.

\begin{theorem}\label{thm:neighbors}
Let $\mathbf{G}$ be a dependency graph of a relational context family $(\mathbf{K}, \mathbf{R})$.  
Let  $\mathcal{K}_i$ be a context of $\mathbf{K}$. We define $S_i = \{ \mathcal{K}_j \in \mathbf{K} | (\mathcal{K}_i,\mathcal{K}_j)$  is an edge of $\mathbf{G}\}$, the set of direct successors of $\mathcal{K}_i$.
 If all elements of $S_i$ admit a concept-poset fixpoint, 
 then $\mathcal{K}_i$ also admits a concept-poset fixpoint. 
\end{theorem}

This follows directly from the fact that if a  context verifies the previous condition then its corresponding extended formal context will stop changing on the step following the one 
where all its successors in the dependency graph have reached their concept-poset fixpoint\footnote{By Definition~\ref{def_attridentity}, the identity of a relational attribute only depends on the extent of the concept it refers to, thus the relational attributes do not change once their target concept-posets are stable.}.

Based on Theorem \ref{thm:neighbors}, the first sufficient property for the convergence results from the absence of circuits in the dependency graph.
\begin{corollary} \label{thm:nocircuit}
If the dependency graph $\mathbf{G}$ of a relational context family is  without circuit, then the RCA-AOC process will converge. 
\end{corollary}
It is easy and fast to verify: all sink contexts reach their concept-poset fixpoints at the first step since they do not depend on other contexts (they have no successor). Then, applying Theorem \ref{thm:neighbors} recursively, we can propagate this result to the predecessors.

However, in case of a dependency graph with circuits, the monotonic growth of contexts is not ensured, and contexts can lose attributes.
For example, in  Table \ref{tab:gsh-forall}, from step 3 to 4, context $K_2$ loses 2 attributes (pointing to \texttt{C\_K3\_2} and \texttt{C\_K3\_3}), and is extended with one  attribute (pointing to \texttt{C\_K3\_1}). Thus $\mathcal{K}_2^3  \not \subseteq \mathcal{K}_2^{4}$.
We need to add a constraint on a context to ensure its monotonic growth.  Theorem \ref{thm:identifiedobject} states that adding an identifier to each object of a context  can ensure the monotonic growth of the concept-poset. It relies on Lemma \ref{lem:identifiedobject}.

\begin{lemma}\label{lem:identifiedobject}
Adding an attribute to a context of identified objects ${\mathcal{K}_{id}}$ cannot remove concepts in the corresponding concept-poset ${\mathcal{A}_{id}}$ (${\mathcal{A}_{id}}$ grows). 
\end{lemma}

The proof of the lemma is given in Appendix \ref{appendix_the3}. From this lemma we derive the proof of Theorem \ref{thm:identifiedobject}: if all successors of a context of identified objects poset-grow monotonically, then at each step, this context will receive new attributes and its associated poset will grow. Then monotonic growth of ${\mathcal{A}_{id}}$ is ensured.

\begin{theorem}\label{thm:identifiedobject}
Let $(\mathbf{K}, \mathbf{R})$ be an RCF, $\mathbf{G}$ its dependency graph, and ${\mathcal{K}_{id}}$  a context of identified objects in $\mathbf{K}$. Let $S_{id}$ be the set of ${\mathcal{K}_{id}}$'s direct successors in $\mathbf{G}$.
If there exists $n \in \mathbb{N}$ from which  all $\mathcal{K}_j \in S_{id}$  poset-grow monotonically then for all steps $q \geq n+1$, $\mathcal{A}_{id}^q \subseteq \mathcal{A}_{id}^{q+1}$.
Thus, ${\mathcal{K}_{id}}$ admits a concept-poset fixpoint.
\end{theorem}

Based on Lemma \ref{lem:identifiedobject} and Theorem \ref{thm:identifiedobject} we show that convergence can be ensured in relational schemas with circuits if some contexts are contexts of identified objects.

\begin{corollary}[Convergence on circuits with identified objects]
\label{cor:circuit}
Let $(\mathbf{K},\mathbf{R})$ be an RCF and $\mathbf{S} \subseteq \mathbf{K}$
a set of contexts, closed under the successor relation of the dependency
graph (i.e. every successor of a context of $\mathbf{S}$ is in $\mathbf{S}$),
such that every context of $\mathbf{S}$ is either a sink or a context of
identified objects. Then every context of $\mathbf{S}$ poset-grows monotonically
and admits a concept-poset fixpoint.
\end{corollary}

\begin{proof}
We proceed by induction on the steps of RCA-AOC.
At step $0$, contexts contain no relational attribute, hence between
steps $0$ and $1$ every context of $\mathbf{S}$ only gains attributes.
Then  for any context of identified objects $\mathcal{K}_{id}$, the associated poset grows, 
hence $\mathcal{A}^0_{id} \subseteq \mathcal{A}^1_{id}$ (Lemma~\ref{lem:identifiedobject}).  Sinks are not extended: $\mathcal{A}^0_s \equiv \mathcal{A}^1_s$, and they will not be extended at any later step.
Assume $\mathcal{A}^{p-1}_j \subseteq \mathcal{A}^{p}_j$ for every
$\mathcal{K}_j \in \mathbf{S}$. 
The relational attributes of a context of $\mathbf{S}$ at step $p+1$ are built on the concepts 
of the step $p$ posets of its successors, which all belong to $\mathbf{S}$.
By the induction hypothesis applied to these successors, their
posets contain all the concepts of the corresponding step $(p-1)$ posets, thus
every relational attribute present at step $p$ is still present at
step $p+1$: each context of $\mathbf{S}$ only gains attributes between
steps $p$ and $p+1$. Then $\mathcal{A}^{p}_{j} \subseteq \mathcal{A}^{p+1}_{j}$ for every $\mathcal{K}_j \in \mathbf{S}$. 
Consequently, all successors of any context of
identified objects in $\mathbf{S}$  poset-grow monotonically, this context  admits thus a concept-poset fixpoint (Theorem \ref{thm:identifiedobject}).
\end{proof}

\subsection{Ensuring convergence in RCA-AOC}
\label{sec_convprocess}

\begin{figure}
\begin{center}
\includegraphics[scale=0.5]{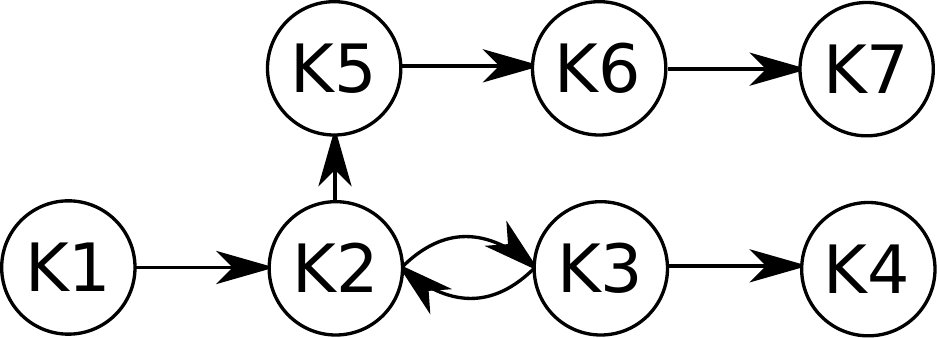}
\caption{Example of dependency graph.\label{fig:depg}}
\end{center}
\end{figure}

Considering Theorems \ref{thm:neighbors} and \ref{thm:identifiedobject}, we can elaborate a simple process that modifies the dataset to ensure convergence of the application of RCA-AOC for any given RCF:

\begin{enumerate}
\item Detect circuits and contexts without successors (sinks).
\item Add identifiers to contexts in circuits and to every context reachable from a circuit (i.e., all its transitive successors), except sink contexts. 
\end{enumerate}

The dependency graph depicted in Fig.~\ref{fig:depg} will help illustrate the process and why it leads to convergence in RCA-AOC.
{In this graph, $K_4$ and $K_7$ have no successor, thus they reach their fixpoint at the first step.}
Let us now consider the contexts that are in a circuit or reachable
from one.  We denote by $\mathbf{S}$ the set composed of the contexts in
circuits, together with all their transitive successors (including
sinks). In Fig.~\ref{fig:depg}, $\mathbf{S} = \{K_2, K_3,
K_4, K_5, K_6, K_7\}$. After applying the process above, every context
of $\mathbf{S}$ is either a sink or a context of identified objects and thus will admit a fixpoint {by reference to Corollary \ref{cor:circuit}.}
{The remaining context ($K_1$ in Fig. \ref{fig:depg}) has no predecessor and is not in a circuit nor reachable from one.} 
Its only successor  $K_2$ belongs to $\mathbf{S}$, hence admits a concept-poset fixpoint. Applying Theorem~\ref{thm:neighbors}, $K_1$ has thus a fixpoint.
Note that $K_1$ being a source context, with no predecessor, it has no
effect on other contexts.

To sum up, identifiers are added to contexts $K_2$, $K_3$, $K_5$ and $K_6$; the other contexts admit fixpoints as they are either sinks ($K_4$, $K_7$) or have all their successors admitting fixpoints ($K_1$).

Applying this process to the UML example, and thus adding identifiers in the Operation context (the only context where they are needed for ensuring convergence), would ensure a converging process.

To check if a context contains only identified objects, it is sufficient to verify that $\forall o \in G $, $\{o\}'' = \{o\}$. In the case of unidentified objects, one can add identifiers as described previously. In an automatic process, we propose to add identifiers without checking for two reasons: the first reason being that the cost of adding an identifying attribute is small in the AOC-poset computation; the second reason being that when applying RCA-AOC to different RCFs sharing the same dependency graph, the user can predict when the system will add identifiers and treat them consistently. The only drawback from adding identifiers is that each object is introduced alone in a concept, leading to more concepts than initially wanted. But the information added from identifiers can be easily spotted and removed from the analysis of the final AOC-posets if needed.

\subsection{A converging approach: RCA-AOC-conv}
\label{sec_rcaacoconv}

In this section, we present a converging variant of RCA-AOC, called RCA-AOC-conv. 
It relies on the property that ensures the convergence of the RCA process: the concept lattice of a context ${\cal K}^n_i$ at step $n$ is included, under the extent inclusion, in the concept lattice of ${\cal K}^{n+1}_i$ at step $n+1$. The number of concepts of a context being bounded by the powerset of its object set, if this number is monotonically increasing then convergence is ensured.
From this same property, an implementation of this process may use an optimization that consists in considering at each step the concept lattice from the previous step and computing only new concepts generated by new attributes, as in Galicia, one of the first implementations supporting RCA~\citep{Valtchev2003GaliciaA}. 
This is of course correct, but it cannot be adapted to RCA-AOC as we define it because the AOC-poset of a context at one step may not be included in the AOC-poset of the next step.

The process RCA-AOC-conv is inspired by this optimization.
At each step, the relational attributes built on the concepts of the previous concept-posets are accumulated into the extended contexts, and the concept-poset of a context is the AOC-poset of its cumulated extended context (see Algorithm~\ref{alg:rca-aoc-conv} in Appendix \ref{appendix_RCAAOCCONV}). As a consequence, new introducer concepts appear, while concepts that no longer introduce any element disappear. 
The convergence is guaranteed, but the invariant differs from the one of RCA: here, \emph{relational attributes}, rather than concepts, are never removed.
When a concept disappears, the relational attributes built on it at previous steps are kept, together with their incidence, which is fixed at creation (Definition~\ref{def_attridentity}).
We call such attributes \emph{dangling attributes}: they refer to a concept that no longer belongs to the current concept-posets.
For example, concept \texttt{co\_8} of step 0 of the UML example (Fig.~\ref{fig_bank_step0}) disappears at step 1, as it is no longer an object introducer; the relational attribute $\exists$\_ownedOperation(\texttt{co\_8}), built at step 1, is nevertheless kept at all the following steps, where it refers to \texttt{co\_8} as a dangling attribute.

Since relational attributes are only ever added and their incidence never changes, the successive extended contexts of a context $\mathcal{K}_i$ form an increasing sequence for the context inclusion defined in Sect.~\ref{sec_convconditions}: $\mathcal{K}_i^p \subseteq \mathcal{K}_i^{p+1}$ for every step $p$. 
Lemma~\ref{thm:growingcontext} applies and RCA-AOC-conv converges on any RCF. 

Because non-introducer concepts are removed, by construction, the structure built at step $p$ is exactly the AOC-poset of the extended context $\mathcal{K}_i^p$. What is sacrificed is thus not the AOC-poset structure itself, but two properties of lattice-based RCA. First, the step-to-step inclusion of the structures is lost: the concept-poset of one step may not be included in the concept-poset of the next step, since a concept that becomes a non-introducer disappears (see the example of \texttt{C\_K1\_3} in Appendix  \ref{appendix_RCAAOCCONV}). Second, the result is no longer self-contained: at the fixpoint, concept intents may contain dangling attributes $\rho(r)\,r(C)$ where $C$ belongs to none of the final concept-posets. Such attributes remain interpretable, since the extent of $C$ both identifies the attribute (Definition~\ref{def_attridentity}) and determines its incidence.

An alternative convergent variant would instead keep the concepts, as in RCA: non-introducer concepts, such as \texttt{co\_8} at step 1, would remain in the structure. The step-to-step inclusion of the structures would then be preserved and no dangling attribute would appear, but the conceptual structure would no longer be an AOC-poset in the general case, and the concept number would grow from the AOC-poset size towards, at worst, the concept lattice size. We adopt the attribute-keeping variant, which preserves the AOC-poset structure and its compactness while ensuring convergence: with respect to the AOC-posets that RCA-AOC builds on the same extended contexts, the only additional concepts are the attribute-concepts introducing dangling attributes, which often have an empty simplified extent (see e.g. concept \texttt{C\_K3\_13} in Appendix  \ref{appendix_RCAAOCCONV}).

In RCA-AOC, some disappearing concepts may have little  interest, as in the case of concept \texttt{co\_8} of step 0 (that groups together all operations), but some others may be very relevant, e.g. concepts \texttt{cpr\_22} (Fig. \ref{fig_bank_step2}) and \texttt{cc\_17} (Fig. \ref{fig_bank_step1}) that represent respectively the new suggested property \texttt{adminId:String} whose addition brings the suggestion of the new class \texttt{FinancialStructure}. These concepts, successively derived from one another and alternately appearing and disappearing in RCA-AOC, are summarized in Fig. \ref{fig_3concepts}. RCA-AOC-conv retains them: \texttt{cc\_17} is the attribute-concept of the kept attribute $\exists$\_ownedOperation(\texttt{co\_8}), and \texttt{cpr\_22} the attribute-concept of $\exists$\_class(\texttt{cc\_17}); both persist at every step following their creation and belong to the final result.

\begin{figure}[htb]
  \centering
    \includegraphics[width=0.31\linewidth]{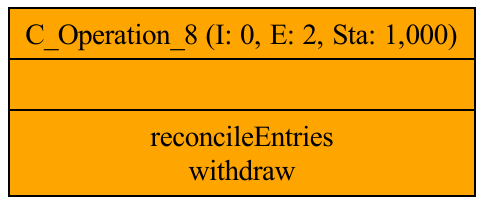}
      \includegraphics[width=0.33\linewidth]{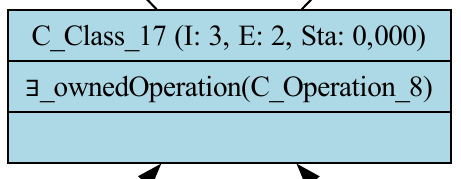}  
      \includegraphics[width=0.33\linewidth]{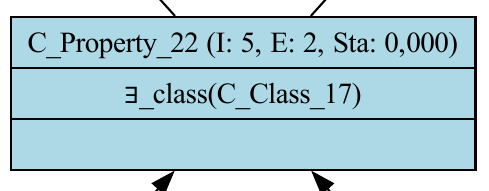}  
    \caption{Bank example - Concepts created and disappearing along the steps of RCA-AOC: \texttt{co\_8} present only at step 0; \texttt{cc\_17} present at step 1, created from \texttt{co\_8}, and absent at step 2;  \texttt{cpr\_22} present at step 2, created from \texttt{cc\_17}, and absent at step 3. In RCA-AOC-conv, \texttt{co\_8} also disappears at step 1, but \texttt{cc\_17} and \texttt{cpr\_22}, being the attribute-concepts of the kept relational attributes built on \texttt{co\_8} and \texttt{cc\_17} respectively, persist from their creation on.}
  \label{fig_3concepts}
\end{figure}
RCA-AOC-conv is also useful when relational attributes are developed by unfolding the concept references, to avoid having 
to keep the AOC-posets of all steps \citep{DBLP:conf/concepts/GutierrezHMZ25}.
RCA-AOC and RCA-AOC-conv are implemented in fca4j as options of the \texttt{RCA} command.  

Only RCA-AOC is implemented in RCAexplore. In this software, the 2015 version does not implement any divergence check, letting the user choose between stopping manually or stopping automatically when the number of concepts stays stable between two steps. Version 26.1 implements a repetition check to stop the process when a step is identically repeated twice\footnote{\url{https://forge.icube.unistra.fr/dolques/RCAExplore}}.

%% file: tex__relatedwork.tex
\section{Related work}
\label{sec:relatedwork}

 \paragraph{AOC-posets in FCA}

To the best of our knowledge, the AOC-posets have been introduced by \cite{Godi93a} in the domain of software engineering (object-oriented programming). In their paper, the AOC-poset is called \textit{pruned lattice} and they consider a specific case where each formal object owns a specific formal attribute (not owned by the others). 
Algorithms and tools for building AOC-posets were proposed in \citep{Godi98a,hermes2012}.

 The AOC-poset has also been used in applications of FCA to non-monotonic reasoning and domain theory~\citep{Hitz04} and to produce classifications from linguistic data~\citep{Ossw02,petersen2004set}.
Several software engineering works have relied on specific parts of the AOC-poset, in particular the attribute-concept component. It has been used, for example, to refactor class hierarchies in code reengineering~\citep{Huch00a}, and to extract feature trees from sets of products in software product lines~\citep{DBLP:conf/splc/RysselPK11}.

 \paragraph{RCA foundations and connections with other frameworks}
 
 Relational Concept Analysis originated from a practical knowledge-representation problem encountered in software engineering, databases, and ontology-based settings: discovering concepts latent in a conceptual model.  More precisely, the objective is to make explicit concepts that provide a better factorization of descriptions, reducing duplication and thus improving the overall level of abstraction of the conceptual model. 
 The initial motivation was to extend early FCA-based approaches~\citep{Godi93a} to relational models, in particular UML models~\citep{DBLP:conf/iccs/DaoHHRV04}.
Formalizations of RCA have been proposed in~\citep{HuchardHRV07,rouane2013}; here, we build on the latter.
 This framework computes in an iterative manner (with a possible stop at each step) several concept lattices from data represented in relational format. The concept lattices are connected by links that abstract the relations between objects. Several operators borrowed from Description Logics are used to build links between concepts.
Relations between these various operators and the corresponding concept lattices are described in \citep{DBLP:journals/kbs/BraudDHB18}. 
RCA has also been studied in relation to Description Logics~\citep{DBLP:conf/icfca/RouaneHNV07} and propositionalization~\citep{DBLP:conf/icfca/DolquesMBHB14}.
\cite{euzenat-jair25} adopts a functional view on the RCA process, and defines  the acceptable  solutions (families of concept lattices) as the common fixpoints of two functions. He shows that the RCA process returns the least element of the set of acceptable solutions.
Beyond RCA, a review of FCA methods applied to  relational data has been done by  \cite{DBLP:journals/kais/LeutwylerLFPTT24}.
 
 \paragraph{RCA variants}

Beyond the core formalism, RCA has been improved with specific navigation  control features on the dataset structure, on the scaling operators and on the built conceptual structures  \citep{DBLP:conf/f-egc/DolquesBHN13a,DBLP:conf/f-egc/OuzerdineBDHB19a}. These features have been operationalized on the basis of the dedicated tool \textit{RCAexplore} \citep{DBLP:conf/icfca/DolquesBHB19}.
More recently, LLM-based approaches have been proposed to support the interpretation of RCA results, using knowledge-delivery mechanisms based on rewriting strategies. These mechanisms translate relational attributes into logic-like formulas grounded either in concept extents or in the non-relational attributes of concept intents~\citep{DBLP:conf/concepts/GutierrezHMZ25}. 
In parallel, Fuzzy RCA enriches the paradigm with a graded semantics \citep{DBLP:conf/eusflat/BoffaM23}.
RCA being based on binary relations, higher-arity settings are considered in Polyadic RCA \citep{DBLP:journals/ijar/BazinGK24}.
Besides, to cope with the scalability issues of RCA, \cite{DBLP:conf/cla/DolquesBH13} proposed a variant based on AOC-posets, namely RCA-AOC. 
\cite{ijgis2016}  performed more specifically a comparison between RCA-AOC and RCA based on  Iceberg lattices \citep{Stumme2002}. We showed that RCA-AOC was more efficient and pertinent since it allows extracting interesting and less frequent  behaviors than Iceberg lattices that limit the computed concepts to the most general ones. Such variants should also be interesting for fuzzy or Polyadic RCA, whose  results are more complex than RCA.

\paragraph{RCA Applications}
 
RCA  has been used for the analysis and modernization of UML elements, namely in class diagrams and in use case diagrams~\citep{DBLP:conf/models/ArevaloFHN06,DBLP:journals/fuin/DolquesHNR12,DBLP:conf/cla/GuediMHN13}. 
In~\citep{DBLP:conf/icfca/MohaHVG08}, RCA is used to exploit relations between methods and attributes to detect and fix design defects.
Model transformations are learned from transformation examples thanks to several kinds of relations between model elements (e.g. between elements inside a model, transformation links between source elements and target elements)~\citep{DBLP:conf/models/SaadaDHNS12}. 
In the context of Web service composition, \cite{DBLP:conf/icws/AzmehDHHMT11} used relations between tasks in an abstract task pipeline to classify Web services according to their relevance for instantiating the pipeline tasks.
Other applications can be found in ontology engineering~\citep{bendaoud08b,rvn-iccs11}.  More recently,  RCA was used to extract knowledge graphs from relational data about neurological examinations \citep{WAJNBERG20181397}, 
to analyze data on pediatric cancers \citep{DBLP:conf/icdm/WajnbergVMBKLLS20} or 
faults in aluminum die casting process \citep{DBLP:conf/jowo/WajnbergVLMP19}, to query legal documents \citep{DBLP:conf/icail/MimouniFNBS13}, 
while \cite{nica-dam2020} used it to extract closed partially-ordered patterns (acyclic graphs) from temporal sequences on water quality.
 RCA has also been used for extracting interdependent linked keys from RDF datasets, including  cyclic dependencies \citep{atencia:hal-02984963,DBLP:journals/dam/AtenciaDENV20}.  \cite{semeraro2023} used FCA and RCA  to extract rules from data on physical systems, for the design of digital twins in the engineering domain.
 The same approach is applied in the specific case of digital twins for compressed air energy storage systems \citep{semeraro2025data}.
In most of these applications, the existential scaling operator is used, and the datasets are medium-sized guaranteeing the feasibility of the approach. When dealing with larger or more complex datasets,  RCA-AOC or RCA using Iceberg lattices can be applied, e.g.  to analyze time series from river monitoring \citep{braud2022} or complex data about ancient remedies \citep{fokou-elhaff2024}. 

These works show that RCA-AOC is relevant not only as a more compact alternative to lattice-based RCA, but also as a way to focus on concepts that are meaningful for specific applications, including concepts that may be discarded by support-based restrictions such as Iceberg lattices.

%% file: tex__conclusion.tex
\section{Conclusion}
\label{sec:conclusion}

This paper addressed the convergence of Relational Concept Analysis when concept lattices are replaced by AOC-posets, yielding the RCA-AOC variant. 
AOC-posets provide a compact and efficient representation that helps mitigate the computational complexity associated with large datasets. In some applications, only introducer concepts are useful, while in others they are sufficient to capture the main outcomes of the analysis. However, replacing concept lattices with AOC-posets generally means losing the convergence guarantee of lattice-based RCA.
We illustrated possible divergence through three examples involving existential and strict universal scaling, including one grounded in a UML class-model refactoring task. 
We then identified properties that prevent divergence. We discussed how convergence can be recovered through conditions on the data and on the process. 
Finally, we proposed a convergent variant (RCA-AOC-conv) whose structures are the AOC-posets of cumulated extended contexts: convergence is obtained by never removing relational attributes, some of which may end up referring to concepts absent from the final structures.

As future work, we plan to develop practical tools to detect potential non-convergence in real application domains and to automatically repair problematic datasets using the procedure proposed in this paper. 
A key objective will be to predict divergence without having to compute entire families of concept-posets until a repetition of step sequences is observed. 
We will also examine whether it is preferable to enforce convergence by modifying the dataset (as discussed in the previous section) or by adopting a convergent process variant (e.g. RCA-AOC-conv). In particular, we will study the trade-offs between these two approaches in terms of computational complexity and the size (and usefulness) of the resulting conceptual structures. 
We plan to study the alternative convergent variant discussed in Sect.~\ref{sec_rcaacoconv}, which would keep the concepts that no longer introduce any element: in particular, adding attributes to these concepts so that they remain introducers, as suggested by \cite{aranda_corral_2026_18608706,arandaConcepts2024}, would preserve both the AOC-poset structure and the step-to-step inclusion of the structures, avoiding dangling attributes.
Moreover, since we can now ensure convergence for a given dataset, we expect to implement RCA-AOC more efficiently, for instance by relying on an incremental AOC-poset construction algorithm. 
We also intend to further study how divergence relates to the choice and combination of scaling operators, in particular by characterizing which ones are more prone to trigger non-convergence and under which data conditions.
Finally, we could extend to RCA-AOC the functional view of \cite{euzenat-jair25}, which defines the acceptable solutions (families of concept lattices) as the common fixpoints of two functions, where RCA returns the least element of the set of acceptable solutions.

%% file: tex__appendix.tex
\appendix
\renewcommand{\thesection}{\Alph{section}}

\section{Proof of Lemma \ref{lem:identifiedobject}}
\label{appendix_the3}

\begin{proof}

\begin{figure}
\begin{center}
\includegraphics[width=0.95\textwidth]{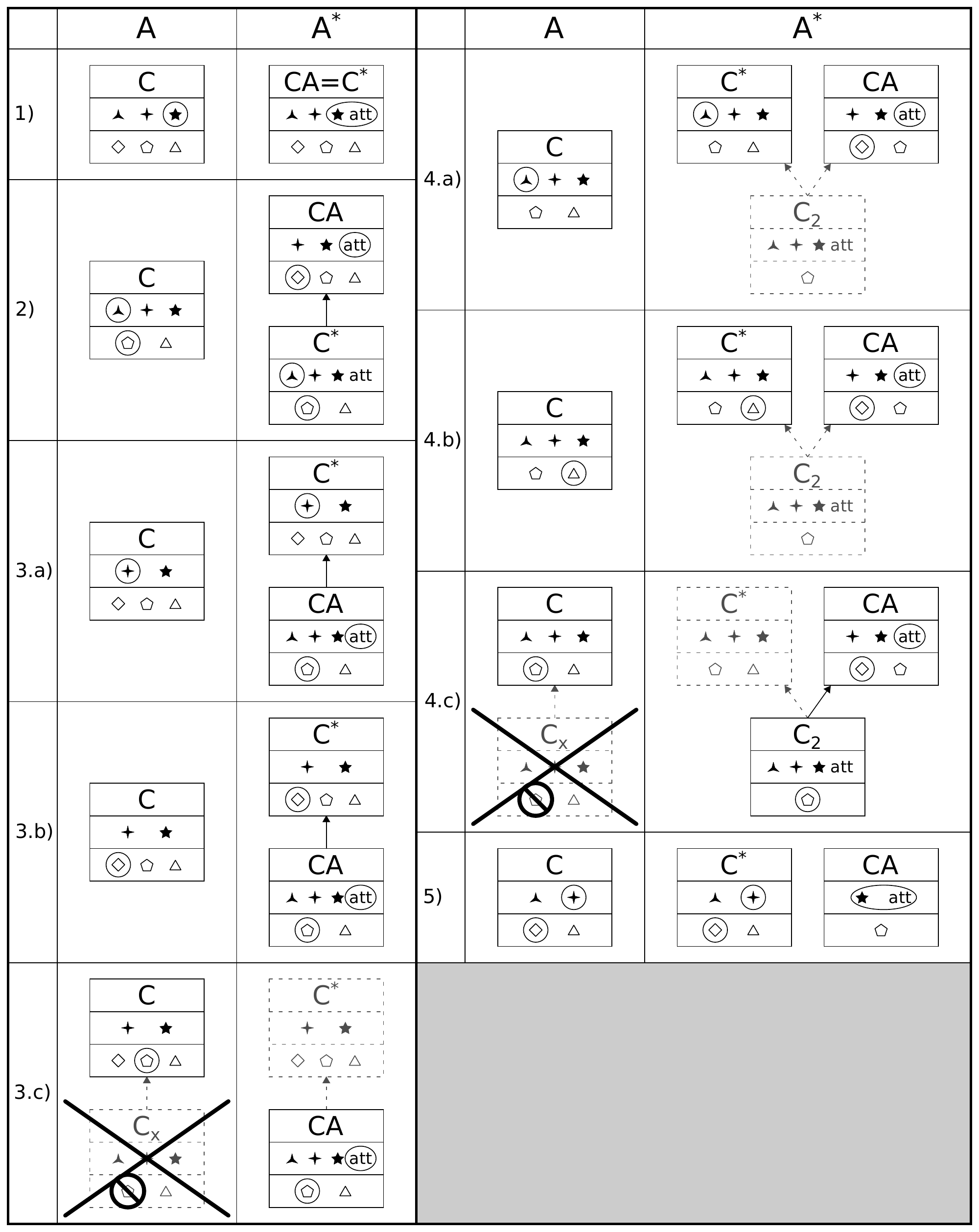}
\caption{Illustration of the incremental construction of an AOC-poset by adding an attribute.
Each concept is represented with its full intent and extent. 
The simplified intent and simplified extent elements are circled. 
Concepts represented with dashed lines are concepts from the lattice that do not belong to the AOC-poset in the pictured cases: the concepts $C^*$ from cases 3.c and 4.c can never exist in the AOC-poset while the concepts $C_2$ from cases 4.a and 4.b could belong to the AOC-poset if $\mathit{Extent}_S(C)\cap \mathit{Extent}(CA) \neq \emptyset$.
Crossed-out concepts are concepts that cannot exist even in the concept lattice.}\label{fig:prop1}
\end{center}
\end{figure}

We begin with a few notations. 
 Let $\mathcal{K}=(G,M,I)$ be a context and let $\mathcal{A}$ be its AOC-poset. Let $\mathit{att}$ denote  a new attribute and $\mathit{g}\subseteq G$ the set of objects owning this attribute. $\mathcal{K}^* = (G,M\cup \{\mathit{att}\},I \cup g\times \{\mathit{att}\})$ is the context resulting from the addition of the  attribute $\mathit{att}$ to the objects of $\mathit{g}$  in  $\mathcal{K}$ and $\mathcal{A}^*$ is its AOC-poset. 
 $\mathit{Intent}_S(C)$ denotes the simplified intent of $C$, $\mathit{Extent}_S(C)$ its simplified extent.
 
 Then we introduce two lemmas.
 The first one states that the
property of being a context of identified objects is preserved by a single attribute addition.
 The second one states that $\mathcal{A}^*$
necessarily contains a concept whose extent is $g$, namely the
attribute-concept of $\mathit{att}$, denoted $\mathit{CA}$ in the
following.
 
\begin{lemma}[Preservation]
\label{lem_preservation}
If $\mathcal{K}$ is a context of identified objects, then  $\mathcal{K}^*$ is also a context of identified objects.
\end{lemma}

\begin{proof}
Let $o \in G$, and let  $(\cdot)'^{*}$ denote
derivation in \ $\mathcal{K}^*$. Since
$\{o\}'^{*} \cap M = \{o\}'$, every object owning all attributes of
$\{o\}'^{*}$ owns in particular all attributes of $\{o\}'$, hence
$\{o\} \subseteq \{o\}''^{*} \subseteq \{o\}'' = \{o\}$, and therefore
$\{o\}''^{*} = \{o\}$.
\end{proof}

 \begin{lemma}[Existence of an attribute-concept CA]
\label{lem_existence}
The set $g$ is closed in $\mathcal{K}^*$, and
$\mathit{CA} = (g,\, g')$ is a formal concept of $\mathcal{K}^*$,
namely the attribute-concept of $\mathit{att}$. In particular,
$\mathit{att} \in \mathit{Intent}_S(\mathit{CA})$ and $\mathit{CA}$
belongs to $\mathcal{A}^*$.
\end{lemma}

\begin{proof}
The set $g$ is closed in $\mathcal{K}^*$ since $\{att\}'=g$ by construction (any image by a derivation operator is closed \citep{Gant99a}). 
$\mathit{CA} = (g, g')$ is thus a well-defined concept of $\mathcal{K}^*$ whose simplified intent contains $att$.
\end{proof}
 
 For each $C=(X,Y) \in \mathcal{A}$, we now exhaustively study the different cases that may occur. 
 Those cases are illustrated by Fig.~\ref{fig:prop1} to help the reader.
During our reasoning we consider that the identifier of a concept is its extent and that two concepts from different AOC-posets are equivalent if their extents are the same regardless of their intent. We denote the equivalence relation between those concepts by the binary operator $\sim$. 
Relational attributes are identified as stated in Definition~\ref{def_attridentity}: if a concept equivalent to $C$
exists at the following step, the attribute $\rho(r)\,r(C)$ is considered to be the same, and
its incidence is unchanged, since it only depends on $(\rho(r), r, \mathit{Extent}(C))$.

\begin{enumerate}
%
%
\item if $g=X$ then $C\sim CA$. 
  In this situation, the considered concept is equivalent to the concept introducing \textit{att}, so \textit{CA} is built just by adding \textit{att} to the intent of \textit{C}.
%
%
\item if $X \subset g$, $\exists C^*=(X,Y\cup \{att\}) \in \mathcal{A}^*$ and $C\sim C^*$.
In other words, for every concept whose extent is strictly included in \textit{g},  \textit{att} is added to its intent. It becomes more specific than \textit{CA}, but it keeps its simplified intent and simplified extent.
%
%
\item if $g \subset X$, 
\begin{enumerate}
\item If $\mathit{Intent}_S(C) \neq \emptyset$ then $\exists C^*=(X,Y) \in \mathcal{A}^*$ as $\mathit{Intent}_S(C^*)=\mathit{Intent}_S(C)$. 
This corresponds to the concepts whose extent contains \textit{g} and which have a non-empty simplified intent. The concept $C$ remains as $C^*$ with the same simplified intent, the simplified extent may be reduced, and become $\mathit{Extent}_S(C^*)=\mathit{Extent}_S(C) \setminus g$, if $\mathit{Extent}_S(C)\cap g \neq \emptyset$.
\item If $\mathit{Intent}_S(C) = \emptyset$ and 
$\mathit{Extent}_S(C)\nsubseteq g $
then $\exists C^*=(X,Y) \in \mathcal{A}^*$ as $\mathit{Extent}_S(C^*)=\mathit{Extent}_S(C) \setminus g \neq \emptyset$. 
This corresponds to the concepts whose extent contains \textit{g}, whose simplified intent is empty and whose simplified extent is not contained by \textit{g}. The concept $C$ remains as $C^*$ with the same simplified intent, the simplified extent of $C^*$ may contain fewer objects if $\mathit{Extent}_S(C)\cap g \neq \emptyset$.
\item \label{item:rem1} If $\mathit{Intent}_S(C)=\emptyset$ and $\mathit{Extent}_S(C)\subseteq g$ then there does not exist any concept $C^*$ of $\mathcal{A}^*$ such that $C\sim C^*$. In this case, $\mathit{CA}$ is not equivalent to any existing concept of $\mathcal{A}$, otherwise the condition $\mathit{Extent}_S(C)\subseteq g$ could not be verified. Thus, $\mathit{CA}$ is a new concept in $\mathcal{A}^*$, and a subconcept  of the concept lattice $(X,Y)$ that has no equivalent in $\mathcal{A}^*$.
In other words, the concept $C$ whose extent contains \textit{g}, whose simplified extent is contained by \textit{g} and whose simplified intent is empty has no equivalent in $\mathcal{A}^*$ as the objects it introduces in $\mathcal{A}$ are introduced by \textit{CA} in $\mathcal{A}^*$. 
In Fig.~\ref{fig:prop1} the barred $C_x$ is here to emphasize that no concept in $\mathcal{A}$ more specific than \textit{C} can contain the objects introduced by \textit{C}.
\end{enumerate}
%
%
\item  if $g \cap X \neq \emptyset$ and  $g \nsubseteq X$ and  $X \nsubseteq g$. 
This corresponds to the concepts whose extent neither contains nor is contained by \textit{g} but shares some objects with \textit{g}.
\begin{enumerate}
\item If $\mathit{Intent}_S(C)\neq\emptyset$ then there exists $C^*=(X,Y)$ in $\mathcal{A}^*$.
If the simplified intent is not empty then an equivalent concept remains in $\mathcal{A}^*$. Note that if some objects of the simplified extent are in \textit{g}, then a concept $C_2$ as pictured in dashed lines in Fig.~\ref{fig:prop1} appears in $\mathcal{A}^*$ introducing the objects from $\mathit{Extent}_S(C)\cap g$.
\item  If $\mathit{Intent}_S(C)=\emptyset$ and $\mathit{Extent}_S(C) \nsubseteq g$ then there exists $C^*=(X,Y)$ in $\mathcal{A}^*$.
If the simplified intent is empty but the simplified extent is not included in \textit{g} then the concept will remain in $\mathcal{A}^*$. Note that if some of the simplified extent is included in \textit{g}, then a concept $C_2$ as pictured in dashed lines in Fig.~\ref{fig:prop1} appears in $\mathcal{A}^*$ introducing the objects from $\mathit{Extent}_S(C)\cap g$.

\item \label{item:rem2} If $\mathit{Intent}_S(C)=\emptyset$ and $\mathit{Extent}_S(C) \subset g$ then there does not exist any concept $C^*$ of $\mathcal{A}^*$ such that $C\sim C^*$.\\
Let $C_2$ be a concept from $\mathcal{K}^*$ such that $\mathit{Extent}(C_2)=g\cap X$ ($g \cap X$ is closed in $\mathcal{K}^*$, as
$(Y \cup \{\mathit{att}\})' = Y' \cap g = X \cap g$).
$C_2$ is a concept of $\mathcal{A}^*$ as $g \cap \mathit{Extent}_S(C) \neq \emptyset$ and $g \cap \mathit{Extent}_S(C) \subseteq \mathit{Extent}_S(C_2)$. There is no $C_3  \in \mathcal{A}$ such that $C_3 \sim C_2$.
In other words, if the simplified intent is empty and the whole simplified extent is included in $g$, then the objects of the simplified extent are introduced by a concept $C_2$ more specific than $C$ leading $C$ to disappear in $\mathcal{A}^*$. In Fig.~\ref{fig:prop1} the barred $C_x$ is here to emphasize that no concept in $\mathcal{A}$ more specific than \textit{C} can contain the object introduced by \textit{C}.
\end{enumerate}
%
%
\item if $g \cap X= \emptyset$, $\exists C^* \in \mathcal{A}^*$ such that $C \sim C^*$.
If the concept shares nothing with $g$ then it is unaffected by the addition of \textit{att}.
\end{enumerate}

From the previous cases, only cases \ref{item:rem1} and \ref{item:rem2} can lead to removing a concept. 
In all the other cases $C$ has an equivalent
concept in $\mathcal{A}^*$, and the possible new concepts $\mathit{CA}$ can only enlarge $\mathit{Ext}_{\mathcal{A}^*}$.
However in a context of identified objects, cases \ref{item:rem1} and \ref{item:rem2} never occur: in both cases $\mathit{Intent}_S(C)=\emptyset$.
Since $C$ belongs to the AOC-poset, it introduces at least one element, here an object $o$, so that $o \in \mathit{Extent}_S(C) \neq \emptyset$. In a context of identified objects, $\{o\}'' = \{o\}$. The object-concept of $o$ is $(\{o\}, \{o\}')$, and since $C$ introduces $o$, $C$ is this object-concept. 
Therefore
$\mathit{Extent}(C)= \mathit{Extent}_S(C) = \{o\}$ 
and $|\mathit{Extent}(C)|=1$.
In \ref{item:rem1}, $g\subset \mathit{Extent}(C)$ which means that $|g|=0$, and $\mathit{Extent}_S(C)\subseteq g$ which is impossible considering that $|\mathit{Extent}_S(C)|=1$ and $|g|=0$. In \ref{item:rem2}, $g \cap \mathit{Extent}(C) \neq \emptyset$ and $\mathit{Extent}(C)\nsubseteq g$ which is impossible considering that $|\mathit{Extent}(C)|=1$.
Consequently, no concept of $\mathcal{A}$ is removed, and the AOC-poset of a context of identified objects can only grow when an attribute is added.

\end{proof}

\section{Algorithm RCA-AOC-conv}
\label{appendix_RCAAOCCONV}

Algorithm~\ref{alg:rca-aoc-conv} implements the approach described in Sect.  \ref{sec_rcaacoconv}. 
Part of the notation used was introduced in Sect. \ref{sec:rca_background} and \ref{sec:rca_gsh}. We recall them for the sake of clarity and define the new ones. We denote by:
\begin{itemize}
    \item $\mathcal{P}_i^{p}$ the concept-poset built at step $p$ for the context derived from $\mathcal{K}_i$
    \item $\mathrm{AOC\text{-}poset(\mathcal{K})}$ the function used to compute the AOC-poset of a context. It is used at every step of the process
     \item $t_r$ the index of the target context of relation $r$
    \item $R_{G_i}$ the subset of relations $r \in \mathbf{R}$ such that $G_{s_{r}}=G_i$, where $s_r$ is the index of the source context of relation $r$
    \item $\mu(m) = (\{m\}', \{m\}'')$ the attribute-concept of $m$, where
      the derivation operators are those of the current extended
      context $\mathcal{K}^p_i$ (used in the implementation note below)
\item $\gamma(o) = (\{o\}'', \{o\}')$ the object-concept of $o$, with
      the same convention
\end{itemize}

Algorithm~\ref{alg:rca-aoc-conv} makes the convergence argument directly
visible: line~\ref{line:appose} only ever adds relational attributes,
whose incidence is fixed at creation and which are identified by their
triple (Definition~\ref{def_attridentity}), so that
$\mathcal{K}_i^{p-1} \subseteq \mathcal{K}_i^{p}$ at every step and the
number of attributes is bounded for each context. The increasing
sequence of extended contexts is therefore stationary
(Lemma~\ref{thm:growingcontext}) and the stop condition of the loop
is eventually satisfied. 
Note that scaling is
applied to all the concepts of the previous posets: the attributes
built on concepts already present at earlier steps exist in
$\mathcal{K}_i^{p-1}$ and are not duplicated, while 
 the attributes built at previous steps on concepts that have since been removed are kept by the apposition, 
as dangling attributes (Sect.~\ref{sec_rcaacoconv}).

\begin{algorithm}[t]
\caption{RCA-AOC-conv}
\label{alg:rca-aoc-conv}
\begin{algorithmic}[1]
\Require an RCF $(\mathbf{K}, \mathbf{R})$, with $\mathbf{K}$ a set of formal contexts  $\{\mathcal{K}_i\}_{i=1,\dots,n}$,
        and a function $\rho$ which associates a scaling operator to each relation
\Ensure  a family of concept-posets $(\mathcal{P}_i)_{i=1,\dots,n}$
\State $p \gets 0$
\ForAll{$\mathcal{K}_i \in \mathbf{K}$}
    \State $\mathcal{K}_i^0 \gets \mathcal{K}_i$;\quad
           $\mathcal{P}_i^0 \gets \mathrm{AOC\text{-}poset}(\mathcal{K}_i^0)$
           \Comment{initial AOC-poset for $\mathcal{K}_i$}
\EndFor
\Repeat
    \State $p \gets p + 1$
    \ForAll{$\mathcal{K}_i \in \mathbf{K}$}
        \State $\mathcal{C}_{t_r}^{\,p-1} \gets \mathit{concept~set~of~} \mathcal{P}_{t_r}^{\,p-1}$
               for each $r \in R_{G_i}$
               \Comment{with $R_{G_i} =\{r^i_1, \ldots, r^i_{k_i}\}$}
        \State \label{line:appose}
               $\mathcal{K}_i^p \gets \mathcal{K}_i^{p-1}
               \,|\, \mathbb{S}_{\rho(r^i_1)}(\mathcal{K}_i, r^i_1, \mathcal{C}_{t_{r^i_1}}^{\,p-1})
               \,|\, \dots \,|\,
               \mathbb{S}_{\rho(r^i_{k_i})}(\mathcal{K}_i, r^i_{k_i}, \mathcal{C}_{t_{r^i_{k_i}}}^{\,p-1})$
               \Comment{apposition up to attribute identity (Def.~\ref{def_attridentity}): attributes already in $\mathcal{K}_i^{p-1}$ are not duplicated}
        \State \label{line:union}
               $\mathcal{P}_i^p \gets \mathrm{AOC\text{-}poset}(\mathcal{K}_i^p)$
    \EndFor
\Until{$\forall i,\ \mathcal{K}_i^p = \mathcal{K}_i^{p-1}$}
\State \Return $(\mathcal{P}_i^p)_{i=1,\dots,n}$
\end{algorithmic}
\end{algorithm}

An implementation need not recompute the AOC-poset from scratch at
each step. The one provided in \textsf{fca4j} proceeds incrementally:
(i) only the relational attributes referring to the concepts created at
the previous step are generated, the others being already present;
(ii) the new introducer concepts are obtained as the attribute-concepts
$\mu(m)$ of the new attributes $m$, and as the
object-concepts $\gamma(o)$, recomputed for the
objects owning a new attribute, since the introducer concept of an
object may become more specific when new attributes are added;
(iii) conversely, a concept whose introduced objects have all migrated
to more specific concepts, and which introduces no attribute, is
removed from the poset. The incremental result coincides with
$\mathrm{AOC\text{-}poset}(\mathcal{K}_i^p)$: the extent $\{m\}'$ of
the attribute-concept of an existing attribute $m$ is fixed and remains
closed, and the object-concept of an object owning no new attribute is
unchanged. This is the counterpart, for RCA-AOC-conv, of the
optimization used in Galicia for RCA (Sect.~\ref{sec_rcaacoconv}).

\begin{table}[htb]
\caption{RCF making RCA-AOC diverge when using the existential scaling operator on each relation, and requiring object-concept  
recomputation in RCA-AOC-conv.}
\label{tab_rcf-div-oc}
\centering
\footnotesize
\setlength{\tabcolsep}{2.5pt}
\renewcommand{\arraystretch}{1.0}
\resizebox{\textwidth}{!}{%
\begin{tabular}{|c@{\;\;}c@{\;\;}c@{\;\;}c|}
\hline
\multicolumn{4}{|c|}{Formal Contexts}\\
\hline
\begin{tabular}{c|c}
$K_1$ & a1\_1 \\ \hline
o1\_1 & $\times$ \\
o1\_2 & $\times$ \\
o1\_3 & \\
\end{tabular}
&
\begin{tabular}{c|cc}
$K_2$ & a2\_1 & a2\_2 \\ \hline
o2\_1 & $\times$ & \\
o2\_2 & & $\times$ \\
\end{tabular}
&
\begin{tabular}{c|cc}
$K_3$ & a3\_1 & a3\_2 \\ \hline
o3\_1 & $\times$ & \\
o3\_2 & & $\times$ \\
\end{tabular}
&
\begin{tabular}{c|cc}
$K_4$ & a4\_1 & a4\_2 \\ \hline
o4\_1 & $\times$ & \\
o4\_2 & & $\times$ \\
\end{tabular}
\\
\hline
\multicolumn{4}{|c|}{Relations}\\
\hline
\begin{tabular}{c|cc}
$r1\_2$ & o2\_1 & o2\_2 \\ \hline
o1\_1 & $\times$ & \\
o1\_2 & & $\times$ \\
o1\_3 & $\times$ & \\
\end{tabular}
&
\begin{tabular}{c|ccc}
$r3\_1$ & o1\_1 & o1\_2 & o1\_3 \\ \hline
o3\_1 & & $\times$ & \\
o3\_2 & & & $\times$ \\
\end{tabular}
&
\begin{tabular}{c|cc}
$r4\_3$ & o3\_1 & o3\_2 \\ \hline
o4\_1 & $\times$ & \\
o4\_2 & & $\times$ \\
\end{tabular}
&
\begin{tabular}{c|cc}
$r3\_4$ & o4\_1 & o4\_2 \\ \hline
o3\_1 & $\times$ & \\
o3\_2 & & $\times$ \\
\end{tabular}
\\
\hline
\end{tabular}
}
\end{table}

We now develop an example which combines divergence 
and the need for recomputing object-concepts when new attributes arrive. 
Table \ref{tab_rcf-div-oc} shows the RCF. 
The dependency graph is shown in Fig. \ref{fig_graph-div-oc}.
Figures~\ref{fig_conv_1} and~\ref{fig_conv_2} 
show the concept-posets computed by \textsf{fca4j} with
RCA-AOC-conv at steps~0 to~4.

\begin{figure}[htb]
\centering
\includegraphics[width=0.7\linewidth]{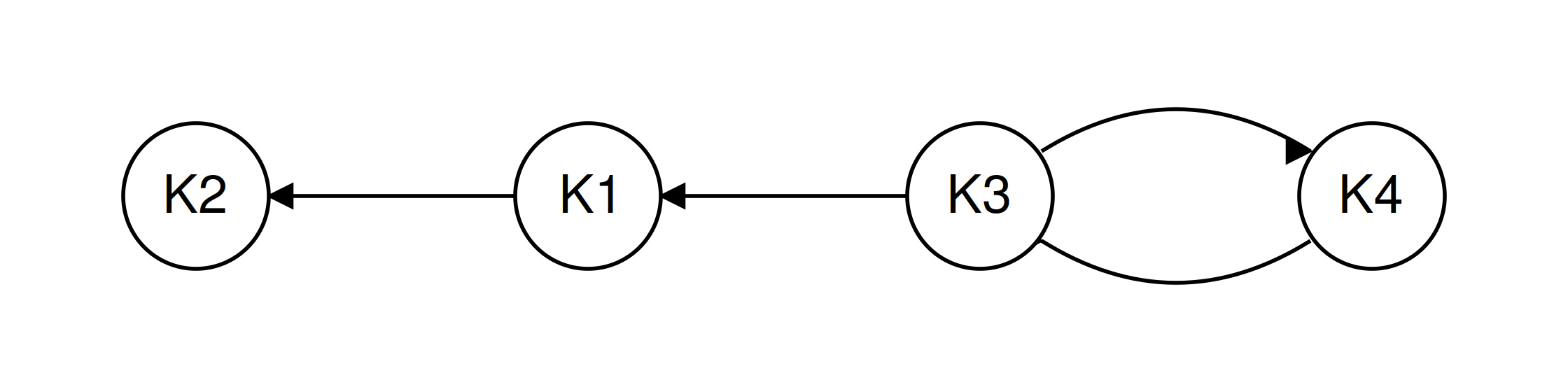}
\caption{Dependency graph of Table \ref{tab_rcf-div-oc} }
\label{fig_graph-div-oc}
\end{figure}

\begin{figure}[htb]
\centering
\includegraphics[width=\linewidth]{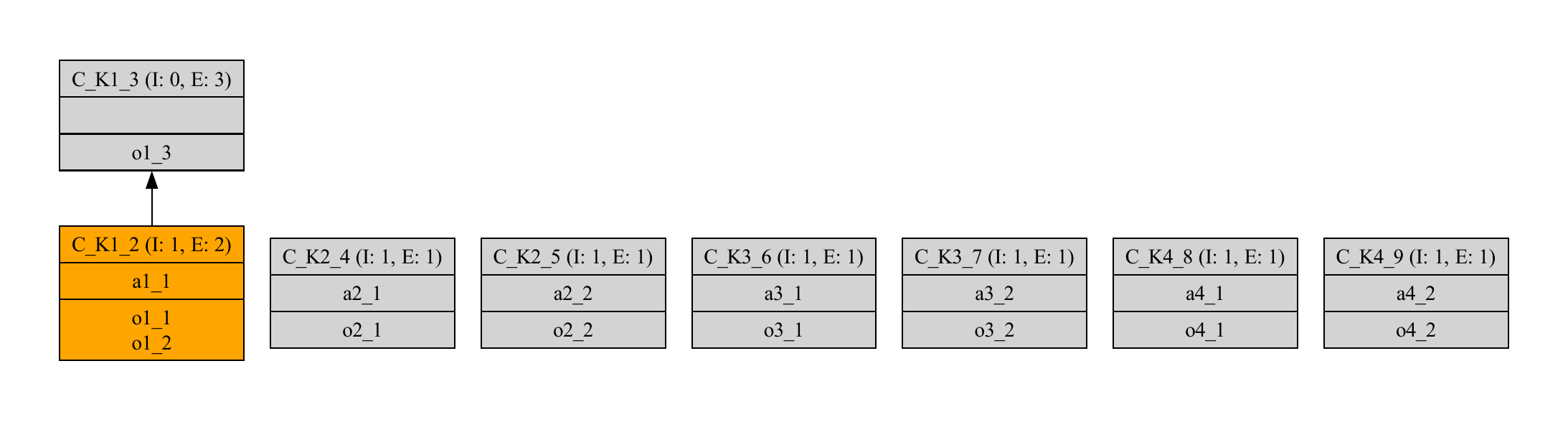}
\includegraphics[width=\linewidth]{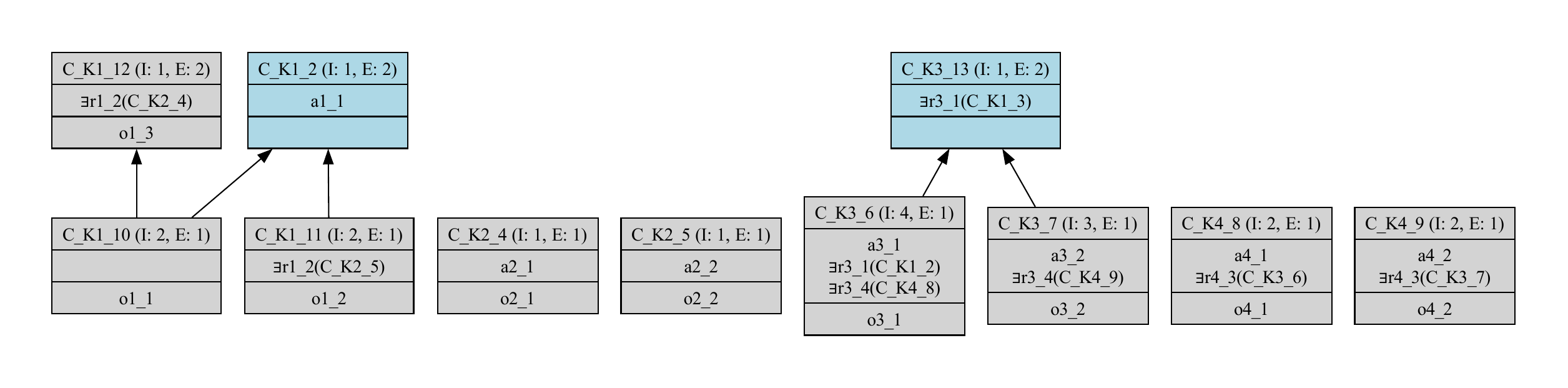}
\includegraphics[width=\linewidth]{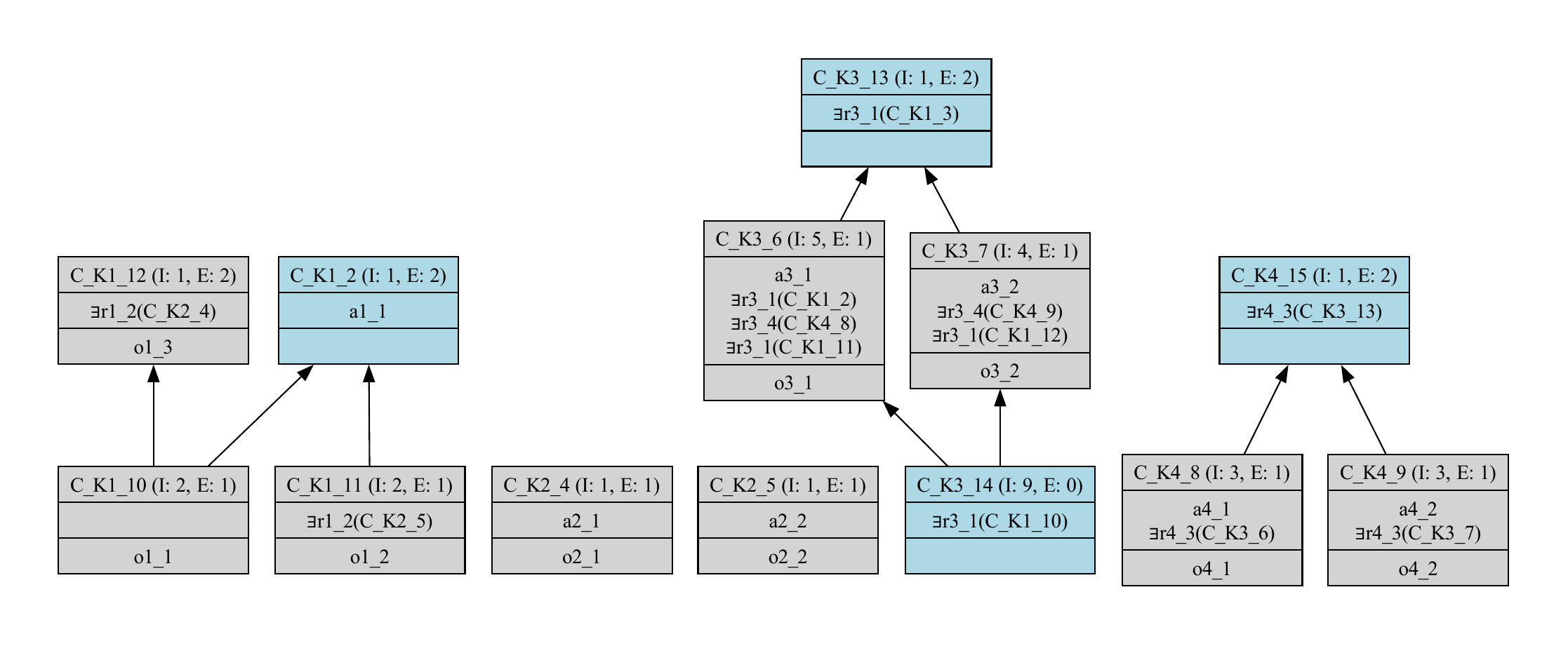}
\caption{AOC-posets built with RCA-AOC-conv at steps 0 to 2 for Table \ref{tab_rcf-div-oc}. The process converges.}
\label{fig_conv_1}
\end{figure}

\begin{figure}[htb]
\centering
\includegraphics[width=\linewidth]{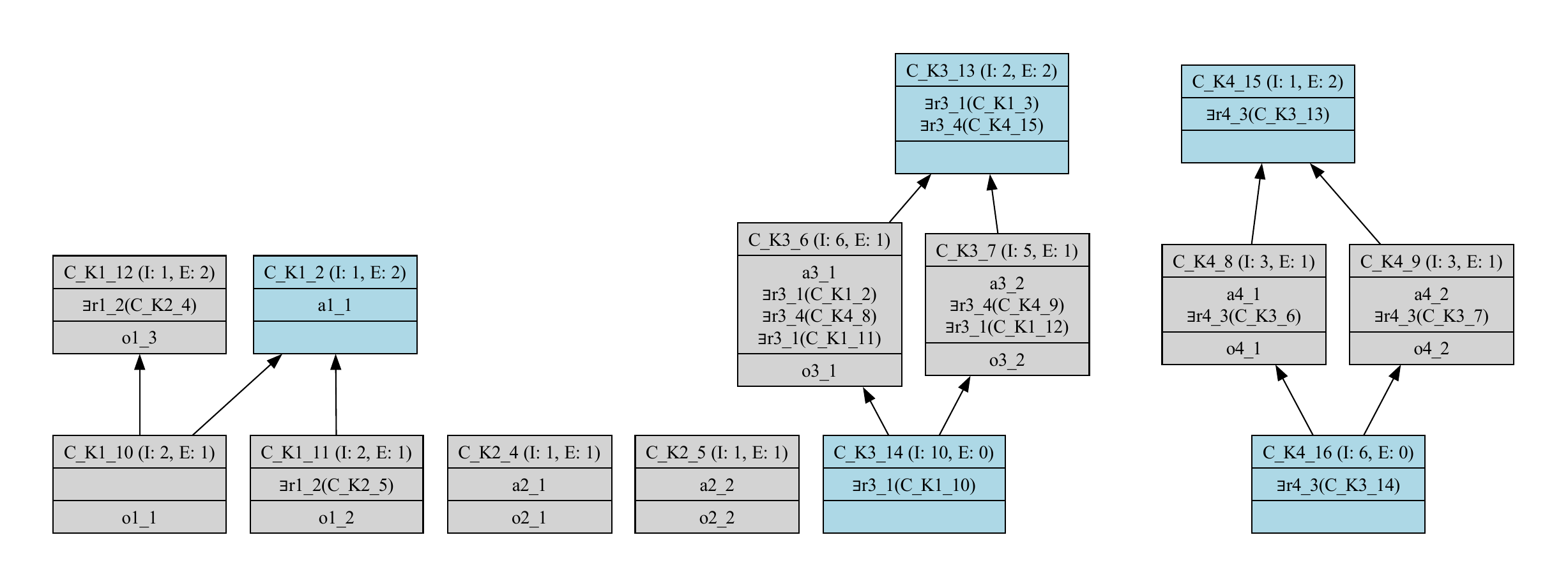}
\includegraphics[width=\linewidth]{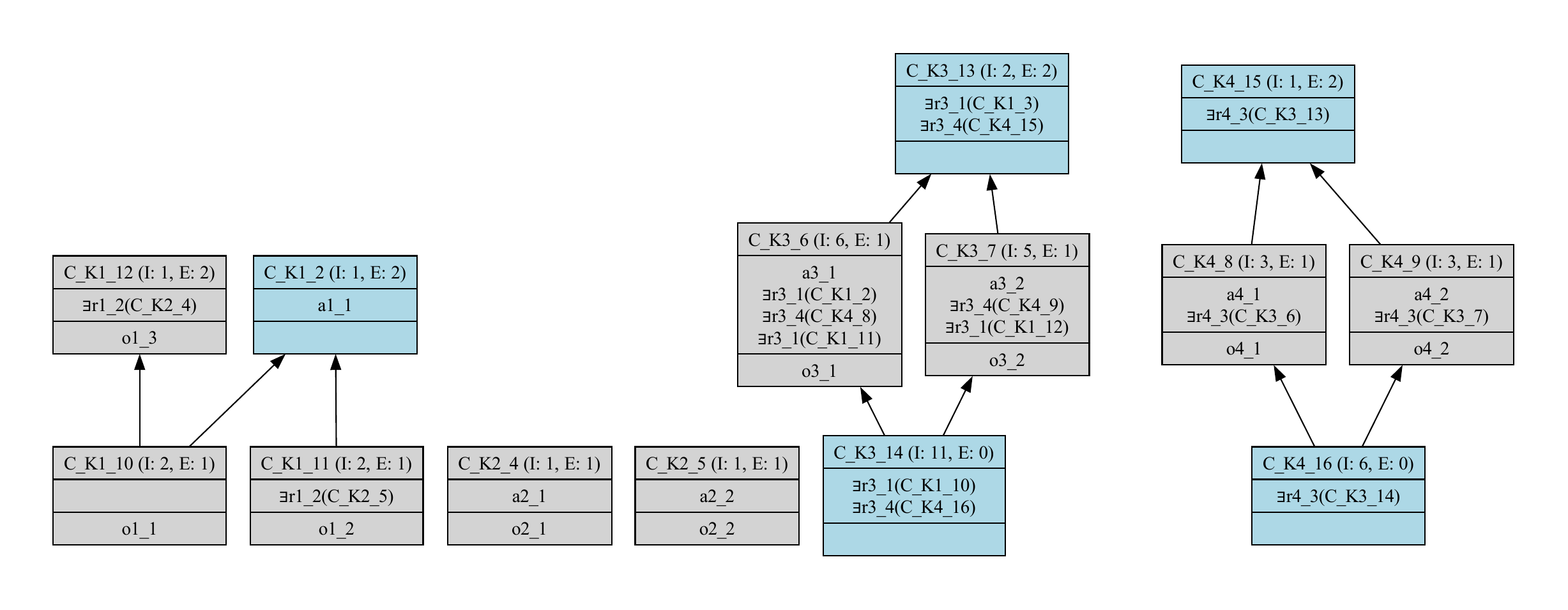}
\caption{AOC-posets built with RCA-AOC-conv at steps 3 to 4 for Table \ref{tab_rcf-div-oc}. 
The process converges.}
\label{fig_conv_2}
\end{figure}

At step~0 (Fig.~\ref{fig_conv_1}, top), $K_2$, $K_3$ and $K_4$ are
contexts of identified objects. 
In $K_1$, \texttt{C\_K1\_2} $= (\{o1\_1, o1\_2\}, \{a1\_1\})$ introduces
\texttt{o1\_1} and \texttt{o1\_2}, while the top concept
\texttt{C\_K1\_3} introduces \texttt{o1\_3}; in $K_3$,
\texttt{C\_K3\_6} introduces \texttt{o3\_1} and \texttt{C\_K3\_7}
introduces \texttt{o3\_2}; 
 in $K_4$,
\texttt{C\_K4\_8} introduces \texttt{o4\_1} and \texttt{C\_K4\_9}
introduces \texttt{o4\_2}.
 
At step~1 (Fig.~\ref{fig_conv_1}, second row), two noteworthy concepts appear.
First, \texttt{C\_K1\_10} $= (\{o1\_1\}, \{a1\_1,
\exists r1\_2(\texttt{C\_K2\_4})\})$ is the new object-concept of
\texttt{o1\_1}: it introduces no attribute (\texttt{a1\_1} is
introduced by \texttt{C\_K1\_2} and
$\exists r1\_2(\penalty500\texttt{C\_K2\_4})$ by \texttt{C\_K1\_12}) and is only
obtained by recomputing the object-concept of \texttt{o1\_1}, whose
closure is refined by the new attribute (point (ii) of the
implementation note above). 
Second, \texttt{C\_K3\_13} $= (\{o3\_1, o3\_2\},$ 
$\{\exists r3\_1(\texttt{C\_K1\_3})\})$ has an empty simplified
extent and refers to \texttt{C\_K1\_3}, which no longer belongs to the
poset of $K_1$ at this step: as \texttt{o1\_3} is now introduced by
\texttt{C\_K1\_12}, \texttt{C\_K1\_3} no longer introduces any element and
is removed, while the attribute
$\exists r3\_1(\texttt{C\_K1\_3})$, whose incidence is fixed
(Definition~\ref{def_attridentity}), is kept: it is our first example
of a dangling attribute. Note that the extent $\{o1\_1, o1\_2, o1\_3\}$
of \texttt{C\_K1\_3} belongs to $Ext_{\mathcal{P}_1^0}$ but not to
$Ext_{\mathcal{P}_1^1}$: RCA-AOC-conv does not preserve the
step-to-step inclusion of the structures
(Sect.~\ref{sec_rcaacoconv}). This dangling reference will also be the
seed of the divergence of RCA-AOC.
 
At step~2 (Fig.~\ref{fig_conv_1}, third row), the presence of
\texttt{C\_K1\_11} and \texttt{C\_K1\_12} refines the intents of the
object-concepts of $K_3$: \texttt{C\_K3\_6} introduces \texttt{o3\_1}
and gains $\exists r3\_1(\texttt{C\_K1\_11})$, and \texttt{C\_K3\_7}
introduces \texttt{o3\_2} and gains $\exists r3\_1(\texttt{C\_K1\_12})$. The attribute
$\exists r3\_1(\texttt{C\_K1\_10})$, owned by no object, is
introduced by \texttt{C\_K3\_14}, whose extent is empty. In parallel, \texttt{C\_K3\_13}
propagates into the circuit:   \texttt{C\_K4\_15} $= (\{o4\_1, o4\_2\},$
$\{\exists r4\_3(\penalty500\texttt{C\_K3\_13})\})$
 is created, with an empty simplified extent as well. 
 
At step~3 (Fig.~\ref{fig_conv_2}, first row), the convergence mechanism of
RCA-AOC-conv becomes visible. Since \texttt{C\_K4\_15} persists, because \texttt{C\_K3\_13} also persists, the
new attribute $\exists r3\_4(\penalty500\texttt{C\_K4\_15})$ is simply added to
the intent of the existing concept \texttt{C\_K3\_13}, whose extent
$\{o3\_1, o3\_2\}$ is already present: no concept is created in
$K_3$, and the mutual dependency between $K_3$ and $K_4$ is resolved
instead of oscillating. Similarly, \texttt{C\_K4\_16} introduces
$\exists r4\_3(\texttt{C\_K3\_14})$ with an empty extent. At step~4
(Fig.~\ref{fig_conv_2}, second row), the only change is the addition of
$\exists r3\_4(\texttt{C\_K4\_16})$ to the intent of
\texttt{C\_K3\_14}: the sets of extents are unchanged, the posets of
steps~3 and~4 are equivalent (Definition~\ref{def_equiv}), and the algorithm stops at step 5. In the final result, $\exists r3\_1(\texttt{C\_K1\_3})$
remains a dangling attribute, in the intent of \texttt{C\_K3\_13}: it
is interpreted through the extent $\{o1\_1, o1\_2, o1\_3\}$ of the
concept \texttt{C\_K1\_3} it was originally built on.

Figure~\ref{fig_div} shows the concept-posets computed by \textsf{fca4j} with
RCA-AOC at steps~2 to~5, and the divergence of the process.
The first two steps are identical and not shown on the figure. 
At step~2,
\texttt{C\_K3\_13} disappears, since \texttt{C\_K1\_3} does not exist
at step~1 and no other attribute is shared by \texttt{o3\_1} and
\texttt{o3\_2}; \texttt{C\_K4\_15} then disappears at step~3, while an
equivalent of \texttt{C\_K3\_13} reappears (\texttt{C\_K3\_16}), built
on the step 2 concept \texttt{C\_K4\_15}. The concepts on the circuit
appear alternately, as in the counterexample of
Table~\ref{tab:gsh-exists}: no two successive steps are equivalent,
and the process loops with period two (steps~3 and~5 share the same
extents, as do steps~4 and~6). 

Note also that a variant of
RCA-AOC-conv that would only compute the attribute-concepts of the new
attributes, skipping the recomputation of object-concepts (point (ii)
of the implementation note above), would converge as well, but would
never create
\texttt{C\_K1\_10}. 

\enlargethispage{\baselineskip}

\begin{figure}[htb]
\centering
\includegraphics[width=\linewidth]{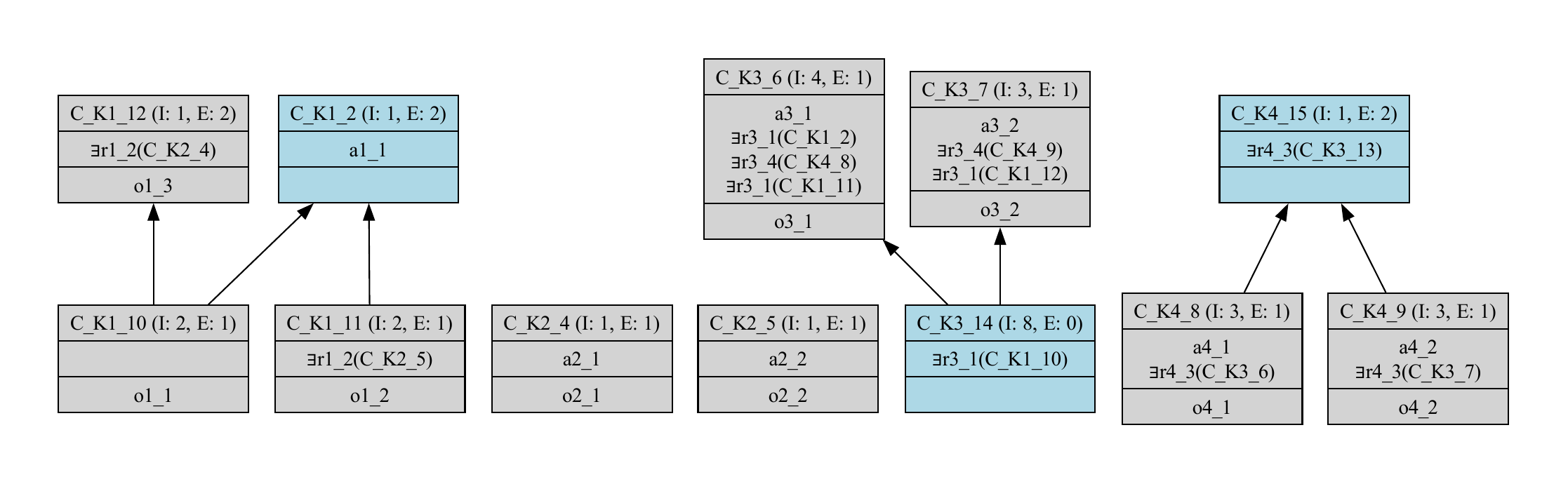}
\includegraphics[width=\linewidth]{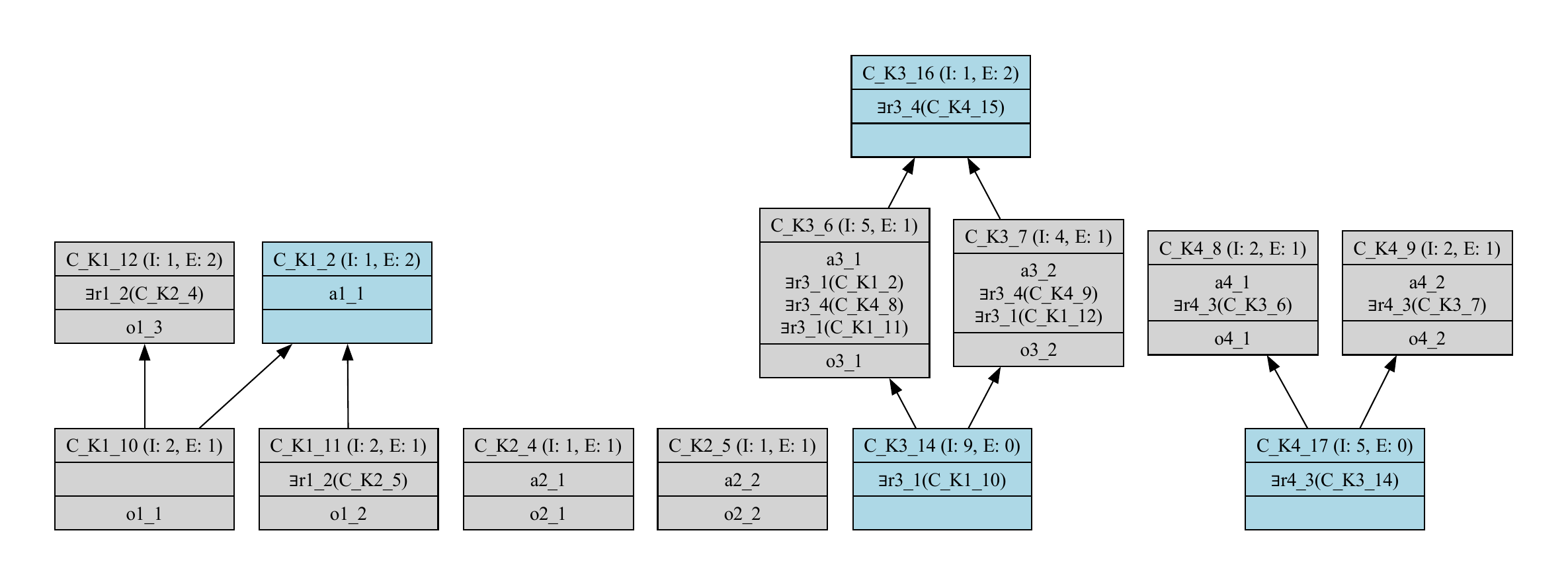}
\includegraphics[width=\linewidth]{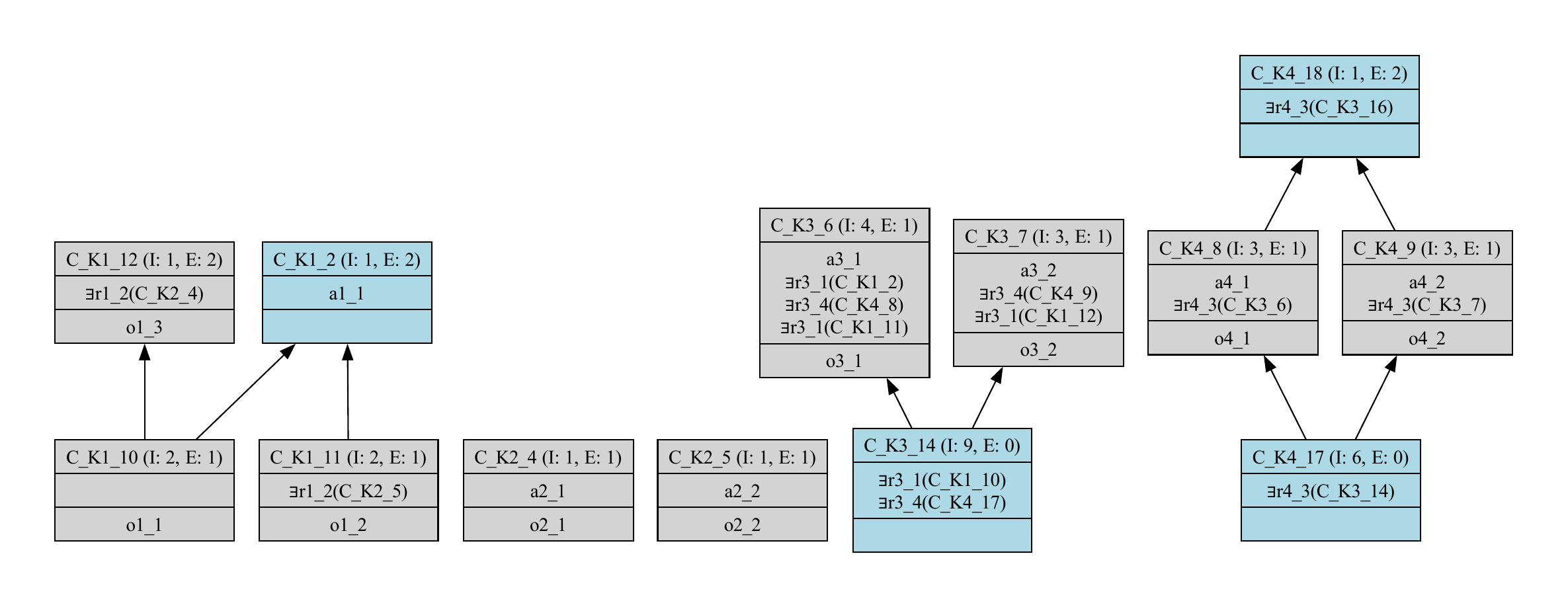}
\includegraphics[width=\linewidth]{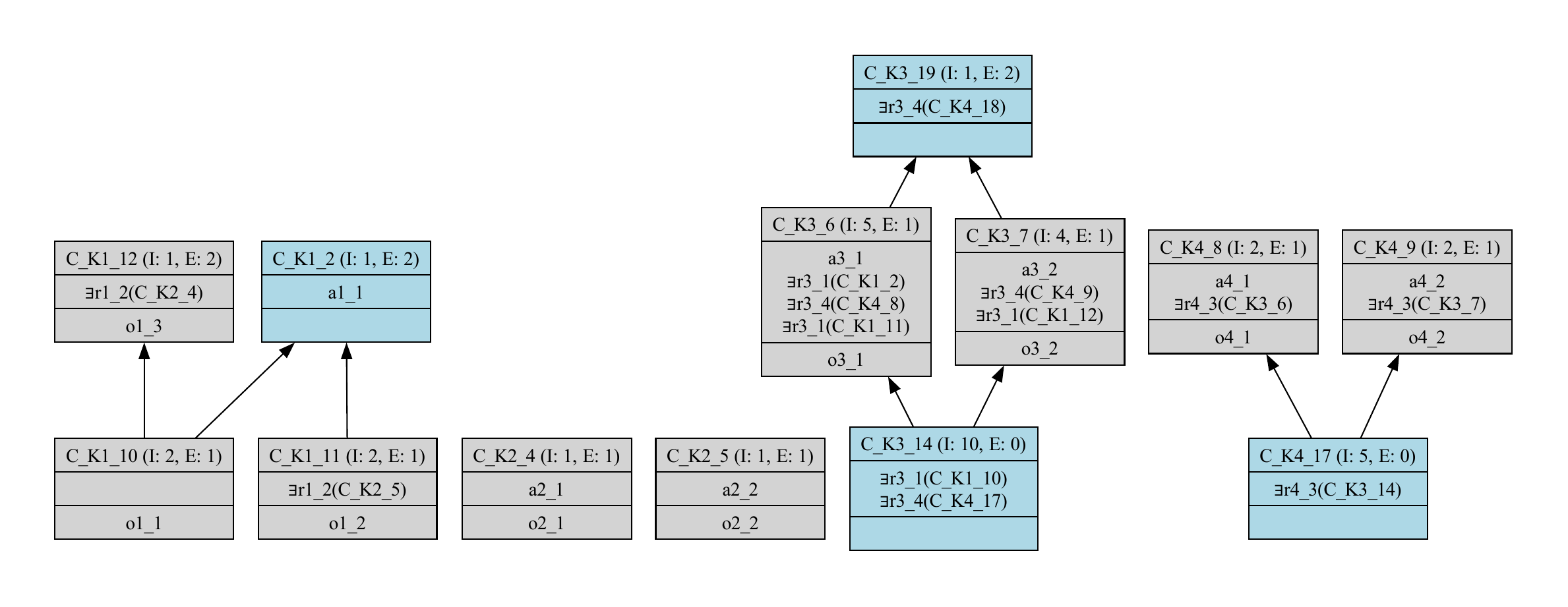}
\caption{AOC-posets built with RCA-AOC at steps 2 to 5 for Table \ref{tab_rcf-div-oc}. The process diverges.}
\label{fig_div}
\end{figure}

%% file: bib-propre.bib
@book{hastie2009elements,
  title={The Elements of Statistical Learning: Data Mining, Inference, and Prediction},
  author={Hastie, Trevor and Tibshirani, Robert and Friedman, Jerome},
  edition={2},
  year={2009},
  publisher={Springer}
}

@article{LeCun2015,
	Author = {LeCun, Yann and Bengio, Yoshua and Hinton, Geoffrey},
	Da = {2015/05/01},
	Doi = {10.1038/nature14539},
	Id = {LeCun2015},
	Isbn = {1476-4687},
	Journal = {Nature},
	Number = {7553},
	Pages = {436--444},
	Title = {Deep learning},
	Ty = {JOUR},
	OPTUrl = {https://doi.org/10.1038/nature14539},
	Volume = {521},
	Year = {2015},
	OPTBdsk-Url-1 = {https://doi.org/10.1038/nature14539}}

@article{10.1007/s10618-024-01041-y,
author = {Atzmueller, Martin and F\"{u}rnkranz, Johannes and Kliegr, Tom\'{a}\v{s} and Schmid, Ute},
title = {Explainable and interpretable machine learning and data mining},
year = {2024},
issue_date = {Sep 2024},
publisher = {Kluwer Academic Publishers},
address = {USA},
volume = {38},
number = {5},
issn = {1384-5810},
OPTurl = {https://doi.org/10.1007/s10618-024-01041-y},
doi = {10.1007/s10618-024-01041-y},
journal = {Data Min. Knowl. Discov.},
OPTmonth = jul,
pages = {2571–2595},
numpages = {25}
}

@article{10.1145/3236009,
author = {Guidotti, Riccardo and Monreale, Anna and Ruggieri, Salvatore and Turini, Franco and Giannotti, Fosca and Pedreschi, Dino},
title = {A Survey of Methods for Explaining Black Box Models},
year = {2018},
issue_date = {September 2019},
publisher = {Association for Computing Machinery},
address = {New York, NY, USA},
volume = {51},
number = {5},
issn = {0360-0300},
OPTurl = {https://doi.org/10.1145/3236009},
doi = {10.1145/3236009},
journal = {ACM Comput. Surv.},
OPTmonth = aug,
articleno = {93},
numpages = {42}
}

@article{DBLP:journals/isci/KuznetsovM18,
  author       = {Sergei O. Kuznetsov and
                  Tatiana P. Makhalova},
  title        = {On interestingness measures of formal concepts},
  journal      = {Inf. Sci.},
  volume       = {442-443},
  pages        = {202--219},
  year         = {2018},
  OPTurl          = {https://doi.org/10.1016/j.ins.2018.02.032},
  OPTdoi          = {10.1016/J.INS.2018.02.032},
  bibsource    = {dblp computer science bibliography, https://dblp.org}
}

@misc{aranda_corral_2026_18608706,
  author       = {G.A. Aranda-Corral and A. Bundy and J. Borrego-Díaz and P.Y. Chan},
  title        = {{Grounding problem in FCA by means of LLMs}},
  OPTmonth        = feb,
  year         = 2026,
  howpublished    = {Zenodo, Slides  presented at conference Concepts 24},
  url          = {https://zenodo.org/records/18608706}}

@misc{arandaConcepts2024,
author       = {G.A. Aranda-Corral and A. Bundy and J. Borrego-Díaz and P.Y. Chan},
title        = {{Grounding problem in Formal Concept Analysis by means of Large Language Models}},
  year         = {2024},
  howpublished = {{Workshop on Late Breaking Advances on Conceptual Structures @ Concepts 2024}},
  OPTurl={https://concepts2024.uca.es/lbacs/}
}

@article{buzmakov2014concept,
  title={Is concept stability a measure for pattern selection?},
  author={Buzmakov, Aleksey and Kuznetsov, Sergei O. and Napoli, Amedeo},
  journal={Procedia Computer Science},
  volume={31},
  pages={918--927},
  year={2014},
  publisher={Elsevier}
}

@inproceedings{DBLP:conf/models/ArevaloFHN06,
  author    = {Gabriela Ar{\'e}valo and
               Jean-R{\'e}my Falleri and
               Marianne Huchard and
               Cl{\'e}mentine Nebut},
  title     = {{Building Abstractions in Class Models: Formal Concept Analysis
               in a Model-Driven Approach}},
  booktitle = {MoDELS 2006},
  year      = {2006},
  pages     = {513-527}
}

@inproceedings{atencia:hal-02984963,
  TITLE = {{A guided walk into link key candidate extraction with relational concept analysis}},
  AUTHOR = {Atencia, Manuel and David, J{\'e}r{\^o}me and Euzenat, J{\'e}r{\^o}me and Napoli, Amedeo and Vizzini, J{\'e}r{\'e}my},
  OPTURL = {https://hal.science/hal-02984963},
  BOOKTITLE = {{ISWC 2019}},
  OPTADDRESS = {Auckland, New Zealand},
  OPTPUBLISHER = {{No commercial editor.}},
  PAGES = {1-9},
  YEAR = {2019},
  OPTMONTH = Oct,
  HAL_ID = {hal-02984963},
  HAL_VERSION = {v1},
}

@article{DBLP:journals/dam/AtenciaDENV20,
  author       = {Manuel Atencia and
                  J{\'{e}}r{\^{o}}me David and
                  J{\'{e}}r{\^{o}}me Euzenat and
                  Amedeo Napoli and
                  J{\'{e}}r{\'{e}}my Vizzini},
  title        = {Link key candidate extraction with relational concept analysis},
  journal      = {Discret. Appl. Math.},
  volume       = {273},
  pages        = {2--20},
  year         = {2020},
  OPTurl          = {https://doi.org/10.1016/j.dam.2019.02.012},
  doi          = {10.1016/J.DAM.2019.02.012},
  OPTbiburl       = {https://dblp.org/rec/journals/dam/AtenciaDENV20.bib},
  OPTbibsource    = {dblp computer science bibliography, https://dblp.org}
}

@inproceedings{DBLP:conf/icws/AzmehDHHMT11,
  author    = {Zeina Azmeh and
               Maha Driss and
               Fady Hamoui and
               Marianne Huchard and
               Naouel Moha and
               Chouki Tibermacine},
  title     = {{Selection of Composable Web Services Driven by User Requirements}},
  booktitle = {ICWS 2011},
  year      = {2011},
  pages     = {395-402}
}

@misc{Valtchev2003GaliciaA,
  title={Galicia : an open platform for lattices},
  author={Petko Valtchev and David Grosser and Cyril Roume and Mohamed {Rouane-Hac\`ene}},
  year={2003},
  url={https://api.semanticscholar.org/CorpusID:16650318}
}

@inproceedings{azmehCla11,
	Author = {Zeina Azmeh and Marianne Huchard and Amedeo Napoli and Mohamed Rouane-Hac\`ene and Petko Valtchev},
	Booktitle = {CLA 2011: Concept Lattices and their Applications},
	Isbn = {978-2-905267-78-8},
	Pages = {377--392},
	Title = {Querying Relational Concept Lattices},
	Year = {2011}}

@Book{baader03,
  editor =     {Franz Baader and Diego Calvanese and Deborah McGuinness and Daniele Nardi and Peter Patel-Schneider},
  title =      {The Description Logic Handbook
Theory, Implementation and Applications},
  publisher =          {Cambridge University Press},
  year =       {2003},
  address =    {Cambridge, MA}}

@article{DBLP:journals/ijar/BazinGK24,
  author       = {Alexandre Bazin and
                  Jessie Galasso and
                  Giacomo Kahn},
  title        = {Polyadic relational concept analysis},
  journal      = {Int. J. Approx. Reason.},
  volume       = {164},
  pages        = {109067},
  year         = {2024},
  OPTurl          = {https://doi.org/10.1016/j.ijar.2023.109067},
  doi          = {10.1016/J.IJAR.2023.109067},
  bibsource    = {dblp computer science bibliography, https://dblp.org}
}

@inproceedings{DBLP:conf/cla/BelohlavekV05,
  author       = {Radim Belohl{\'{a}}vek and
                  Vil{\'{e}}m Vychodil},
  OPTeditor       = {Radim Belohl{\'{a}}vek and
                  V{\'{a}}clav Sn{\'{a}}sel},
  title        = {What is a fuzzy concept lattice?},
  booktitle    = {{CLA} 2005: Concept Lattices and their Applications},
   series       = {{CEUR}-WS Proc.},
  volume       = {162},
  publisher    = {CEUR-WS.org},
  year         = {2005},
  OPTurl          = {https://ceur-ws.org/Vol-162/paper4.pdf},
  bibsource    = {dblp computer science bibliography, https://dblp.org}
}

@inproceedings{bendaoud08b,
	Author = {Rokia Bendaoud and Amedeo Napoli and Yannick Toussaint},
	Booktitle = {EKAW 2008},
	Pages = {156-171},
	Series = {LNCS 5268},
	Title = {{Formal Concept Analysis: A unified framework for building and refining ontologies}},
	Year = 2008}

@inproceedings{hermes2012,
  author    = {Anne Berry and
               Marianne Huchard and
               Amedeo Napoli and
               Alain Sigayret},
  title     = {{Hermes: an efficient algorithm for building Galois Sub-hierarchies}},
  booktitle = {CLA 2012: Concept Lattices and their Applications},
  pages = {21-32},
  year      = {2012}
}

@inproceedings{DBLP:conf/eusflat/BoffaM23,
  author       = {Stefania Boffa and
                  Petra Murinov{\'{a}}},
  OPTeditor       = {Sebastia Massanet and
                  Susana Montes and
                  Daniel Ruiz{-}Aguilera and
                  Manuel Gonz{\'{a}}lez Hidalgo},
  title        = {{Logical Relations Between T-Scaling Quantifiers and Their Implications  in Fuzzy Relational Concept Analysis}},
  booktitle    = {Fuzzy Logic and Technology, and Aggregation Operators -  {EUSFLAT}
                  2023, and    {AGOP} 2023, Proceedings},
  series       = {LNCS},
  volume       = {14069},
  pages        = {393--404},
  publisher    = {Springer},
  year         = {2023},
  OPTurl          = {https://doi.org/10.1007/978-3-031-39965-7\_33},
  doi          = {10.1007/978-3-031-39965-7\_33},
  bibsource    = {dblp computer science bibliography, https://dblp.org} }

@incollection{braud2022,
  TITLE = {{Dealing with large volumes of complex relational data using RCA}},
  AUTHOR = {Braud, Agn{\`e}s and Dolques, Xavier and Gutierrez, Alain and Huchard, Marianne and Keip, Priscilla and {Le Ber}, Florence and Martin, Pierre and Nica, Cristina and Silvie, Pierre},
  OPTURL = {https://hal.science/hal-03744342},
  BOOKTITLE = {{Complex Data Analytics with Formal Concept Analysis}},
   OPTEDITOR = {Rokia Missaoui and L{\'e}onard Kwuida and Talel Abdessalem},
  PUBLISHER = {{Springer}},
  PAGES = {105--134},
  YEAR = {2022},
  DOI = {10.1007/978-3-030-93278-7\_5},
  HAL_ID = {hal-03744342},
}

@article{DBLP:journals/kbs/BraudDHB18,
  author       = {Agn{\`{e}}s Braud and
                  Xavier Dolques and
                  Marianne Huchard and
                  Florence {Le Ber}},
  title        = {Generalization effect of quantifiers in a classification based on
                  relational concept analysis},
  journal      = {Knowl. Based Syst.},
  volume       = {160},
  pages        = {119--135},
  year         = {2018},
 OPTurl          = {https://doi.org/10.1016/j.knosys.2018.06.011},
  doi          = {10.1016/J.KNOSYS.2018.06.011},
  bibsource    = {dblp computer science bibliography, https://dblp.org}
}

@inproceedings{DBLP:conf/iccs/DaoHHRV04,
  author       = {Michel Dao and
                  Marianne Huchard and
                  Mohamed {Rouane-Hac\`ene} and
                  Cyril Roume and
                  Petko Valtchev},
  OPTeditor       = {Karl Erich Wolff and
                  Heather D. Pfeiffer and
                  Harry S. Delugach},
  title        = {Improving Generalization Level in {UML} Models Iterative Cross Generalization
                  in Practice},
  booktitle    = {Conceptual Structures at Work:  {ICCS} 2004},
  series       = {LNCS},
  volume       = {3127},
  pages        = {346--360},
  publisher    = {Springer},
  year         = {2004},
  OPTurl          = {https://doi.org/10.1007/978-3-540-27769-9\_23},
  doi          = {10.1007/978-3-540-27769-9\_23},
  bibsource    = {dblp computer science bibliography, https://dblp.org}
}

@inproceedings{DBLP:conf/f-egc/DolquesBHN13a,
  author       = {Xavier Dolques and
                  Florence {Le Ber} and
                  Marianne Huchard and
                  Cl{\'{e}}mentine Nebut},
  OPTeditor       = {Fabrice Guillet and
                  Bruno Pinaud and
                  Gilles Venturini and
                  Djamel Abdelkader Zighed},
  title        = {Relational Concept Analysis for Relational Data Exploration, Vol. 5},
  booktitle    = {Advances in Knowledge Discovery and Management},
  series       = {SCI},
  volume       = {615},
  pages        = {57--77},
  publisher    = {Springer},
  year         = {2013},
  OPTurl          = {https://doi.org/10.1007/978-3-319-23751-0\_4},
  OPTdoi          = {10.1007/978-3-319-23751-0\_4},
  bibsource    = {dblp computer science bibliography, https://dblp.org}
}

@inproceedings{DBLP:conf/icfca/DolquesBHB19,
  author       = {Xavier Dolques and
                  Agn{\`{e}}s Braud and
                  Marianne Huchard and
                  Florence {Le Ber}},
  OPTeditor       = {Diana Cristea and
                  Florence {Le Ber} and
                  Rokia Missaoui and
                  L{\'{e}}onard Kwuida and
                  Baris Sertkaya},
  title        = {RCAexplore, a {FCA} based Tool to Explore Relational Data},
  booktitle    = {Supplementary Proceedings of {ICFCA} 2019 Conference},
  series       = {{CEUR}-WS Proc.},
  volume       = {2378},
  pages        = {55--59},
  publisher    = {CEUR-WS.org},
  year         = {2019},
  OPTurl          = {https://ceur-ws.org/Vol-2378/shortAT5.pdf},
  bibsource    = {dblp computer science bibliography, https://dblp.org}
}

@inproceedings{dolques2010learning,
  title={Learning transformation rules from transformation examples: An approach based on relational concept analysis},
  author={Dolques, Xavier and Huchard, Marianne and Nebut, Cl{\'e}mentine and Reitz, Philippe},
  booktitle={2010 14th IEEE International Enterprise Distributed Object Computing Conference Workshops},
  pages={27--32},
  year={2010},
  organization={IEEE}}

@article{DBLP:journals/fuin/DolquesHNR12,
  author    = {Xavier Dolques and
               Marianne Huchard and
               Cl{\'e}mentine Nebut and
               Philippe Reitz},
  title     = {{Fixing Generalization Defects in UML Use Case Diagrams}},
  journal   = {Fundam. Inform.},
  volume    = {115},
  number    = {4},
  year      = {2012},
  pages     = {327-356},
  ee        = {http://dx.doi.org/10.3233/FI-2012-658},
  bibsource = {DBLP, http://dblp.uni-trier.de}
}

@inproceedings{DBLP:conf/cla/DolquesBH13,
  author    = {Xavier Dolques and
               Florence {Le Ber} and
               Marianne Huchard},
  title     = {{AOC-Posets: a Scalable Alternative to Concept Lattices for
               Relational Concept Analysis}},
  booktitle = {CLA 2013: Concept Lattices and their Applications},
   series       = {{CEUR}-WS Proc.},
   year      = {2013},
  pages     = {129-140},
  ee        = {http://ceur-ws.org/Vol-1062/paper11.pdf},
  bibsource = {DBLP, http://dblp.uni-trier.de}
}

@article{ijgis2016,
author = {Xavier Dolques and Florence {Le Ber} and Marianne Huchard and Corinne Grac},
title = {{Performance-friendly rule extraction in large water data-sets with AOC posets and relational concept analysis}},
journal = {International Journal of General Systems},
volume = {45},
number = {1},
pages = {1--24},
year = {2016},
OPTdoi = {10.1080/03081079.2015.1072927},
OPTURL = {http://dx.doi.org/10.1080/03081079.2015.1072927},
eprint = {http://dx.doi.org/10.1080/03081079.2015.1072927}
}

@inproceedings{DBLP:conf/icfca/DolquesMBHB14,
  author       = {Xavier Dolques and
                  Kartick Chandra Mondal and
                  Agn{\`{e}}s Braud and
                  Marianne Huchard and
                  Florence {Le Ber}},
  OPTeditor       = {Cynthia Vera Glodeanu and  Mehdi Kaytoue and
                  Christian Sacarea},
  title        = {{RCA} as a Data Transforming Method: {A} Comparison with Propositionalisation},
  booktitle    = {Formal Concept Analysis -  {ICFCA} 2014},
  series       = {LNCS},
  volume       = {8478},
  pages        = {112--127},
  publisher    = {Springer},
  year         = {2014},
  OPTurl          = {https://doi.org/10.1007/978-3-319-07248-7\_9},
  OPTdoi          = {10.1007/978-3-319-07248-7\_9},
  bibsource    = {dblp computer science bibliography, https://dblp.org}}

@article{euzenat-jair25,
  author       = {J{\'{e}}r{\^{o}}me Euzenat},
  title        = {The Fixed-Point Semantics of Relational Concept Analysis},
  journal      = {J. Artif. Intell. Res.},
  volume       = {83},
  year         = {2025},
  OPTurl          = {https://doi.org/10.1613/jair.1.17882},
  doi          = {10.1613/JAIR.1.17882},
  bibsource    = {dblp computer science bibliography, https://dblp.org}}

@inproceedings{ferre2018hierarchies,
  title={{How Hierarchies of Concept Graphs Can Facilitate the Interpretation of RCA Lattices?}},
  author={Ferr{\'e}, S{\'e}bastien and Cellier, Peggy},
  booktitle={CLA 2018 : Concept Lattices and their Applications},
  Series = {CEUR-WS Proc.},
  volume ={2123},
  year={2018}}

@article{ferre2020graph,
  title={{Graph-FCA: An} extension of formal concept analysis to knowledge graphs},
  author={Ferr{\'e}, S{\'e}bastien and Cellier, Peggy},
  journal={Discrete applied mathematics},
  volume={273},
  pages={81--102},
  year={2020},
  publisher={Elsevier}}

@article{fokou-ijar2025,
  TITLE = {{Theoretical comparison of Relational Concept Analysis (RCA) and Graph-FCA (GCA)}},
  AUTHOR = {Fokou, Vanessa and Cellier, Peggy and Dolques, Xavier and Ferr{\'e}, S{\'e}bastien and Le Ber, Florence},
  OPTURL = {https://hal.science/hal-05138994},
  JOURNAL = {{Int. Journal of Approximate Reasoning}},
  PUBLISHER = {{Elsevier}},
  VOLUME = {186},
  PAGES = {109496},
  YEAR = {2025},
  OPTMONTH = Nov,
  DOI = {10.1016/j.ijar.2025.109496},
  HAL_ID = {hal-05138994}}

@inproceedings{fokou-elhaff2024,
  TITLE = {{Exploring Old Arabic Remedies with Formal and Relational Concept Analysis}},
  AUTHOR = {Fokou, Vanessa and El Haff, Karim and Braud, Agn{\`e}s and Dolques, Xavier and Le Ber, Florence and Pitchon, Veronique},
  OPTURL = {https://hal.science/hal-04622852},
  BOOKTITLE = {{Conceptual Knowledge Structure - Concepts 2024, Proceedings}},
  OPTADDRESS = {Cadiz, Spain},
  YEAR = {2024},
  OPTMONTH = Sep,
  HAL_ID = {hal-04622852}}

@book{Gant99a,
  author = {B. Ganter and R. Wille},
  title = {Formal Concept Analysis: Mathematical Foundations},
  publisher = {Springer Verlag},
  year = {1999}}

@inproceedings{Godi93a,
   author = {R. Godin and H. Mili},
   title = {Building and {M}aintaining {A}nalysis-{L}evel {C}lass {H}ierarchies
   using   {G}alois {L}attices},
   booktitle = {OOPSLA '93},
   OPTlocation = {Washington, DC, USA},
   pages = {394--410},
   year = {1993},
   volume = {28}}

@article{Godi98a,
  author = {R. Godin and H. Mili and G. W. Mineau and R. Missaoui
  and A. Arfi and T.-T. Chau},
  title = {Design of {C}lass {H}ierarchies based on {C}oncept ({G}alois)
  {L}attices},
  journal = {Theory and Practice of Object Systems},
  year = {1998},
  volume = {4},
  number = {2},
  pages = {117-134} }

@inproceedings{DBLP:conf/cla/GuediMHN13,
  author       = {Abdoulkader Osman Gu{\'{e}}di and
                  Andr{\'{e}} Miralles and
                  Marianne Huchard and
                  Cl{\'{e}}mentine Nebut},
  OPTeditor       = {Manuel Ojeda{-}Aciego and       Jan Outrata},
  title        = {A Practical Application of Relational Concept Analysis to Class Model
                  Factorization: Lessons Learned from a Thematic Information System},
  booktitle    = {CLA 2013 : Concept Lattices   and Their Applications},
  series       = {{CEUR-WS} Proc.},
  volume       = {1062},
  pages        = {9--20},
  publisher    = {CEUR-WS.org},
  year         = {2013},
  OPTurl          = {https://ceur-ws.org/Vol-1062/paper1.pdf},
  bibsource    = {dblp computer science bibliography, https://dblp.org}
}

@inproceedings{DBLP:conf/concepts/GuenouneGHLMMZ25,
  author       = {Hani Guenoune and
                  Alain Gutierrez and
                  Marianne Huchard and
                  Mathieu Lafourcade and
                  Pierre Martin and
                  Andr{\'{e}} Miralles and
                  Huaxi Zhang},
  OPTeditor       = {Peggy Cellier and
                  Bernhard Ganter and
                  Rokia Missaoui},
  title        = {{LLM-Assisted Relational Concept Analysis for Class Model Restructuring}},
  booktitle    = {Conceptual Knowledge Structures -  {CONCEPTS} 2025, Proceedings},
  series       = {LNCS},
  volume       = {15941},
  pages        = {107--123},
  publisher    = {Springer},
  year         = {2025},
  OPTurl          = {https://doi.org/10.1007/978-3-032-03364-2\_7},
  doi          = {10.1007/978-3-032-03364-2\_7},
  bibsource    = {dblp computer science bibliography, https://dblp.org}}

@inproceedings{DBLP:conf/concepts/GutierrezHMZ25,
  author       = {Alain Gutierrez and
                  Marianne Huchard and
                  Pierre Martin and
                  Huaxi Zhang},
  OPTeditor       = {Peggy Cellier and
                  Bernhard Ganter and
                  Rokia Missaoui},
  title        = {Empowering Relational Concept Analysis Using Large Language Model
                  Knowledge Delivery},
  booktitle    = {Conceptual Knowledge Structures -
                  {CONCEPTS} 2025, Proceedings},
  series       = {LNCS},
  volume       = {15941},
  pages        = {124--139},
  publisher    = {Springer},
  year         = {2025},
  OPTurl          = {https://doi.org/10.1007/978-3-032-03364-2\_8},
  OPTdoi          = {10.1007/978-3-032-03364-2\_8},
  OPTbiburl       = {https://dblp.org/rec/conf/concepts/GutierrezHMZ25.bib},
  OPTbibsource    = {dblp computer science bibliography, https://dblp.org}}

@inproceedings{DBLP:conf/icfca/RouaneHNV07,
  author       = {Mohamed {Rouane-Hac\`ene} and  Marianne Huchard and
                  Amedeo Napoli and    Petko Valtchev},
  OPTeditor       = {Sergei O. Kuznetsov and
                  Stefan Schmidt},
  title        = {A Proposal for Combining Formal Concept Analysis and Description Logics  for Mining Relational Data},
  booktitle    = {Formal Concept Analysis,  {ICFCA} 2007,    Proceedings},
  series       = {LNCS},
  volume       = {4390},
  pages        = {51--65},
  publisher    = {Springer},
  year         = {2007},
  OPTurl          = {https://doi.org/10.1007/978-3-540-70901-5\_4},
  doi          = {10.1007/978-3-540-70901-5\_4},
  bibsource    = {dblp computer science bibliography, https://dblp.org}}

@inproceedings{rvn-iccs11,
	Author = {Mohamed Rouane-Hac\`ene and Petko Valtchev and Roger Nkambou},
	Bibsource = {DBLP, http://dblp.uni-trier.de},
	Booktitle = {ICCS 2011},
	Pages = {257-269},
	Title = {{Supporting Ontology Design through Large-Scale FCA-Based Ontology Restructuring}},
	Year = {2011}}

@inproceedings{Hitz04,
  author    = {P. Hitzler},
  title     = {{Default Reasoning over Domains and Concept Hierarchies}},
  booktitle = {{KI 2004: Advances in Artificial Intelligence}},
  year      = {2004},
  pages     = {351-365},
  publisher = {Springer Verlag},
  series    = {LNCS},
  volume    = {3238}}

@article{Huch00a,
  author = {M. Huchard and H. Dicky and H. Leblanc},
  title = {{G}alois {L}attice as a {F}ramework to specify {A}lgorithms {B}uilding
  {C}lass {H}ierarchies},
  journal = {Theoretical Informatics and Applications},
  year = {2000},
  volume = {34},
  pages = {521-548} }

@ARTICLE{HuchardHRV07,
  author = { Huchard, Marianne and  Rouane-Hac\`ene, Mohamed and  Roume,Cyril and  Valtchev,Petko},
  title = {Relational concept discovery in structured datasets},
  journal = {Ann. Math. Artif. Intell.},
  year = {2007},
  volume = {49},
  pages = {39-76},
  number = {1-4},
  bibsource = {DBLP, http://dblp.uni-trier.de},
  ee = {http://dx.doi.org/10.1007/s10472-007-9056-3}}

@article{DBLP:journals/dam/KottersE20,
  author       = {Jens K{\"{o}}tters and
                  Peter W. Eklund},
  title        = {Conjunctive query pattern structures: {A} relational database model
                  for Formal Concept Analysis},
  journal      = {Discret. Appl. Math.},
  volume       = {273},
  pages        = {144--171},
  year         = {2020},
  OPTurl          = {https://doi.org/10.1016/j.dam.2019.08.019},
  doi          = {10.1016/J.DAM.2019.08.019},
  bibsource    = {dblp computer science bibliography, https://dblp.org}}

@article{DBLP:journals/kais/LeutwylerLFPTT24,
  author       = {Nicol{\'{a}}s Leutwyler and  Mario Lezoche and   Chiara Franciosi and    Herv{\'{e}} Panetto and    Laurent Teste and    Diego Torres},
  title        = {Methods for concept analysis and multi-relational data mining: a systematic
                  literature review},
  journal      = {Knowl. Inf. Syst.},
  volume       = {66},
  number       = {9},
  pages        = {5113--5150},
  year         = {2024},
  OPTurl          = {https://doi.org/10.1007/s10115-024-02139-x},
  doi          = {10.1007/S10115-024-02139-X},
  bibsource    = {dblp computer science bibliography, https://dblp.org}}

@inproceedings{DBLP:conf/icail/MimouniFNBS13,
  author       = {Nada Mimouni and
                  Meritxell Fern{\'{a}}ndez and
                  Adeline Nazarenko and
                  Dani{\`{e}}le Bourcier and
                  Sylvie Salotti},
  editor       = {Enrico Francesconi and Bart Verheij},
  title        = {A relational approach for information retrieval on {XML} legal sources},
  booktitle    = {International Conference on Artificial Intelligence and Law, {ICAIL}'13},
  pages        = {212--216},
  publisher    = {{ACM}},
  year         = {2013},
  OPTurl          = {https://doi.org/10.1145/2514601.2514629},
  doi          = {10.1145/2514601.2514629},
  bibsource    = {dblp computer science bibliography, https://dblp.org}}

@inproceedings{DBLP:conf/cla/MirallesMHNDD15,
  author       = {Andr{\'{e}} Miralles and
                  Guilhem Molla and
                  Marianne Huchard and
                  Cl{\'{e}}mentine Nebut and
                  Laurent Deruelle and
                  Mustapha Derras},
  OPTeditor       = {Sadok Ben Yahia and          Jan Konecny},
  title        = {Class Model Normalization - Outperforming Formal Concept Analysis
                  Approaches with AOC-posets},
  booktitle    = {CLA 2015 : Concept Lattices and Their Applications},
   series       = {{CEUR}-WS Proc.},
  volume       = {1466},
  pages        = {111--122},
  publisher    = {CEUR-WS.org},
  year         = {2015},
  OPTurl          = {https://ceur-ws.org/Vol-1466/paper09.pdf},
  bibsource    = {dblp computer science bibliography, https://dblp.org} }

@inproceedings{DBLP:conf/icfca/MohaHVG08,
  author    = {Naouel Moha and  Mohamed Rouane-Hac\`ene and
               Petko Valtchev and   Yann-Ga{\"e}l Gu{\'e}h{\'e}neuc},
  title     = {{Refactorings of Design Defects Using Relational Concept
               Analysis}},
  booktitle = {Formal Concept Analysis, ICFCA 2008},
  year      = {2008},
  pages     = {289–304},
publisher = {Springer},
series = {LNAI},
volume = {4933}}

@article{nica-kbs2020,
title = {{FastRCA-Seq:} An efficient approach for extracting hierarchies of multilevel closed partially-ordered patterns},
journal = {Knowledge-Based Systems},
volume = {210},
pages = {106533},
year = {2020},
issn = {0950-7051},
doi = {https://doi.org/10.1016/j.knosys.2020.106533},
OPTurl = {https://www.sciencedirect.com/science/article/pii/S0950705120306626},
author = {Cristina Nica and Victor-Petru Alm{\u{a}}{\c{s}}an and Adrian Groza} }

@article{nica-dam2020,
  TITLE = {{RCA-Seq: an Original Approach for Enhancing the Analysis of Sequential Data Based on Hierarchies of Multilevel Closed Partially-Ordered Patterns}},
  AUTHOR = {Nica, Cristina and Braud, Agn{\`e}s and Le Ber, Florence},
  OPTURL = {https://hal.science/hal-02081393},
  JOURNAL = {{Discrete Applied Mathematics}},
  PUBLISHER = {{Elsevier}},
  VOLUME = {273},
  PAGES = {232-251},
  YEAR = {2020},
  OPTmonth = Feb,
  DOI = {10.1016/j.dam.2019.02.037},
  HAL_ID = {hal-02081393} }

@Misc{OMG-UML,
    Author       = "{Object Management Group}",
    Title        = "{Unified Modeling Language, Version 2.5.1}",
    Howpublished = "OMG Document Number formal/2017-12 (\url{https://www.omg.org/spec/UML/2.5.1})",
    Year         = 2017 }

@inproceedings{Ossw02,
  author    = {R. Osswald and W. Petersen},
  title     = {{Induction of Classifications from Linguistic Data}},
  booktitle = {ECAI'02 Workshop},
  year      = {2002},
  OPTmonth     = {July} }

@inproceedings{DBLP:conf/f-egc/OuzerdineBDHB19a,
  author       = {Amirouche Ouzerdine and
                  Agn{\`{e}}s Braud and
                  Xavier Dolques and
                  Marianne Huchard and
                  Florence {Le Ber}},
 OPTeditor       = {Rakia Jaziri and
                  Arnaud Martin and
                  Marie{-}Christine Rousset and
                  Lydia Boudjeloud{-}Assala and
                  Fabrice Guillet},
  title        = {Adjusting the Exploration Flow in Relational Concept Analysis - An
                  Experience on a Watercourse Quality Dataset},
  booktitle    = {Advances in Knowledge Discovery and Management, Vol. 9},
  series       = {SCI},
  volume       = {1004},
  pages        = {175--198},
  publisher    = {Springer},
  year         = {2019},
  OPTurl          = {https://doi.org/10.1007/978-3-030-90287-2\_9},
  doi          = {10.1007/978-3-030-90287-2\_9},
  bibsource    = {dblp computer science bibliography, https://dblp.org}
}

@article{petersen2004set,
  title={A set-theoretical approach for the induction of inheritance hierarchies},
  author={Petersen, Wiebke},
  journal={Electronic Notes in Theoretical Computer Science},
  volume={53},
  pages={296--308},
  year={2004},
  publisher={Elsevier} }

@article{DBLP:journals/eswa/PoelmansIKD13,
  author       = {Jonas Poelmans and
                  Dmitry I. Ignatov and
                  Sergei O. Kuznetsov and
                  Guido Dedene},
  title        = {Formal concept analysis in knowledge processing: {A} survey on applications},
  journal      = {Expert Syst. Appl.},
  volume       = {40},
  number       = {16},
  pages        = {6538--6560},
  year         = {2013},
  OPTurl          = {https://doi.org/10.1016/j.eswa.2013.05.009},
  doi          = {10.1016/J.ESWA.2013.05.009},
  bibsource    = {dblp computer science bibliography, https://dblp.org}}

@article{DBLP:journals/eswa/PoelmansKID13,
  author       = {Jonas Poelmans and
                  Sergei O. Kuznetsov and
                  Dmitry I. Ignatov and
                  Guido Dedene},
  title        = {Formal Concept Analysis in knowledge processing: {A} survey on models
                  and techniques},
  journal      = {Expert Syst. Appl.},
  volume       = {40},
  number       = {16},
  pages        = {6601--6623},
  year         = {2013},
  OPTurl          = {https://doi.org/10.1016/j.eswa.2013.05.007},
  doi          = {10.1016/J.ESWA.2013.05.007},
  bibsource    = {dblp computer science bibliography, https://dblp.org}}

@article{rouane2013,
  author    = {Mohamed Rouane-Hac\`ene and
               Marianne Huchard and
               Amedeo Napoli and
               Petko Valtchev},
  title     = {Relational concept analysis: mining concept lattices from
               multi-relational data},
  journal   = {Ann. Math. Artif. Intell.},
  volume    = {67},
  number    = {1},
  year      = {2013},
  pages     = {81-108},
  ee        = {http://dx.doi.org/10.1007/s10472-012-9329-3},
  bibsource = {DBLP, http://dblp.uni-trier.de} }

@inproceedings{DBLP:conf/splc/RysselPK11,
  author    = {Uwe Ryssel and
               Joern Ploennigs and
               Klaus Kabitzsch},
  title     = {Extraction of feature models from formal contexts},
  booktitle = {SPLC 2011 Workshops},
  year      = {2011} }

@inproceedings{DBLP:conf/models/SaadaDHNS12,
  author    = {Hajer Saada and
               Xavier Dolques and
               Marianne Huchard and
               Cl{\'e}mentine Nebut and
               Houari A. Sahraoui},
  title     = {{Generation of Operational Transformation Rules from Examples
               of Model Transformations}},
  booktitle = {MoDELS 2012, Model Driven Engineering Languages and Systems},
   series    = {LNCS},
  volume    = {7590},
year      = {2012},
  pages     = {546-561} }

@article{semeraro2025data,
  title={Data-driven digital twin for fault detection in compressed air energy storage systems: Design and experimental validation},
  author={Semeraro, Concetta and Ababneh, Rawnaq Faisal and Alkhatib, Lamis Ahmed and Saqallah, Dana and Al Koutoubi, Rawad and Aljaghoub, Haya and Alami, Abdul Hai and Abdelkareem, Mohammad Ali and Olabi, Abdul Ghani},
  journal={Energy},
  pages={138401},
  year={2025},
  publisher={Elsevier}}

@article{semeraro2023,
title = {Data-driven invariant modelling patterns for digital twin design},
journal = {Journal of Industrial Information Integration},
volume = {31},
pages = {100424},
year = {2023},
issn = {2452-414X},
doi = {https://doi.org/10.1016/j.jii.2022.100424},
OPTurl = {https://www.sciencedirect.com/science/article/pii/S2452414X22000917},
author = {Concetta Semeraro and Mario Lezoche and HervÃÂ© Panetto and Michele Dassisti}}

@article{Stumme2002,
 author = {Stumme, Gerd and Taouil, Rafik and Bastide, Yves and Pasquier, Nicolas and Lakhal, Lotfi},
 title = {Computing Iceberg Concept Lattices with TITANIC},
 journal = {Data Knowl. Eng.},
 issue_date = {August 2002},
 volume = {42},
 number = {2},
 OPTmonth = aug,
 year = {2002},
 issn = {0169-023X},
 pages = {189--222},
 numpages = {34},
 OPTurl = {http://dx.doi.org/10.1016/S0169-023X(02)00057-5},
 doi = {10.1016/S0169-023X(02)00057-5},
 acmid = {606457},
 publisher = {Elsevier Science Publishers B. V.},
 address = {Amsterdam, The Netherlands, The Netherlands} }

@article{DBLP:journals/order/Voutsadakis02,
  author       = {George Voutsadakis},
  title        = {Polyadic Concept Analysis},
  journal      = {Order},
  volume       = {19},
  number       = {3},
  pages        = {295--304},
  year         = {2002},
  OPTurl          = {https://doi.org/10.1023/A:1021252203599},
  doi          = {10.1023/A:1021252203599},
  bibsource    = {dblp computer science bibliography, https://dblp.org} }

@article{WAJNBERG20181397,
title = {Semantic interoperability of large systems through a formal method: Relational Concept Analysis},
journal = {IFAC-PapersOnLine},
volume = {51},
number = {11},
pages = {1397-1402},
year = {2018},
note = {Special Issue: 16th IFAC Symposium on Information Control Problems in Manufacturing INCOM 2018},
issn = {2405-8963},
doi = {https://doi.org/10.1016/j.ifacol.2018.08.330},
OPTurl = {https://www.sciencedirect.com/science/article/pii/S240589631831454X},
author = {Micka{\"{e}}l Wajnberg and Mario Lezoche and Alexandre Blondin-Mass{\'{e}} and Petko Valtchev and Herv\'{e} Panetto and Louise Tyvaert}}

@inproceedings{DBLP:conf/jowo/WajnbergVLMP19,
  author       = {Micka{\"{e}}l Wajnberg and
                  Petko Valtchev and
                  Mario Lezoche and
                  Alexandre Blondin Mass{\'{e}} and
                  Herv{\'{e}} Panetto},
  OPTeditor       = {Adrien Barton and
                  Selja Sepp{\"{a}}l{\"{a}} and
                  Daniele Porello},
  title        = {Concept Analysis-Based Association Mining from Linked Data: {A} Case
                  in Industrial Decision Making},
  booktitle    = {Proceedings of the Joint Ontology Workshops 2019 Episode {V:} The
                  Styrian Autumn of Ontology},
   series       = {{CEUR}-WS Proc.},
  volume       = {2518},
  publisher    = {CEUR-WS.org},
  year         = {2019},
  OPTurl          = {https://ceur-ws.org/Vol-2518/paper-DAOSI3.pdf},
  bibsource    = {dblp computer science bibliography, https://dblp.org}}

@inproceedings{DBLP:conf/icdm/WajnbergVMBKLLS20,
  author       = {Micka{\"{e}}l Wajnberg and
                  Petko Valtchev and
                  Alexandre Blondin Mass{\'{e}} and
                  Abderrahim Benmoussa and
                  Maja Krajinovic and
                  Caroline Laverdi{\`{e}}re and
                  Emile Levy and
                  Daniel Sinnett and
                  Val{\'{e}}rie Marcil},
  OPTeditor       = {Giuseppe Di Fatta and
                  Victor S. Sheng and
                  Alfredo Cuzzocrea and
                  Carlo Zaniolo and
                  Xindong Wu},
  title        = {Mining Heterogeneous Associations from Pediatric Cancer Data by Relational Concept Analysis},
  booktitle    = {{ICDM} Workshops  2020},
  pages        = {597--604},
  publisher    = {{IEEE}},
  year         = {2020},
  OPTurl          = {https://doi.org/10.1109/ICDMW51313.2020.00085},
  doi          = {10.1109/ICDMW51313.2020.00085},
  bibsource    = {dblp computer science bibliography, https://dblp.org} }
